\documentclass[fleqn,12pt]{wlscirep}
\usepackage[utf8]{inputenc}
\usepackage[T1]{fontenc}
\usepackage{float}
\usepackage{graphicx}
\usepackage{latexsym}
\usepackage{amsmath}
\usepackage{amssymb}
\usepackage{epsfig}
\usepackage{cite}
\usepackage{algorithmic}
\usepackage{overpic}
\usepackage{multirow}
\usepackage{colortbl,array}
\usepackage{multirow,bigdelim}
\usepackage{subfigure}
\usepackage{amsthm}
\usepackage{booktabs,amsmath}
\usepackage{fancyhdr, graphicx, amsmath, amssymb}
\usepackage[linesnumbered,ruled]{algorithm2e}
\usepackage{arydshln}
\usepackage[english]{babel}
\usepackage{xcolor}
\usepackage{tcolorbox}
\usepackage{soul}
\usepackage{mathtools}
\usepackage{booktabs}
\usepackage[export]{adjustbox}
\usepackage{tabularx}
\usepackage{makecell}
\usepackage[customcolors]{hf-tikz}
\usepackage{nicematrix}
\usepackage{tikz}
\usepackage{wrapfig}
\usepackage{xurl}
\usepackage[framemethod=tikz]{mdframed}

\newcommand{\mymatrix}{
    \left[\begin{gathered}
    \tikzpicture[every node/.style={anchor=south west}]
        \node[minimum width=4cm,minimum height=1.0cm] at (2.3,0) {$\bigzero_{(\text{len}(\mathbf{y}_{\left\langle t \right\rangle }) - \text{len}(\mathbf{y}))  \times \text{len}(\mathfrak{m}(\mathbf{y}_{\left\langle t-1 \right\rangle }, r_{\left\langle t \right\rangle}))}$};
        \node[minimum width=1.6cm,minimum height=1.0cm] at (0,1.2) {$\mathbf{0}_{\text{len}(\mathbf{y})}$};
        \node[minimum width=4.3cm,minimum height=1.0cm] at (0.72,1.2) {$\bigid_{\text{len}(\mathbf{y})}$};
        \node[minimum width=4.3cm,minimum height=1.0cm] at (4.4,1.2) {$\bigzero_{\text{len}(\mathbf{y}) \times (\text{len}(\mathfrak{m}(\mathbf{y}_{\left\langle t-1 \right\rangle }, r_{\left\langle t \right\rangle}))-\text{len}(\mathbf{y}) - 1)}$};
        \draw[dashed] (0,1.2) -- (9.5,1.2);
        \draw[dashed] (1.8,1.2) -- (1.8,2.4);
        \draw[dashed] (4.0,1.2) -- (4.0,2.4);
    \endtikzpicture
    \end{gathered}\right]
}

\newtheorem{defa}{Definition}
\newtheorem{prom}{Problem}
\newtheorem{ex}{Example}

\newtheorem{thm}{Theorem}

\newcommand{\bigzero}{\mbox{\normalfont\small\bfseries O}}
\newcommand{\bigid}{\mbox{\normalfont\small\bfseries I}}
\makeatletter
\newcommand*\bigcdot{\mathpalette\bigcdot@{0.8}}
\newcommand*\bigcdot@[2]{\mathbin{\vcenter{\hbox{\scalebox{#2}{$\m@th#1\bullet$}}}}}
\makeatother

{\centering \small \title{Phy-Taylor: Physics-Model-Based Deep Neural Networks}}

\author[1,*]{Yanbing Mao}
\author[2]{Lui~Sha}
\author[3]{Huajie~Shao}
\author[4]{Yuliang~Gu}
\author[5]{Qixin~Wang}
\author[2]{Tarek Abdelzaher}
\affil[1]{Engineering Technology Division, Wayne State University, Detroit, MI 48201, USA}
\affil[2]{Department of Computer Science, University of Illinois at Urbana-Champaign, Urbana, IL 61801, USA}
\affil[3]{Department of Computer Science, College of William \& Mary, Williamsburg, VA 23185, USA}
\affil[4]{Department of Mechanical Engineering, University of Illinois at Urbana--Champaign, Urbana, IL 61801, USA}
\affil[5]{Department of Computing, Hong Kong Polytechnic University, Hong Kong SAR, China}
\affil[*]{corresponding author: Yanbing Mao (e-mail: hm9062@wayne.edu)}

\begin{abstract}
Purely data-driven deep neural networks (DNNs) applied to physical engineering systems can infer relations that violate physics laws, thus leading to unexpected consequences. To address this challenge, we propose a physics-model-based DNN framework, called Phy-Taylor, that accelerates learning compliant representations with physical knowledge. The Phy-Taylor framework makes two key contributions; it introduces a new architectural physics-compatible neural network (PhN), and features a novel compliance mechanism, we call {\em Physics-guided Neural Network Editing\/}. The PhN aims to directly capture nonlinearities inspired by physical quantities, such as kinetic energy, potential energy, electrical power, and aerodynamic drag force. To do so, the PhN augments neural network layers with two key components: (i) monomials of Taylor series expansion of nonlinear functions capturing physical knowledge, and (ii) a suppressor for mitigating the influence of noise. The neural-network editing mechanism further modifies network links and activation functions consistently with physical knowledge. As an extension, we also propose a self-correcting Phy-Taylor framework that introduces two additional capabilities: (i) physics-model-based safety relationship learning, and (ii) automatic output correction when violations of safety occur. Through experiments, we show that (by expressing hard-to-learn nonlinearities directly and by constraining dependencies) Phy-Taylor features considerably fewer parameters, and a remarkably accelerated training process, while offering enhanced model robustness and accuracy.
\end{abstract}

\begin{document}

\flushbottom
\maketitle
\thispagestyle{empty}

\section{Introduction}
The paper proposes a novel physics-model-based deep neural network framework, called Phy-Taylor, that addresses a critical flaw in purely data-driven neural networks, when used to model aspects of physical engineering systems. Namely, it addresses the potential lack of agreement between learned latent neural network representations and prior physical knowledge -- a flaw that sometimes leads to catastrophic consequences~\cite{brief2021ai}. As shown in Figure \ref{PhyArt}, the Phy-Taylor framework introduces two contributions: the deep physics-compatible neural networks and a physics-guided neural network editing mechanism, aiming at ensuring compliance with prior physical knowledge. 

\begin{figure*}[!t]
\centering
\includegraphics[scale=0.50]{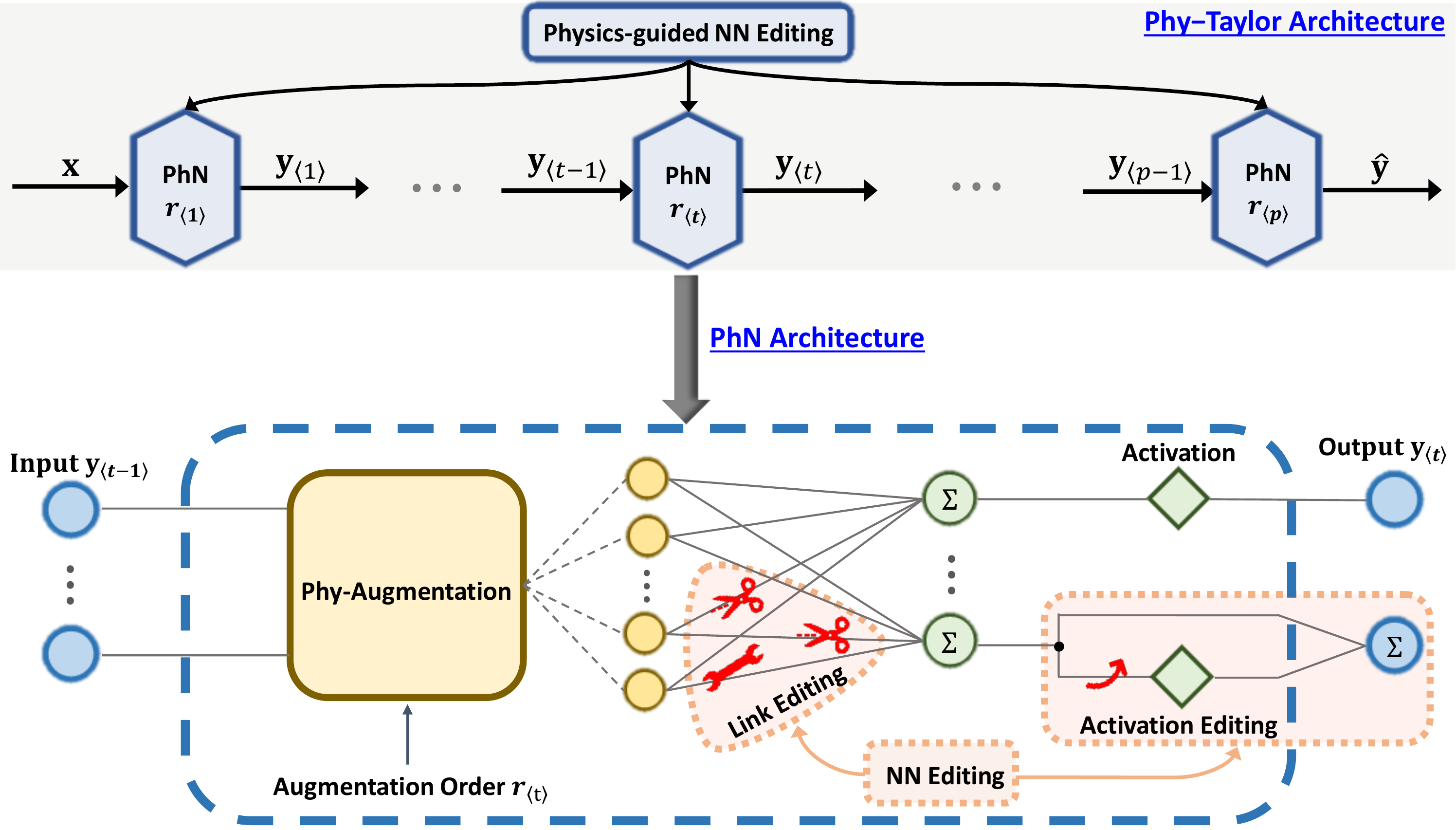}
\caption{Architectures of Phy-Taylor and physics-compatible neural network, having neural network (NN) editing including link editing and activation editing.}
\label{PhyArt}
\end{figure*}

The work contributes to emerging research on physics-enhanced deep neural networks. Current approaches include physics-informed neural networks~\cite{wang2021physics,willard2021integrating,jia2021physics,jia2019physics,wang2021deep,lu2021physics,chen2021theory,wang2020deep,xu2022physics,karniadakis2021physics,wang2020towards,daw2017physics,cranmer2020lagrangian,finzi2020simplifying,greydanus2019hamiltonian}, physics-guided neural-network architectures~\cite{muralidhar2020phynet,masci2015geodesic,monti2017geometric,horie2020isometric,wang2021incorporating,li2019learning} and physics-inspired neural operators~\cite{lusch2018deep,li2020fourier}. The physics-informed networks and physics-guided architectures use compact partial differential equations (PDEs) for formulating loss functions and/or architectural components. Physics-inspired neural operators, such as the Koopman neural operator~\cite{lusch2018deep} and the Fourier neural operator~\cite{li2020fourier}, on the other hand, map nonlinear functions into alternative domains, where it is easier to train their parameters from observational data and reason about convergence.
These frameworks improve consistency with prior analytical knowledge, but remain problematic in several respects. For example, (i) due to incomplete knowledge, the compact or precise PDEs may not always be available, and (ii) fully-connected neural networks can introduce spurious correlations that deviate from strict compliance with available well-validated physical knowledge. Instead, through the use of a Taylor-series expansion, the Phy-Taylor is able to leverage partial knowledge. Moreover, thanks to the neural editing mechanism, the framework removes links and reshapes activation functions not consistent with physics-based representations.

 The Phy-Taylor framework leverages the intuition that most physical relations live in low-dimensional manifolds, shaped by applicable physical laws. It's just that the estimation of key physical variables from high-dimensional system observations is often challenging. By expressing known knowledge as relations between yet-to-be-computed latent variables, we force representation learning to converge to a space, where these variables represent desired physical quantities, shaped by the applicable (expressed) physical laws. In effect, by shaping non-linear terms and relations in the latent space, we arrive at a desired physics-compliant latent representation. More specifically, Phy-Taylor offers the following two advantages:
\begin{itemize}
\vspace{-0.1in}
\item \textbf{\textit{Non-linear Physics Term Representation:}} Classical neural networks can learn arbitrary non-linear relations by unfolding them into layers of linear weighting functions and switch-like activations. This mechanism is akin to constructing nonlinearities by stitching together piecewise linear behaviors. Instead, by directly exploiting non-linear terms of the Taylor series expansion, we offer a set of features that express physical nonlinearities much more succinctly, thereby reducing the number of needed parameters and improving accuracy of representation. Monomials of the Taylor series can capture common nonliearities present in physics equations, such as kinetic energy, potential energy, rolling resistance and aerodynamic drag force. The (controllable) model error of the series drops significantly as the series order increases~\cite{konigsberger2013analysis}. 
The approach constructs input features that represent monomials of the Taylor series and adds a compressor for mitigating influence of noise on augmented inputs. 
\vspace{-0.1in}
\item \textbf{\textit{Removing Spurious Correlations:}} The general topology of neural networks allows for models that capture spurious correlations in training samples (overfitting)~\cite{yang2022understanding, sagawa2020investigation}. In contrast, we develop a neural network (topology) editing mechanism in the latent space that removes links among certain latent variables, when these links contradict their intended physical behaviors, thereby forcing the latent representation to converge to variables with the desired semantic interpretation that obey the desired physical relations. 
\vspace{-0.1in}
\end{itemize}

Through experiments with learning the dynamics of autonomous vehicles and other non-linear physical systems, we show that Phy-Taylor exhibits a considerable reduction in learning parameters, a remarkably accelerated training process, and greatly enhanced model robustness and accuracy (viewed from the perspective of long-horizon prediction of a trajectory).  Experiments with safe velocity regulation in autonomous vehicles further demonstrate that the self-correcting Phy-Taylor successfully addresses the dilemma of prediction horizon and computation time that nonlinear model-predictive control and control barrier function are facing in safety-critical control. 

\section{Problem Formulation} \label{sec:problem}
\begin{table}[ht] \footnotesize{
\centering
\caption{Table of Notation}
\begin{tabular}{|l|l|}
\hline
$\mathbb{R}^{n}$:~set of $\emph{n}$-dimensional real vectors  & $\mathbb{R}_{\ge 0}$:~ set of non-negative real numbers         \\ \hline
$\mathbb{N}$:~set of natural numbers & $[\mathbf{x}]_{i}$:~$i$-th entry of vector $\mathbf{x}$       \\ \hline
$[\mathbf{x}]_{i:j}$:~a sub-vector formed by the $i$-th to $j$-th entries of vector $\mathbf{x}$ & $[\mathbf{W}]_{i,j}$:~ element at row $i$ and column $j$ of matrix $\mathbf{W}$\\ \hline
$[\mathbf{W}]_{i,:}$:~ $i$-th row of matrix $\mathbf{W}$ &  $\left[\mathbf{x}~;~\mathbf{y}\right]$:~stacked (tall column) vector of vectors $\mathbf{x}$ and  $\mathbf{y}$  \\ \hline
$\mathbf{0}_{n}$:~ $n$-dimensional vector of all zeros  & $\mathbf{1}_{n}$:~ $n$-dimensional vector of all ones        \\ \hline
$\mathbf{O}_{m \times n}$:~  $m \times n$-dimensional zero matrix  & $\mathbf{I}_{n}$:~  $n \times n$-dimensional identity matrix   \\ \hline
$|| \cdot ||$:~ Euclidean norm of a vector or absolute value of a number & $\odot$:~ Hadamard product  \\ \hline
$\bigcdot$:~ multiplication operator & $\mathrm{len}(\mathbf{x})$:~ length of vector $\mathbf{x}$  \\ \hline
$\text{act}$:~ activation function & $\text{sus}$:~ suppressor function  \\ \hline
$\top$:~matrix or vector transposition & \text{ina}: a function that is inactive\\ \hline
$\boxplus$:~a known model-substructure parameter & $*$:~an unknown model-substructure parameter \\ \hline
\end{tabular}\label{notation}}
\end{table}



Consider the problem of computing some output vectors, $\mathbf{y}$, from a set of observations, $\mathbf{x}$. The relation between $\mathbf{x}$ and $\mathbf{y}$ is partially determined by physical models of known structure (but possibly unknown parameter values) and partially unknown, thus calling for representation learning of the missing  substructures using neural network observables. For example, $\mathbf{y}$ might denote the estimated momentum and future position of a target as a function of a vector of $\mathbf{x}$, that include its position, velocity, and type. In this formulation, position and velocity might be directly related to output quantities via known physical relations, but type is represented only indirectly by an image that requires some representation learning in order to translate it into relevant parameters (such as mass and maneuverability) from which the outputs can be computed. We express the overall input/output relation by the function:
\begin{align}
\mathbf{y} = \underbrace{\mathbf{A}}_{\text{weight matrix}} \cdot \underbrace{\mathfrak{m}(\mathbf{x},r)}_{\text{node-representation vector}} + \underbrace{\mathbf{f}(\mathbf{x})}_{\text{model mismatch}} \triangleq \underbrace{\mathbf{g}(\mathbf{x})}_{\text{ground truth model}},
\label{eq:lobja}
\end{align}
\noindent
where $\mathbf{y}$ and $\mathbf{x}$ are the output and input vectors of overall system model, respectively, and the parameter $r \in \mathbb{N}$ controls model size. For convenience, Table~\ref{notation} summarizes the remaining notations used throughout the paper. 
Since Equation~\eqref{eq:lobja} combines known and unknown model substructures, we distinguish them according to the definition below.

\begin{defa}
For all $i \in \{1, 2, \ldots, \mathrm{len}(\mathbf{y})\}$, $j \in \{1, 2, \ldots, \mathrm{len}(\mathfrak{m}(\mathbf{x},r))\}$, element $[\mathbf{A}]_{i,j}$ is said to be a {\em known model-substructure parameter\/} in Equation~\eqref{eq:lobja} if and only if $\frac{{\partial [\mathbf{f}(\mathbf{x})]_i}}{{\partial {[\mathfrak{m}}(\mathbf{x},r)]_{j}}} \equiv 0$. Otherwise, it is called an {\em unknown model-substructure parameter\/}. 
\label{defj}
\end{defa}

\noindent
Definition~\ref{defj} indicates that the model-substructure knowledge includes:
\begin{itemize}
\vspace{-0.1in}
\item (i) \textit{Known parameter values but completely unknown model formula} (see e.g., the Example \ref{exap1}).
\vspace{-0.1in}
\item (ii) \textit{Partially-known model formula}. For example, in the (lateral) force balance equation of autonomous vehicle \cite{rajamani2011vehicle}: 
  \begin{align}
m\left( {\ddot y + \dot \psi {V_x}} \right) = {F_{\text{yf}}} + {F_{\text{fr}}} + {F_{\text{bank}}}, \label{physicalfea}
\end{align}
the force formula due to road bank angle $\phi$, i.e., ${F_{\text{bank}}} = mg\sin(\phi)$, is known while other force formulas are unknown because of complex and unforeseen driving environments.
\vspace{-0.1in}
\item (iii) \textit{Known model formula but unknown parameter values}.
\end{itemize}

\begin{ex} [Identify Known Model-Substructure Parameters] DNN Design Goal: Use the current position $p(k)$, velocity $v(k)$ and mass $m$ to estimate a vehicle's next velocity $v(k+1)$, sensed road friction coefficient $r(k+1)$ and safety metric $s(k+1)$ of velocity regulation, in the dynamic driving environments. With the knowledge of vehicle dynamics and control \cite{rajamani2011vehicle}, the problem can be mathematically described by
\begin{subequations}
\begin{align}
v\left( {k + 1} \right) &= {g_1}\left( {p\left( k \right),~v\left( k \right),~m} \right), \label{example2a}\\
r\left( {k + 1} \right) &= {g_2}\left( {m,~v\left( k \right)} \right), \label{example2b}\\
s\left( {k + 1} \right) &= {g_3}\left( {v\left( k \right)} \right), \label{example2c}
\end{align}\label{example2}
\end{subequations}
\noindent
\!We here assume the \underline{formulas of ${g_1}(\cdot)$--${g_3}(\cdot)$ are unknown}. While the Equations \eqref{example2b} and \eqref{example2c} indicate that given the inputs, the available knowledge are (i) the road friction coefficient varies with a vehicle's mass and real-time velocity only, and (ii) the considered safety metric depends on velocity only. We let $\mathbf{x} = [p(k);~ v(k);~m]$, $\mathbf{y} = [v(k+1);~r(k+1);~s(k+1)]$, $\mathbf{g}(\mathbf{x}) = [{g_1}({p(k),v(k),m});~ {g_2}( {m,v( k)});~ {g_3}( {v(k)})]$, and  $\mathfrak{m}(\mathbf{x},r) = [1;~ p(k);~ v(k);~ m;~ p^2(k);~ p(k)v(k);~ mp(k);~ v^2(k);~mv(k);~ m^2]$. The ground truth model \eqref{example2} is then equivalently rewritten in the form of \eqref{eq:lobja}:
\begin{align}
\mathbf{y} = \underbrace{\left[ {\begin{array}{*{20}{c}}
*&*&*&*&*&*&*&*&*&*\\
*&0&*&*&0&0&0&*&*&*\\
*&0&*&0&0&0&0&*&0&0
\end{array}} \right]}_{\mathbf{A}} \cdot \mathfrak{m}(\mathbf{x},r) + \underbrace{\mathbf{g}(\mathbf{x}) - \mathbf{A} \cdot \mathfrak{m}(\mathbf{x},r)}_{\mathbf{f}(\mathbf{x})} = \mathbf{g}(\mathbf{x}), \nonumber
\end{align}
which thus encodes the available knowledge points (i) and (ii) to the known model-substructure parameters (i.e., zeros) in system matrix $\mathbf{A}$.\label{exap1}
\end{ex}


\noindent Considering this definition, the problem addressed in this paper is formally stated below. 
\begin{prom}
Given a time-series of inputs, $\mathbf{x}$, the corresponding outputs, $\mathbf{y}$, and the known model substructures in Equation~\eqref{eq:lobja}, it is desired to develop an end-to-end neural network that directly estimates $\mathbf{y}$ (denoted by $\widehat{\mathbf{y}}$), given $\mathbf{x}$, consistently with all known model substructures. In other words, the model must satisfy the property that for each known model-substructure parameter, $[\mathbf{A}]_{i,j}$, the end-to-end model must ensure that $\frac{{\partial {{[\widehat{\mathbf{y}}}]_i}}}{{\partial {[\mathfrak{m}}\left( {\mathbf{x},r} \right)]_{j}}} \equiv [\mathbf{A}]_{i,j}$ for any $\mathfrak{m}(\mathbf{x}, r)$. 
\label{problem}
\end{prom}


\noindent
The above definition allows the system described by Equation~\eqref{eq:lobja} to have an end-to-end model that intertwines well-known substructure properties with high-order unmodeled correlations of unknown nonlinear structure. In this, our problem differs from past seminal frameworks of physics-enhanced DNNs~\cite{wang2021physics,willard2021integrating,jia2021physics,jia2019physics,wang2021deep,lu2021physics,chen2021theory,wang2020deep,xu2022physics,karniadakis2021physics,wang2020towards,daw2017physics,cranmer2020lagrangian,finzi2020simplifying,greydanus2019hamiltonian, muralidhar2020phynet,masci2015geodesic,monti2017geometric,horie2020isometric,wang2021incorporating,li2019learning,lusch2018deep,kani2017dr,belbute2020combining,wu2021deepgleam,guen2020disentangling,garcia2019combining,long2018hybridnet,yin2021augmenting}, that use a compact partial differential equations (PDEs) for formulating the PDEs-regulated loss function and/or DNN architectures to count the degree mean of consistency with PDEs. The proposed solution to Problem~\ref{problem} is the Phy-Taylor framework, which will rely on two building blocks: a deep physics-compatible neural network (PhN) and a physics-guided neural network editing mechanism, presented in the next section.


\section{Phy-Taylor Framework}
The proposed Phy-Taylor for addressing Problem~\ref{problem} is shown in Figure \ref{PhyArt}, which is built on the conjunctive deep physics-compatible neural network (PhN) and physics-guided neural network (NN) editing. In other words, implementing NN editing according to Taylor's theorem for embedding available physical knowledge into deep PhN yields the Phy-Taylor. The PhN is a neural network layer with a key component: a physics-inspired augmentation (called Phy-Augmentation) for generating monomials in Equation~\eqref{eq:lobja} of Taylor series expansion of nonlinear functions capturing physical knowledge. The physics-guided NN editing -- including link editing and activation editing -- further modifies network links and activation functions consistently with physical knowledge. Specifically, the link editing performs removing and preserving links according to the consistency with physical knowledge. Meanwhile, the activation editing performs the physics-knowledge-preserving computing in output channel of each PhN. Collaboratively through link and activation editing, the input/output of Phy-Taylor strictly complies with the available physical knowledge, which is a desired solution to Problem~\ref{problem}. Next, we detail the two components.

\subsection{The Physics-compatible Neural Network (PhN)}\label{sec:TNO}
In order to capture non-linear features of physical functions, we introduce a new type of network layer that is augmented with terms derived from Taylor series expansion. 
The Taylor's Theorem offers a series expansion of arbitrary nonlinear functions, as shown below.
\definecolor{cccolor}{rgb}{.67,.7,.67}
\begin{mdframed}
  \textbf{Taylor's Theorem (Chapter 2.4 \cite{konigsberger2013analysis}):} Let $\mathbf{g}\!:~ \mathbb{R}^{n} \to \mathbb{R}$  be a $r$-times continuously differentiable function at the point $\mathbf{o} \in \mathbb{R}^{n}$. Then there exists $\mathbf{h}_{\alpha}\!:~ \mathbb{R}^{n} \to \mathbb{R}$, where $\left| \alpha  \right| = r$, such that
  \begin{align}
&\mathbf{g}( \mathbf{x} ) = \sum\limits_{\left| \alpha  \right| \le r} {\frac{{{\partial^\alpha }\mathbf{g}( \mathbf{o} )}}{{\alpha !}}} {\left( {\mathbf{x} - \mathbf{o}} \right)^\alpha } + \sum\limits_{\left| \alpha  \right| = r} {{\mathbf{h}_\alpha }( \mathbf{x} ){{( {\mathbf{x} - \mathbf{o}} )^\alpha} }}, \hspace{0.2cm}\text{and}~\mathop {\lim }\limits_{\mathbf{x} \to \mathbf{o}} {\mathbf{h}_\alpha}\left( \mathbf{x} \right) = \mathbf{0}, \label{taylortheorem}
\end{align}
where $\alpha  = \left[ {{\alpha _1};~{\alpha _2};~ \ldots;~{\alpha _n}} \right]$, $\left| \alpha  \right| = \sum\limits_{i = 1}^n {{\alpha _i}}$, ~$\alpha ! = \prod\limits_{i = 1}^n {{\alpha _i}}!$, ~${\mathbf{x}^\alpha } = \prod\limits_{i = 1}^n {\mathbf{x}_i^{{\alpha _i}}}$, ~and ${\partial^\alpha }\mathbf{g} = \frac{{{\partial ^{\left| \alpha  \right|}} \mathbf{g} }}{{\partial \mathbf{x}_1^{{\alpha _1}} \cdot  \ldots  \cdot \partial \mathbf{x}_n^{{\alpha _n}}}}$.
\end{mdframed}
\noindent
The Taylor's theorem has several desirable properties:
\begin{itemize}
\vspace{-0.1in}
  \item \textit{Non-linear Physics Term Representation:} The high-order monomials (i.e., the ones included in $\left( {\mathbf{x} - \mathbf{o}} \right)^\alpha$ with $|\alpha| \ge 2$) of the Taylor series  (i.e., $\sum\limits_{\left| \alpha  \right| \le r} {\frac{{{D^\alpha }\mathbf{g}( \mathbf{o} )}}{{\alpha !}}} {\left( {\mathbf{x} - \mathbf{o}} \right)^\alpha }$) capture core nonlinearities of physical quantities such as kinetic energy ($\triangleq \frac{1}{2}m{v^2}$), potential energy ($\triangleq \frac{1}{2}k{x^2}$), electrical power ($\triangleq V \cdot I$)  and aerodynamic drag force ($\triangleq \frac{1}{2}\rho {v^2}{C_D}A$), that drive the state dynamics of physical systems.
  \vspace{-0.00in}
  \item \textit{Controllable Model Accuracy:} Given ${\mathbf{h}_\alpha}( \mathbf{x} )$ is finite and $\left\| {\mathbf{x} - \mathbf{o}} \right\| < 1$, the error $\sum\limits_{\left| \alpha  \right| = r} {{\mathbf{h}_\alpha }( \mathbf{x} ){{( {\mathbf{x} - \mathbf{o}} )^\alpha} }}$ for approximating the ground truth $\mathbf{g}(\mathbf{x})$ will drop significantly as the order $r = |\alpha|$ increases and $\mathop {\lim }\limits_{|\alpha| = r \to \infty } {\mathbf{h}_\alpha}(\mathbf{x}){\left( {\mathbf{x} - \mathbf{o}} \right)^\alpha} = \mathbf{0}$. This allows for controllable model accuracy via controlling order $r$.
  \vspace{-0.12in}
  \item \textit{Knowledge Embedding:} The Taylor series can directly project the known model substructure parameters of the ground-truth model \eqref{eq:lobja} into neural network parameters including the weight matrix (${\frac{{{D^\alpha }\mathbf{g}( \mathbf{o} )}}{{\alpha !}}}$ with $|\alpha| > 0$)  and bias (${\frac{{{D^\alpha }\mathbf{g}( \mathbf{o} )}}{{\alpha !}}}$ with $|\alpha| = 0$), thus paving the way to embed the available physical knowledge in the form of an appropriately weighted neural network layer. 
\end{itemize}


\begin{wrapfigure}{r}{0.56\textwidth}
\vspace{-0.9cm}
  \begin{center}
    \includegraphics[width=0.58\textwidth]{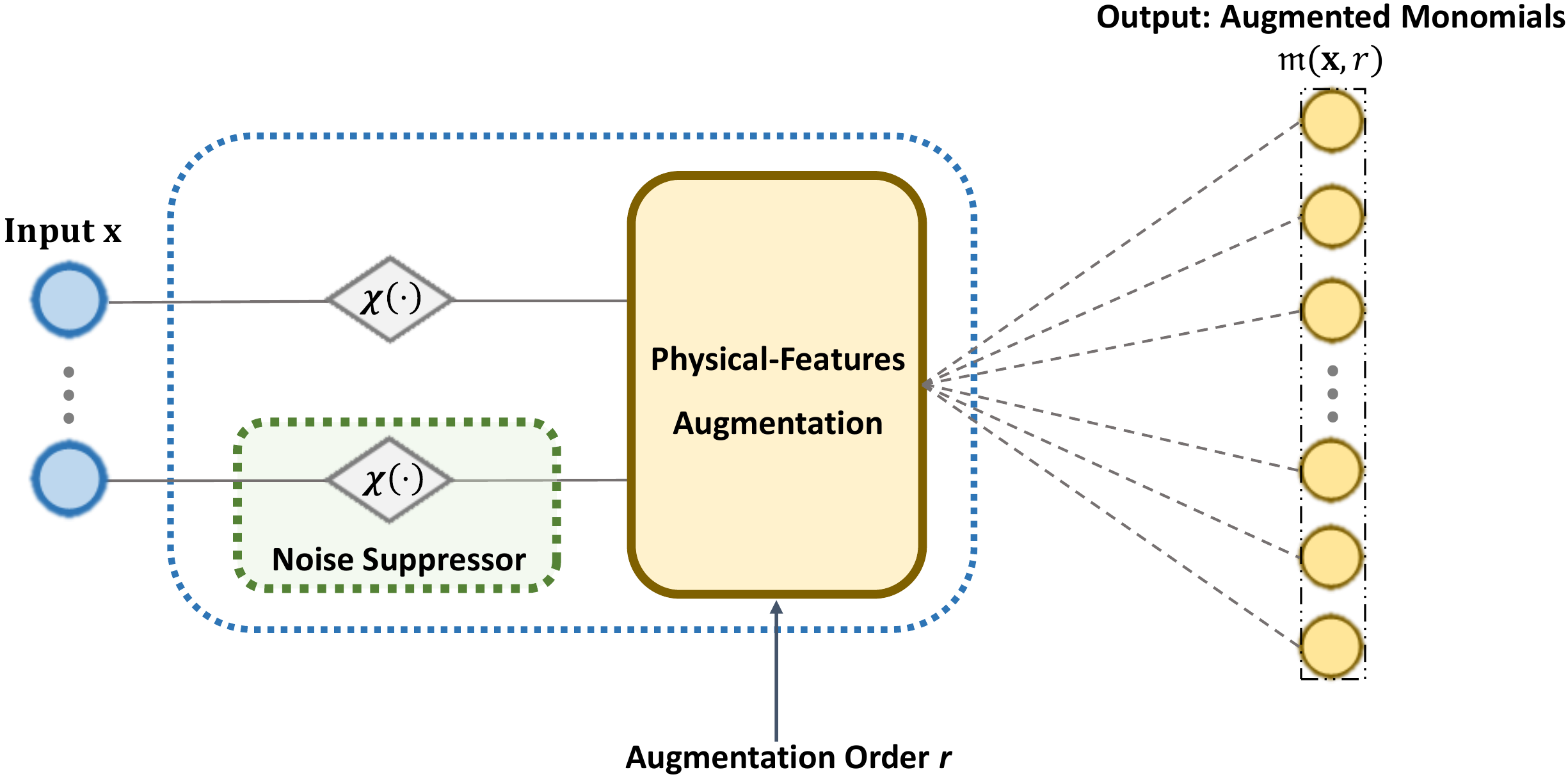}
  \end{center}
  \vspace{-0.64cm}
  \caption{Phy-Augmentation architecture.}
  \vspace{-0.5cm}
  \label{PhyA}
\end{wrapfigure}
We note the Taylor theorem relies on an assumption that the ground truth $\mathbf{g}(\mathbf{x})$ is a $r$-times continuously differentiable function at the point $\mathbf{o}$. If the assumption does not hold, what the Taylor series will approximate is a proximity of ground truth that is $r$-times continuously differentiable. For continuous functions, this is often a sufficient approximation. Next, we describe how PhNs embed the Taylor series expansion into neural network layers.
The resulting architecture (of a single PhN layer) is shown in Figure \ref{PhyArt}.  Compared with a classical neural network layer, we introduce the Phy-Augmentation, whose architecture is shown in Figure \ref{PhyA}. The Phy-Augmentation has two components: (i) augmented inputs that represent monomials of a Taylor series expansion, and (ii) a suppressor for mitigating the influence of noise on such augmented inputs (i.e., high-order monomials). Next, we detail them.

\subsubsection{Phy-Augmentation: Taylor Series Monomials}
\begin{algorithm} \footnotesize{ 
  \caption{Phy-Augmentation Procedure} \label{ALG1}
  \KwIn{augmentation order $r$, input $\mathbf{x}$, point $\mathbf{o}$, suppressor mapping $\chi(\cdot)$.}
   Suppress input: $[\mathbf{x}]_{i} \leftarrow \begin{cases}
		  \chi([{\mathbf{x}}]_{i}), &\text{if suppressor is active}\\
		  [{\mathbf{x}}]_{i}, &\text{otherwise}
	          \end{cases}$, ~~~~~~~~$i \in \{1,2,\ldots, \text{len}(\mathbf{x})\}$;\label{ALG1-102}\\
   Generate index vector of input entries: $\mathbf{i} \leftarrow [1;~2;~\ldots;~\mathrm{len}({\mathbf{x}})]$;\label{ALG1-1}\\
   Generate augmentations: ${\mathfrak{m}}({\mathbf{x}},r)  \leftarrow {\mathbf{x}}$; \label{ALG1-1}\\
   \For{$ \_\ = 2$ to $r$}{ \label{alg1-3}
       \For{$i=1$ to $\mathrm{len}({\mathbf{x}})$ \label{alg1-4}} 
           {Compute temporaries: ${\mathbf{{t}}}_a$ $\leftarrow [{\mathbf{x}}]_{i} \cdot [{\mathbf{x}}]_{\left[ {[\mathbf{i}]_{i}~:~\mathrm{len}({\mathbf{x}})} \right]}$; \label{alg1-5}\\
           \eIf{$i==1$ \label{alg1-6}} 
               {Generate temporaries: $\widetilde{\mathbf{{t}}}_b \leftarrow \widetilde{\mathbf{{t}}}_a$; \label{alg1-7}\\}
               {Generate temporaries: $\widetilde{\mathbf{{t}}}_{b} \leftarrow \left[ \widetilde{\mathbf{{t}}}_{b};~\widetilde{\mathbf{{t}}}_{a}  \right]$; \label{alg1-9}\\}
        Update index entry: $[\mathbf{i}]_i \leftarrow \mathrm{len}({\mathbf{x}})$; \label{alg1-11}\\
        Update augmentations: ${\mathfrak{m}}({\mathbf{x}},r)  \leftarrow \left[ {\mathfrak{m}}({\mathbf{x}},r);~{{\mathbf{{t}}}_{b}}  \right]$; \label{alg1-12}\\
       }
   \label{alg1-13} }\label{alg1-14}
Output vector of augmented monomials: ${\mathfrak{m}}({\mathbf{x}},r)  \leftarrow \left[1; ~{\mathfrak{m}}({\mathbf{x}},r) \right]$. \label{alg1-16}}
\end{algorithm}

\begin{wrapfigure}{r}{0.471\textwidth}
\vspace{-0.9cm}
  \begin{center}
    \includegraphics[width=0.471\textwidth]{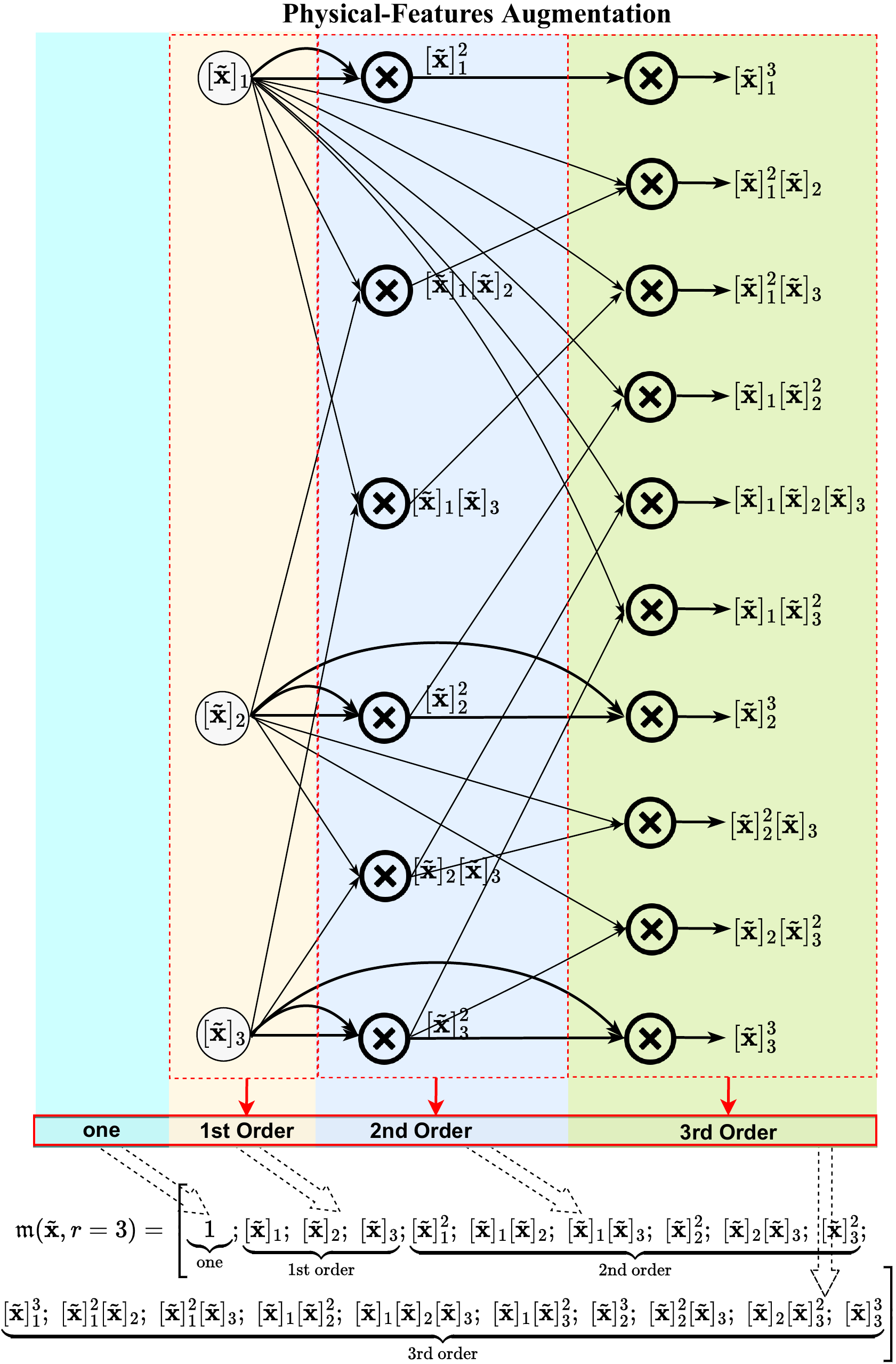}
  \end{center}
  \vspace{-0.5cm}
  \caption{ An example of Algorithm~\ref{ALG1} in TensorFlow framework, where input $\tilde{\mathbf{x}} \in \mathbb{R}^3$ is from Line \ref{ALG1-102} of Algorithm \ref{ALG1}.}
  \vspace{-0.7cm}
  \label{PhyBB}
\end{wrapfigure}
The function of physical-features augmentation in Figure \ref{PhyA} is to generate the vector of physical features (i.e., node representations) in form of Taylor series monomials, which is formally described by Algorithm \ref{ALG1}. The Lines \ref{alg1-5}--\ref{alg1-12} of Algorithm \ref{ALG1} guarantee that the generated node-representation vector embraces all the non-missing and non-redundant monomials of Taylor series. The Line \ref{alg1-16} shows that Algorithm \ref{ALG1} finally stacks vector with one.  This operation means a PhN node will be assigned to be one, and the bias (corresponding to ${\frac{{{D^\alpha }\mathbf{g}( \mathbf{o} )}}{{\alpha !}}}$ with $|\alpha| = 0$ in Taylor series) will be thus treated as link weights in PhN layers. As an example shown in Figure \ref{PhyBB}, the Phy-Augmentation empowers PhN to well capture core nonlinearities of physical quantities (see e.g., kinetic energy, potential energy, electrical power and aerodynamic drag force) that drive the state dynamics of physical systems, and then represent or approximate physical knowledge in form of the Taylor series. 

We note the Line \ref{ALG1-102} of Algorithm \ref{ALG1} means the noise suppressor is not applied to all the input elements. The critical reason is the extra mapping induced by suppressor on inputs can destroy the compliance with physical knowledge, when the available model-substructure knowledge do not include the mapping of suppressor. We next is going to present the function of suppressor. 

\subsubsection{Phy-Augmentation: Noise Suppressor}
The suppressor in Figure \ref{PhyA} is to mitigate the influence of noise on the augmented high-order monomials. Before proceeding on the working mechanism, we present a metric pertaining to noise and true data. 
\vspace{-0.00in}
\begin{defa}
Consider the noisy data and define the data-to-noise ratio ($\mathrm{DNR}$):
\begin{align}
[\bar{\mathbf{x}}]_i = \underbrace{[\mathbf{h}]_i}_{\text{true data}} + \underbrace{[{\mathbf{w}}]_i}_{\text{noise}} \in \mathbb{R}, ~~~~~~~~~~~~~~~\mathrm{DNR}_i \triangleq \frac{[\mathbf{h}]_i}{[\mathbf{w}]_i}.
\label{nd}
\end{align}
\label{def1}
\end{defa}

\vspace{-0.45cm}
The auxiliary Theorem \ref{colthm} presented in Supplementary Information \ref{Aux} implies that the high-order monomials can shrink their DNRs due to nonlinear mapping.  This means the PhN can be vulnerable to the noisy inputs, owning to Phy-Augmentation for generating high-order monomials. Hence, mitigating the influence of noise is vital for enhancing the robustness of PhNs, consequently the Phy-Taylor. As shown in Figure \ref{PhyA}, we incorporate a suppressor into PhN to process the raw input data, such that the high-order monomial from Phy-Augmentation can enlarge their DNRs. Building on Definition \ref{def1},  the proposed noise suppressor mapping is 
\begin{align}
\chi([\bar{\mathbf{x}}]_i) = \chi([\mathbf{h}]_i + [\mathbf{w}]_i) = \begin{cases}
		0, & \text{if}~[\mathbf{h}]_i + [\mathbf{w}]_i < 0\\
		[\mathbf{h}]_i + [\mathbf{w}]_i, & \text{if}~[\mathbf{h}]_i + [\mathbf{w}]_i \ge 0 ~\text{and}~[\mathbf{w}]_i < 0 \\
		([\mathbf{h}]_i + [\mathbf{w}]_i) \cdot \kappa_i + \rho_i, & \text{if}~[\mathbf{h}]_i + [\mathbf{w}]_i \ge 0 ~\text{and}~[\mathbf{w}]_i > 0
	\end{cases}, \label{compbadd} 
\end{align}
where the parameters $\rho_i$ and $\kappa_i$ satisfy 
\begin{align}
|\rho_i|  \ge |[\mathbf{h}]_i + [\mathbf{w}]_i | \cdot |\kappa_i|. \label{compb} 
\end{align}

We next present the suppressor properties in the following theorem, whose proof appears in Supplementary Information \ref{SI02}.
\begin{thm}
Consider the noisy data $[\bar{\mathbf{x}}]_i$ and the suppressor described in Equations \eqref{nd} and \eqref{compbadd}, respectively. Under the condition \eqref{compb}, the suppressor output, denoted by $[\widehat{\mathbf{x}}]_{i} = \chi([\bar{\mathbf{x}}]_i)$, has the properties: 
\begin{align}
&\hspace{-0.90cm}\text{The DNR magnitude of high-order monomial $[\widehat{\mathbf{x}}]_i^p[\widehat{\mathbf{x}}]_j^q$ ($p+q \ge 2$) is strictly increasing with respect to} \nonumber\\
&\hspace{-0.90cm}\hspace{9.5cm}\text{DNR magnitudes of $[\widehat{\mathbf{x}}]_i$ and $[\widehat{\mathbf{x}}]_j$.} \label{ckko}\\
&\hspace{-0.90cm}\text{The true data and the noise of suppressor output $[\widehat{\mathbf{x}}]_i$ are} \nonumber\\
&\hspace{-0.90cm}[\widetilde{\mathbf{h}}]_i = \begin{cases}
        \![\mathbf{h}]_i \cdot \kappa_i + \rho_i, \!\!&[\mathbf{h}]_i \!+\! [\mathbf{w}]_i \!\ge\! 0 ~\text{and}~[\mathbf{w}]_i \!>\! 0 \\
		\![\mathbf{h}]_i, \!\!& \text{otherwise}\\
	\end{cases}, ~~~~[\widetilde{\mathbf{w}}]_i = \begin{cases}
		\!-[\mathbf{h}]_i, \!\!&[\mathbf{h}]_i \!+\! [\mathbf{w}]_i \!<\! 0\\
		\![\mathbf{w}]_i, \!\!&[\mathbf{h}]_i \!+\! [\mathbf{w}]_i \!\ge\! 0 ~\text{and}~[\mathbf{w}]_i \!<\! 0 \\
		\![\mathbf{w}]_i \cdot \kappa_i, \!\!&[\mathbf{h}]_i \!+\! [\mathbf{w}]_i \!\ge\! 0 ~\text{and}~[\mathbf{w}]_i \!>\! 0
	\end{cases},
\label{cohgq3}
\end{align}
such that $[\widehat{\mathbf{x}}]_i = [\widetilde{\mathbf{h}}]_i + [\widetilde{\mathbf{w}}]_i,~ i \in \{1,2,\dots, \mathrm{len}(\widehat{\mathbf{x}})\}$.  \label{th2}
\end{thm}
\noindent 
The result \eqref{cohgq3} implies the parameters $\kappa$ and $\rho$ control the DNRs of suppressed data, consequently the high-order monomials. Furthermore, the result \eqref{ckko} suggests that through designing parameters $\kappa_i$, $\rho_i$, $\kappa_j$ and $\rho_j$ for increasing the DNR magnitudes of data $[\widehat{\mathbf{x}}]_i$ and $[\widehat{\mathbf{x}}]_j$, the DNR of
high-order monomial $[\widehat{\mathbf{x}}]_i^p[\widehat{\mathbf{x}}]_j^q$ can be enlarged consequently, such that the influence of noise is mitigated. 

\subsection{Physics-guided Neural Network Editing}\label{sec:Neu}
\begin{algorithm} \caption{Physics-guided NN Editing}  \label{ALG2}
  \KwIn{Available knowledge included in system matrix $\mathbf{A}$ of ground-truth model \eqref{eq:lobja}, terminal output dimension $\text{len}(\mathbf{y})$, number $p$ of PhNs, activation functions $\text{act}(\cdot)$, $\mathbf{y}_{\left\langle 0 \right\rangle} = \mathbf{x}$ and $r_{\left\langle 1 \right\rangle} = r$.} 
  \SetKwInOut{Output}{Output}
       \For{$t = 1$ to $p$}
           {
            \eIf{$t==1$ \label{alg2-2}}
            {Deactivate noise suppressor; \label{alg2-001a}\\
            Generate node-representation vector $\mathfrak{m}(\mathbf{y}_{\left\langle t-1 \right\rangle}, r_{\left\langle t \right\rangle})$ via Algorithm~\ref{ALG1}; \label{alg2-002a}\\
            Generate knowledge matrix $\mathbf{K}_{\left\langle t \right\rangle }$: ~$[\mathbf{K}_{\left\langle t \right\rangle }]_{i,j} \leftarrow \begin{cases}
		  [\mathbf{A}_{\left\langle t \right\rangle }]_{i,j}, &[\mathbf{A}_{\left\langle t \right\rangle }]_{i,j} = \boxplus\\
		  0, &\text{otherwise}
	          \end{cases}$; \label{alg2-3}\\
	              Generate weight-masking matrix $\mathbf{M}_{\left\langle t \right\rangle }$: ~$[\mathbf{M}_{\left\langle t \right\rangle }]_{i,j} \leftarrow \begin{cases}
		  0, &[\mathbf{A}_{\left\langle t \right\rangle }]_{i,j} = \boxplus\\
		  1, &\text{otherwise}
	          \end{cases}$; \label{alg2-4}\\
	          Generate activation-masking vector $\mathbf{a}_{\left\langle t \right\rangle }$: \!~$[\mathbf{a}_{\left\langle t \right\rangle }]_{i} \!\leftarrow\! \begin{cases}
		  0, \!\!\!&[\mathbf{M}_{\left\langle t \right\rangle }]_{i,j} \!=\! 0, \forall  j \!\in\! \{1,\ldots, \text{len}(\mathfrak{m}(\mathbf{y}_{\left\langle t-1 \right\rangle}, r_{\left\langle t \right\rangle})) \}\\
		  1, \!\!\!&\text{otherwise}
	          \end{cases}$; \label{alg2-5}\\
                  }
               {
               Generate node-representation vector $\mathfrak{m}(\mathbf{y}_{\left\langle t-1 \right\rangle}, r_{\left\langle t \right\rangle})$ via Algorithm~\ref{ALG1}; \label{alg2-003a}\\ 
              Generate knowledge matrix $\mathbf{K}_{\left\langle t \right\rangle }$: 
               \begin{align}
               \mathbf{K}_{\left\langle t \right\rangle } \leftarrow \mymatrix; \nonumber
               \end{align} \label{alg2-7}\\
                 Generate weight-masking matrix $\mathbf{M}_{\left\langle t \right\rangle }$: ~$[\mathbf{M}_{\left\langle t \right\rangle }]_{i,j} \leftarrow \begin{cases}
		  0, &\frac{{\partial [\mathfrak{m}(\mathbf{y}_{\left\langle t \right\rangle}, r_{\left\langle t \right\rangle})]_{j}}}{{\partial [\mathfrak{m}(\mathbf{x}, r_{\left\langle 1 \right\rangle})]_{v}}} \ne 0 ~\text{and}~[\mathbf{M}_{\left\langle 1 \right\rangle }]_{i,v} = 0, ~~v \in \{1,2, \dots, \text{len}(\mathfrak{m}(\mathbf{x}, r_{\left\langle 1 \right\rangle}))\} \\
		  1, &\text{otherwise}
	          \end{cases}$; \label{alg2-8}\\
Generate activation-masking vector $\mathbf{a}_{\left\langle t \right\rangle } \leftarrow  \left[ \mathbf{a}_{\left\langle 1 \right\rangle };~ \mathbf{1}_{\text{len}(\mathbf{y}_{\left\langle t \right\rangle }) - \text{len}(\mathbf{y})} \right]$; \label{alg2-9}\\
                      }
                  Generate weight matrix: ${\mathbf{W}_{\left\langle t \right\rangle }}$; \label{alg2-004a}\\
                  Generate uncertainty matrix $\mathbf{U}_{\left\langle t \right\rangle } \leftarrow  \mathbf{M}_{\left\langle t \right\rangle } \odot {\mathbf{W}_{\left\langle t \right\rangle }}$; \label{alg2-11}\\
                  Compute output: $\mathbf{y}_{\left\langle t \right\rangle } \leftarrow \mathbf{K}_{\left\langle t \right\rangle } \cdot \mathfrak{m}(\mathbf{y}_{\left\langle t-1 \right\rangle }, r_{\left\langle t \right\rangle}) +  \mathbf{a}_{\left\langle t \right\rangle } \odot \text{act}\left( {\mathbf{U}_{\left\langle t   
                  \right\rangle } \cdot \mathfrak{m}\left( {\mathbf{y}_{\left\langle t-1 \right\rangle },  r_{\left\langle t \right\rangle } } \right)} \right)$ \label{alg2-12};
       }
\Output{terminal output: $\widehat{\mathbf{y}} \leftarrow  \mathbf{y}_{\left\langle p \right\rangle}$}.
\end{algorithm}

Building on the deep PhN, this section presents the neural network (NN) editing for embedding and preserving the available physical knowledge, through the physics-guided link editing and activation editing. Specifically, the link editing centers around removing and preserving the links according to the consistency with physical knowledge. Meanwhile, the activation editing performs the physical-knowledge-preserving computing in the output channels of PhNs. Thanks to the concurrent link and activation editing, the input/output of Phy-Taylor can strictly comply with the available physical knowledge. 
Using the notation `$\boxplus$' defined in the Table \ref{notation}, the procedure of physics-guided NN editing is described in Algorithm \ref{ALG2}. 

For the edited weight matrix $\mathbf{W}_{\left\langle t \right\rangle }$, if its entries in the same row are all known model-substructure parameters, the associated activation should be inactivate. Otherwise, the Phy-Taylor cannot strictly preserve the available physical knowledge due to the extra nonlinear mappings induced by the activation functions. This thus motivates the physics-knowledge preserving computing, i.e., the Line \ref{alg2-12} of Algorithm \ref{ALG2}. Figure \ref{flowchartonelayer} summarizes the flowchart of NN editing in a single PhN layer: 
\begin{itemize}
\vspace{-0.10in}
\item (i) Given the node-representation vector from Algorithm \ref{ALG1}, the original (fully-connected) weight matrix is edited via link editing to embed assigned physical knowledge, resulting in $\mathbf{W}_{\left\langle t \right\rangle }$. 
\vspace{-0.10in}
\item (ii) The edited weight matrix ${\mathbf{W}_{\left\langle t \right\rangle }}$ is separated into knowledge matrix $\mathbf{K}_{\left\langle t \right\rangle}$ and uncertainty matrix $\mathbf{U}_{\left\langle t \right\rangle }$, such that ${\mathbf{W}_{\left\langle t \right\rangle }} = \mathbf{K}_{\left\langle t \right\rangle} + \mathbf{U}_{\left\langle t \right\rangle }$. Specifically, the $\mathbf{K}_{\left\langle t \right\rangle}$,  generated in Lines \ref{alg2-3} and \ref{alg2-7}, includes all the known model-substructure parameters. While the $\mathbf{M}_{\left\langle t \right\rangle }$, generated in Lines \ref{alg2-4} and \ref{alg2-8}, is used to generate uncertainty matrix $\mathbf{U}_{\left\langle t \right\rangle }$ (see Line \ref{alg2-11}) to include all the unknown model-substructure parameters, through freezing the known model-substructure parameters of $\mathbf{W}_{\left\langle t \right\rangle}$ to zeros. 
\vspace{-0.10in}
\item (iii) The $\mathbf{K}_{\left\langle t \right\rangle}$, $\mathbf{M}_{\left\langle t  \right\rangle }$ and activation-masking vector $\mathbf{a}_{\left\langle t  \right\rangle }$ (generated in Lines \ref{alg2-5} and \ref{alg2-9}) are used by activation editing for the physical-knowledge-preserving computing of output in each PhN layer. The function of $\mathbf{a}_{\left\langle t  \right\rangle }$ is to avoid the extra mapping (induced by activation) that prior physical knowledge does not include. 
\end{itemize}

\begin{figure*}[!t]
\centering
\includegraphics[scale=0.53]{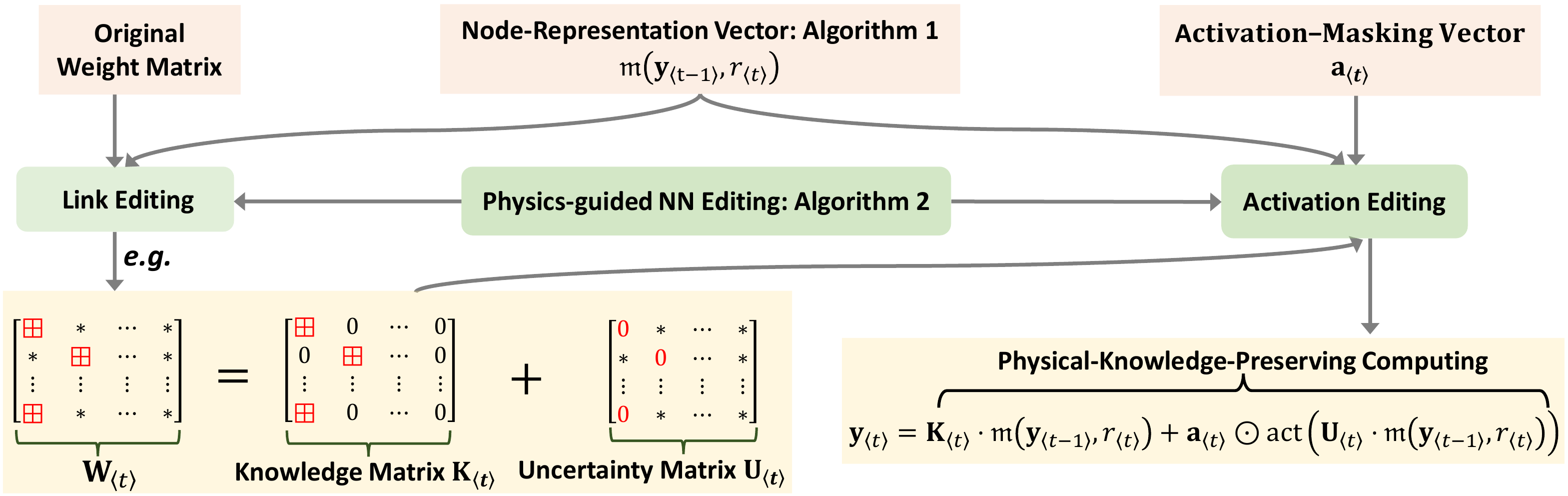}
\caption{Flowchart of NN editing in single PhN layer.}
\label{flowchartonelayer}
\end{figure*}
The flowchart of NN editing operating in cascade PhN is depicted in Figure \ref{exapplot}. The Lines \ref{alg2-2}-\ref{alg2-4} of Algorithm~\ref{ALG2} means that  
$\mathbf{A} = \mathbf{K}_{\left\langle 1 \right\rangle} + \mathbf{M}_{\left\langle 1   
                  \right\rangle } \odot \mathbf{A}$, leveraging which and the setting $r_{\left\langle 1 \right\rangle} = r$, the ground-truth model \eqref{eq:lobja} is rewritten as 
\begin{align}
\mathbf{y} = (\mathbf{K}_{\left\langle 1 \right\rangle} + \mathbf{M}_{\left\langle 1   
                  \right\rangle } \odot \mathbf{A}) \cdot \mathfrak{m}(\mathbf{x},r) + \mathbf{f}(\mathbf{x}) 
= \mathbf{K}_{\left\langle 1 \right\rangle} \cdot \mathfrak{m}(\mathbf{x}, r_{\left\langle 1 \right\rangle}) + (\mathbf{M}_{\left\langle 1   
                  \right\rangle } \odot \mathbf{A}) \cdot \mathfrak{m}(\mathbf{x},r_{\left\langle 1 \right\rangle}) + \mathbf{f}(\mathbf{x}). \label{exppf1}
\end{align}
\begin{figure*}[!t]
\centering
\subfigure{\includegraphics[scale=0.53]{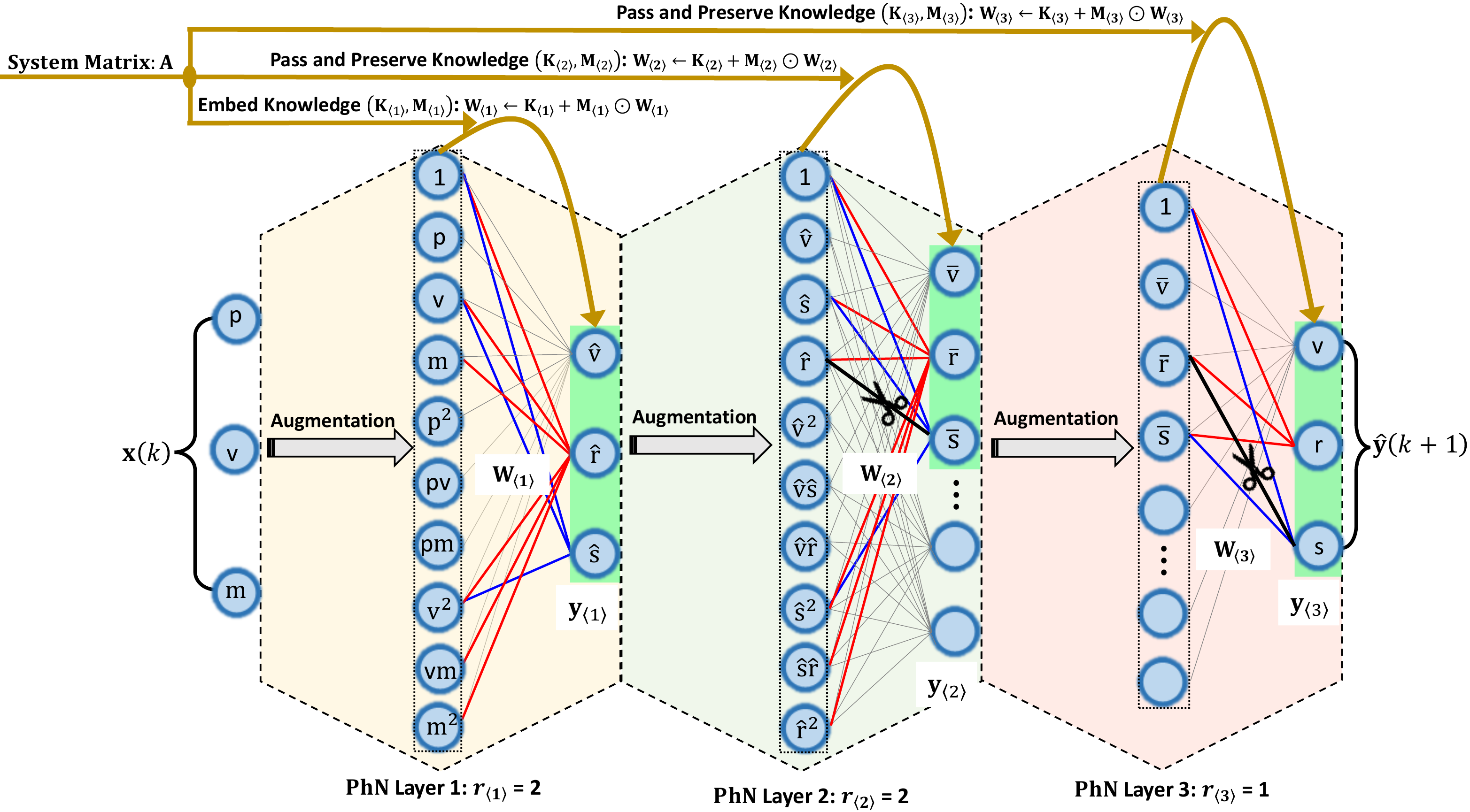}}
\caption{\textbf{Implementing NN editing, i.e., Algorithm~\ref{ALG2}, in the Example \ref{exap1}}. (i) Unknown substructures are formed by the grey links, the known substructures are formed by the red and blue links. (ii) Cutting black links to avoid spurious correlations, otherwise, the links introduce the dependence of $s\left( {k + 1} \right)$ on mass m, thus contradicting physical knowledge.}
\label{exapplot}
\end{figure*}
We obtain from the Line \ref{alg2-12} of Algorithm~\ref{ALG2} that the output of the first PhN layer is 
\begin{align}
\mathbf{y}_{\left\langle 1 \right\rangle } = \mathbf{K}_{\left\langle 1 \right\rangle } \cdot \mathfrak{m}(\mathbf{x}, r_{\left\langle 1 \right\rangle}) +  \mathbf{a}_{\left\langle 1 \right\rangle } \odot \text{act}\left( {\mathbf{U}_{\left\langle 1   
                  \right\rangle } \cdot {\mathfrak{m}} \left( {\mathbf{x},  r_{\left\langle 1 \right\rangle } } \right)} \right). \label{exppf2}
\end{align}
Recalling that $\mathbf{K}_{\left\langle 1 \right\rangle}$ includes all the known model-substructure parameters of $\mathbf{A}$ while the $\mathbf{U}_{\left\langle 1 \right\rangle }$ includes remainders, we conclude from \eqref{exppf1} and \eqref{exppf2} that the available physical knowledge pertaining to the ground-truth model \eqref{eq:lobja} has been embedded to the first PhN layer. As Figure \ref{exapplot} shows the embedded knowledge shall be passed down to the remaining cascade PhNs and preserved therein, such that the end-to-end Phy-Taylor model can strictly with the prior physical knowledge. This knowledge passing is achieved by the block matrix $\mathbf{K}_{\left\langle p \right\rangle}$ generated in Line \ref{alg2-7}, due to which, the output of $t$-th PhN layer satisfies 
 \begin{align}
[\mathbf{y}_{\left\langle t \right\rangle }]_{1:\text{len}(\mathbf{y})} = \underbrace{\mathbf{K}_{\left\langle 1 \right\rangle } \cdot \mathfrak{m}(\mathbf{x}, r_{\left\langle 1 \right\rangle})}_{\text{knowledge passing}} +  \underbrace{[\mathbf{a}_{\left\langle t \right\rangle } \odot \text{act}\!\left( {\mathbf{U}_{\left\langle t   
                  \right\rangle } \cdot {\mathfrak{m}}( {\mathbf{y}_{\left\langle t-1 \right\rangle },  r_{\left\langle t \right\rangle } })}  \right)]_{1:\text{len}(\mathbf{y})}}_{\text{knowledge preserving}}, ~~~\forall t \in \{2,3,\ldots,p\}. \label{exppf3}
 \end{align}
Meanwhile, the $\mathbf{U}_{\left\langle t \right\rangle } = \mathbf{M}_{\left\langle t \right\rangle } \odot {\mathbf{W}_{\left\langle t \right\rangle }}$ means the masking matrix $\mathbf{M}_{\left\langle t \right\rangle}$ generated in Line \ref{alg2-8} is to remove the spurious correlations in the cascade PhN, which is depicted by the cutting link operation in Figure \ref{exapplot}.

\subsection{The Solution to Problem \ref{problem}: Phy-Taylor} \label{sec:pintaylor}
Described in Figure \ref{PhyArt}, with the guidance of Taylor series, implementing the physics-guided NN editing in the deep PhN yields the \underline{Phy-Taylor}. The Phy-Taylor embeds the available physical knowledge into each PhN, such that its input/output strictly complies with the physical knowledge, which is formally stated in the following theorem. The theorem proof appears in Supplementary Information \ref{SI03}. 

\begin{thm} 
Consider the Phy-Taylor described by Figure \ref{PhyArt}. The input/output (i.e., ${\mathbf{x}}$/$\widehat{\mathbf{y}}$) of Phy-Taylor strictly complies with the available knowledge pertaining to the physics model \eqref{eq:lobja} of ground truth, i.e., if the $[\mathbf{A}]_{i,j}$ is a known model-substructure parameter, then $\frac{{\partial {{[\widehat{\mathbf{y}}}]_i}}}{{\partial {[\mathfrak{m}}\left( {\mathbf{x},r} \right)]_{j}}} \equiv \frac{{\partial {[\mathbf{y}]_i}}}{{\partial {[\mathfrak{m}}\left( {\mathbf{x},r} \right)]_{j}}} \equiv [\mathbf{A}]_{i,j}$ always holds. \label{thmmm2}
\end{thm}

\section{Phy-Taylor Properties} \label{sec:analysis}
\begin{figure*}[!t]
\centering
\subfigure{\includegraphics[scale=0.52]{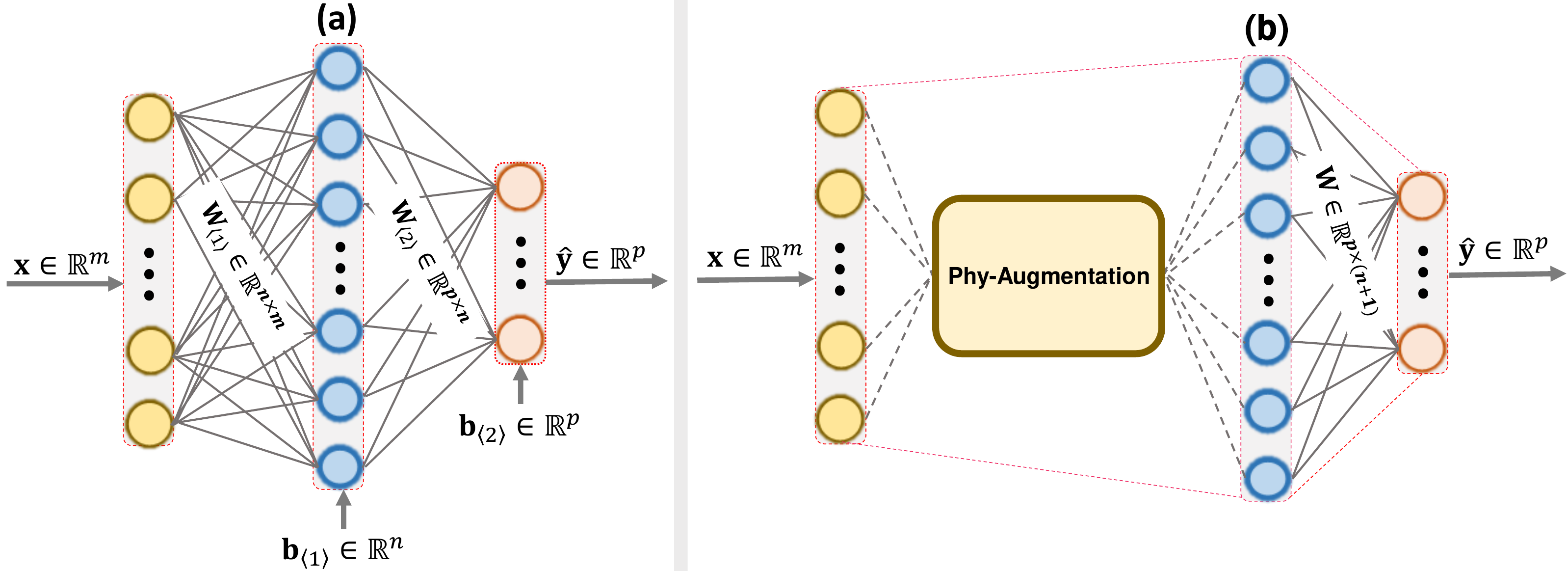}}
\caption{(a): Two Fully-connected NN layers. (b): A single PhN layer.}
\label{pint}
\end{figure*}
 
Moving forward, this section focuses on the property analysis of Phy-Taylor.
     
\subsection{Parameter Quantity Reduction}
The Figures \ref{PhyA} and \ref{PhyBB} show that the Phy-Augmentation, i.e., the Algorithm~\ref{ALG1}, expands the input without involving weight-matrix multiplication. This trait can be leveraged to significantly reduce the quantity of learning parameters including weights and bias. For the demonstration, we consider the two network models in Figure \ref{pint} (a) and (b), where the (a) describes a fully-connected two-layer network, while the (b) describes a single PhN. Observing them, we obtain that given the same dimensions of input and terminal output, the number of learning parameters of Figure \ref{pint} (a) is $(m+1)n + (n+1)p$ (including $(m + p)n$ weights and $n+p$ bias), while the number of learning parameters of Figure \ref{pint} (b) is $(n+1)p$ (including $n \times p$ weights and $p$ bias). The number difference of learning parameters is thus 
\begin{align}
(m+1)n + (n+1)p - (n+1)p = (m+1)n. \label{rs}
\end{align}  
We note the quantity of reduced parameters \eqref{rs} is the lower bound of PhN in Phy-Taylor framework, since it is obtained without considering physics-guided NN editing for removing and freezing links and bias according to the available physical knowledge. 

\subsection{Single PhN v.s. Cascade PhN}
We next investigate if the space complexity (i.e., the quantity of augmented monomials) of Phy-Augmentation of a single PhN with a large augmentation order can be reduced via cascade PhN with relatively small orders. To simplify the presentation, a single PhN and cascade PhN are respectively represented in the following equations.
\begin{align}
&\widehat{\mathbf{y}} = \mathrm{PhN}( {\left. {\mathbf{x}} \in \mathbb{R}^{n} \right| r}) \in \mathbb{R}^{m},\label{sTNO} \\
&{\mathbf{x}} \in {{\mathbb{R}}^n} ~~\longmapsto~~ {\mathbf{y}_{\left\langle 1\right\rangle }} = \mathrm{PhN}( {\left. {\mathbf{x}} \right|{r_{\left\langle 1 \right\rangle }}}) \in {{\mathbb{R}}^{{n_{\left\langle 1 \right\rangle }}}} ~~\longmapsto~~ \ldots ~~\longmapsto~~ {\mathbf{y}_{{\left\langle d-1\right\rangle }}} = \mathrm{PhN}( {\left. {{\mathbf{y}_{{\left\langle d-2 \right\rangle }}}} \right|{r_{{\left\langle d-1 \right\rangle }}}}) \in {{\mathbb{R}}^{{n_{{\left\langle d-1 \right\rangle }}}}} \nonumber\\
&\hspace{8.99cm} \longmapsto~~ \widehat{\mathbf{y}} = \mathrm{PhN}( {\left. {{\mathbf{y}_{{\left\langle d-1 \right\rangle }}}} \right|{r_{\left\langle d \right\rangle }}}) \in \mathbb{R}^{m},\label{cko}  
\end{align}
where the cascade architecture consists of $d$ PhNs. To guarantee the cascade PhN \eqref{cko} and the single PhN \eqref{sTNO} have the same monomials, their augmentation orders are required to satisfy
\begin{align}
\prod\limits_{v = 1}^d {{r_{{\left\langle v \right\rangle}}}} = r, ~~~~~~ \forall r_{{\left\langle v \right\rangle }}, ~~r \in \mathbb{N}. \label{mcc}
\end{align}

The space complexity difference of Phy-Augmentation is formally presented in the following theorem, whose proof is presented in Supplementary Information \ref{SI06}.
\begin{thm}
Under the condition \eqref{mcc}, the space complexity difference between single PhN \eqref{sTNO} and cascade PhN \eqref{cko}, due to Phy-Augmentation, is 
\begin{align}
\hspace{-0.5cm}\mathrm{len}(\mathfrak{m}(\mathbf{x},r)) - \sum\limits_{p = 1}^d {\mathrm{len}(\mathfrak{m}(\mathbf{x},r_{{\left\langle p \right\rangle }}))} = \sum\limits_{s = {r_{\left\langle 1 \right\rangle }} + 1}^r {\frac{{\left( {n + s - 1} \right)!}}{{\left( {n - 1} \right)!s!}}} - \sum\limits_{v = 1}^{d - 1} {\sum\limits_{s = 1}^{{r_{{\left\langle v+1 \right\rangle }}}} {\frac{{\left( {{n_{\left\langle v \right\rangle }} + s - 1} \right)!}}{{\left( {{n_{\left\langle v \right\rangle }} - 1} \right)!s!}}} } + 1 - d. \label{cffc}
\end{align}\label{cor}
\end{thm}

The Theorem \ref{cor} implies that the output dimensions and the augmentation orders of intermediate PhNs are critical in the significant reduction of space complexity via cascade PhN. 
However, an intuitive question arises: \textit{Does the cascade PhN reduce the complexity at the cost of model accuracy?} Without loss of generality, we use the following example to answer the question. 
\begin{ex}
For simplicity in explanation, we ignore the bias and consider the scenario that both the activation and the suppressor are inactive. For the single PhN \eqref{sTNO}, we let $\mathbf{x} \in \mathbb{R}^2$ and $\widehat{\mathbf{y}} \in \mathbb{R}$ and $r= 4$. Its output is then computed according to 
\begin{align}
&\widehat{\mathbf{y}} = {w_1}{[\mathbf{x}]_1} + {w_2}{[\mathbf{x}]_2} + {w_3}[\mathbf{x}]_1^2 + {w_4}{[\mathbf{x}]_1}{[\mathbf{x}]_2} + {w_5}[\mathbf{x}]_2^2 + {w_6}[\mathbf{x}]_1^3 + {w_7}[\mathbf{x}]_1^2{[\mathbf{x}]_2} + {w_8}{[\mathbf{x}]_1}[\mathbf{x}]_2^2 + {w_9}[\mathbf{x}]_2^3  \nonumber\\
&\hspace{4.83cm}+ {w_{10}}[\mathbf{x}]_1^4  + {w_{11}}[\mathbf{x}]_1^3{[\mathbf{x}]_2} + {w_{12}}[\mathbf{x}]_1^2[\mathbf{x}]_2^2 + {w_{13}}{[\mathbf{x}]_1}[\mathbf{x}]_2^3 + {w_{14}}[\mathbf{x}]_2^4, \label{cx1}
\end{align}
For the corresponding cascade PhN \eqref{cko}, we let $r_{{\left\langle 1 \right\rangle }} = r_{{\left\langle 2 \right\rangle }} = 2$, the output dimension of first PhN is 2 while dimension of terminal output is 1. We thus have 
\begin{align}
\widehat{\mathbf{y}} &= {{\hat w}_6}{\mathbf{y}_{\left\langle 1 \right\rangle }} + {{\hat w}_7}\mathbf{y}^2_{\left\langle 1 \right\rangle } = {{\hat w}_6}( {{{\tilde w}_1}{[\mathbf{x}]_1} + {{\tilde w}_2}{[\mathbf{x}]_2} + {{\tilde w}_3}[\mathbf{x}]_1^2 + {{\tilde w}_4}{[\mathbf{x}]_1}{[\mathbf{x}]_2} + {{\tilde w}_5}[\mathbf{x}]_2^2}) \nonumber\\
&\hspace{5.35cm} + {{\hat w}_7}{\left( {{{\tilde w}_1}{[\mathbf{x}]_1} + {{\tilde w}_2}{[\mathbf{x}]_2} + {{\tilde w}_3}[\mathbf{x}]_1^2 + {{\tilde w}_4}{[\mathbf{x}]_1}{[\mathbf{x}]_2} + {{\tilde w}_5}[\mathbf{x}]_2^2} \right)^2}. \label{cx2}
\end{align}
Observing \eqref{cx1} and \eqref{cx2}, we discover that (i) both the single and cascade architectures have the same monomials due to the satisfying condition \eqref{mcc}, and (ii) the single PhN layer has  
14 weight parameters (i.e., the ${w_1}$, ${w_2}$, $\dots$, ${w_{14}}$ in Equation \eqref{cx1}), while the cascade layers have only 7 weight parameters (i.e., the ${\hat{w}_1}$, $\ldots$, and $\hat{w}_7$ in Equation \eqref{cx2}) in total. Intuitively, we can conclude that \textit{if the reduced weights are associated with the links that contradict with physical knowledge, the cascade PhN can further increase model accuracy, otherwise, it can reduce the space complexity at the cost of model accuracy}. 
\end{ex}

\section{Extension: Self-Correcting Phy-Taylor}
\begin{wrapfigure}{r}{0.70\textwidth}
\vspace{-0.9cm}
  \begin{center}
    \includegraphics[width=0.70\textwidth]{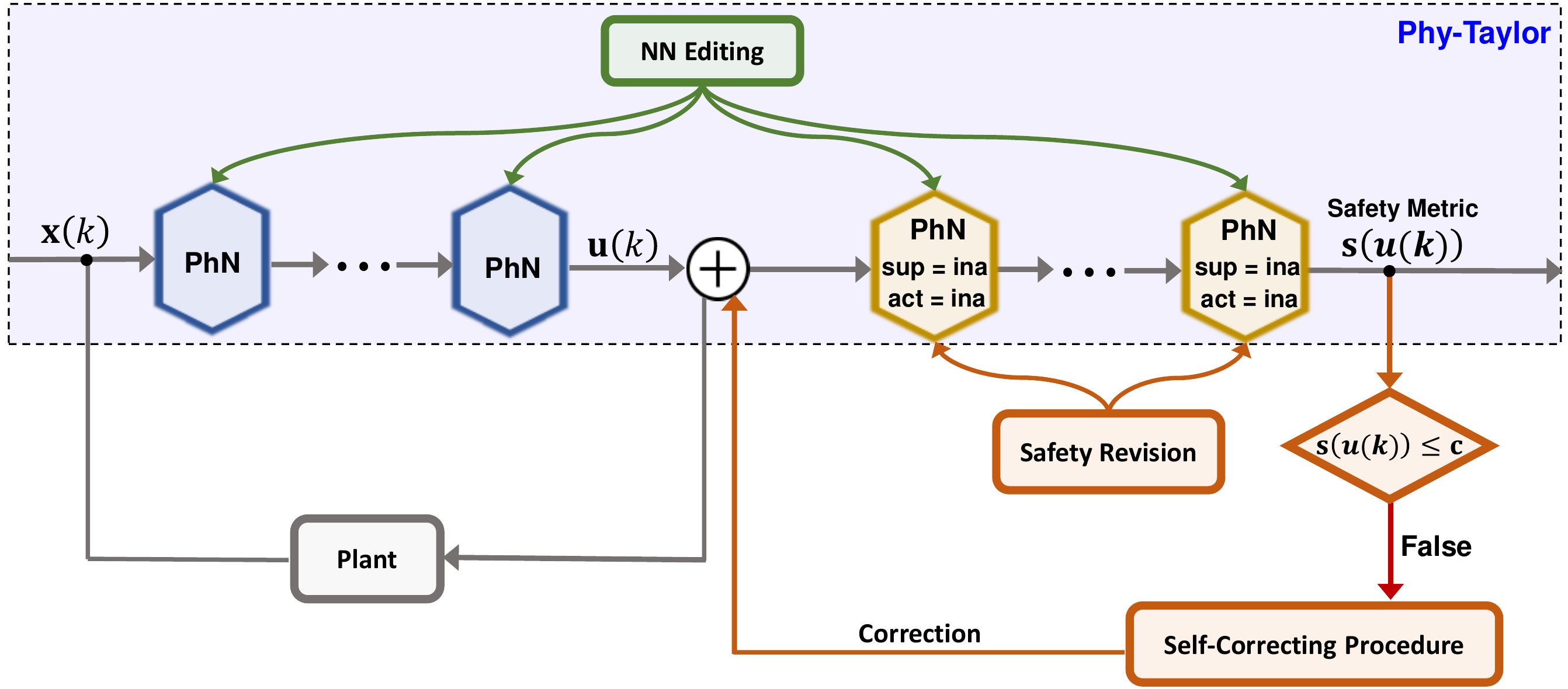}
  \end{center}
  \vspace{-0.5cm}
  \caption{Self-correcting Phy-Taylor Architecture: $\mathbf{u}(k)$ denotes the vector of real-time decisions, ${\bf{s}}(u(k))$ denotes the vector of real-time safety metrics, $\bf{c}$ is the vector of safety bounds.}
  \vspace{-0.5cm}
  \label{ghf}
\end{wrapfigure}
The recent incidents due to deployment of DNNs overshadow the revolutionizing potential of AI in the physical engineering systems \cite{arxr3, brief2021ai}, especially the safety-critical systems, whose unintended behavior results in death or serious injury to people, severe damage to equipment or environments \cite{arxr1,arxr2}. The safe control and planning is a fundamental solution for enhancing safety assurance of the AI-assisted physical systems often operating in environments where time and safety are critical, such as the airplanes, medical drones and autonomous vehicles. To comply with safety constraints in face of potential conflicts from control objectives, the framework of control barrier function (CBF) has been proposed for the computation of real-time safety-critical control commands \cite{ames2019control,singletary2022onboard}. The CBFs however use only current state information without prediction, whose control policy is thus greedy and challenging for proactive safe control. It is well known that model predictive control (MPC) yields a less greedy safe control policy, since it takes future state information into account \cite{williams2018information,falcone2007predictive}. Motivated by the observations, MPC with incorporation of CBF, i.e., MPC-CBF, was proposed \cite{zeng2021safety}. Due to the nonlinear dynamics, the MPC-CBF however faces a dilemma of prediction horizon and computation time of safe control commands, which induces considerable feedback delays and thus leads to failures in the time- and safety-critical operating environments. To address the dilemma, we propose the self-correcting Phy-Taylor, whose architecture is shown in Figure \ref{ghf}. Its one mission is learning the safety relationship between the real-time decisions and the safety metrics, with consideration of future information: 
\begin{align}
\mathbf{s}(\mathbf{x}(k),\mathbf{u}(k),\tau) = \sum\limits_{t = k}^{k + \tau - 1} {\tilde{\mathbf{f}}(\mathbf{x}(t),\mathbf{u}(t))} , \label{mkbzkm}
\end{align}
where $\tilde{\mathbf{f}}(\mathbf{x}(t),\mathbf{u}(t))$ is the predefined vector of safety metrics at time $t$, and $\tau$ denotes the horizon of future information of safety metrics. 

Inside the self-correcting Phy-Taylor, the learned safety relationship for approximating \eqref{mkbzkm} will first be subject to the off-line verification of available physical knowledge, based on which, the necessary revisions can be needed. According to the off-line verified and revised (if needed) safety relationship, the correcting of real-time decision $\mathbf{u}(k)$ will be triggered if any safety metric $[\mathbf{s}(\mathbf{x}(k),\mathbf{u}(k),\tau)]_{i}$, $i \in \{1,2,\ldots,h\}$, exceeds (or leaves) the preset safety bounds (or safety envelopes). However, the current learned formula corresponding to \eqref{mkbzkm} is not ready (if not impossible) for delivering the procedure, owning to the complicated dependence of $[\mathbf{s}(\mathbf{x}(k),\mathbf{u}(k),\tau)]_{i}$ on both system state $\mathbf{x}(k)$ and decision $\mathbf{u}(k)$. To address this problem, as shown in Figure \ref{ghf}, we decouple the real-time decisions from the real-time system states. Specifically, 
\begin{itemize}
\vspace{-0.1in}
\item Given the real-time system state $\mathbf{x}(k)$ as the origin input, the first Phy-Taylor outputs the real-time decision $\mathbf{u}(k)$, which is motivated by the fact that the state-feedback control is used most commonly in physical engineering systems \cite{doyle2013feedback}. In other words, the computation of raw $\mathbf{u}(k)$ directly depends on real-time system state $\mathbf{x}(k)$.
\vspace{-0.1in}
\item Given the real-time decision $\mathbf{u}(k)$ (i.e., the output of the first Phy-Taylor) as the input of the second Phy-Taylor, the terminal output is the real-time safety metric $\mathbf{s}(\mathbf{u}(k))$, which is motivated by the fact that the decision $\mathbf{u}(k)$ manipulates system state. In other words, the safety metric $\mathbf{s}(\mathbf{u}(k))$ directly depends on decision $\mathbf{u}(k)$ and indirectly depends on system state $\mathbf{x}(k)$.
\vspace{-0.1in}
\item The two Phy-Taylors are trained simultaneously according to the training loss function:
\begin{align}
\mathcal{L} = \alpha \left\| {{\bf{s}}(\mathbf{u}(k)) - {\bf{s}}({\bf{x}}(k),{\bf{u}}(k),\tau )} \right\| + \beta \left\| {\breve{\bf{u}}(k) - \bf{u}(k)} \right\|, \label{tranloss}
\end{align}
where the ${\bf{s}}({\bf{x}}(k),{\bf{u}}(k),\tau )$ given in Equation \eqref{mkbzkm} is ground truth of safety-metric vector,  the $\breve{\bf{u}}(k)$ is ground truth of decision vector, the $\alpha$ and $\beta$ are hyperparameters. The two cascade Phy-Taylors thus depend on each other.
\vspace{-0.1in}
\item To render the learned safety relationship  ${\bf{s}}(\mathbf{u}(k))$ tractable, the activation and compressor inside the second Phy-Taylor are inactive, such that the ${\bf{s}}(\mathbf{u}(k))$ is expressed in the form of Taylor series.
\end{itemize}

Given the verified and revised (if needed) relationship, the self-correcting procedure will be triggered (if a safety metric exceeds the safety bound $\mathbf{c}$) for correcting decision according to  
\begin{align}
\mathbf{u}(k) \leftarrow \mathop {\arg \min }\limits_{\widetilde{\mathbf{u}}(k) \in {\mathbb{R}^m}} \left\{ {\left. {\left\| {\widetilde{\mathbf{u}}(k) - \mathbf{u}(k)} \right\|} \right| [\bf{s}(\widetilde{\mathbf{u}}(k)]_i < [\mathbf{c}]_i,  ~~i \in \{1, 2, \ldots, \text{len}({\bf{s}}(\mathbf{u}(k)))\}} \right\}. \label{correctdecision}
\end{align}
We note the self-correcting mechanism and the safety revision of relationship between ${\bf{s}}(\mathbf{u}(k))$ and $\mathbf{u}(k)$  for delivering \eqref{correctdecision} vary with safety problems and physical systems. An example in this paper is the safe control of autonomous vehicles: Algorithm~\ref{ALG4} in the Experiments section.

\section{Experiments}\label{sec:exp}
The demonstrations are performed on three different systems with different degrees of available physical knowledge: autonomous vehicles, coupled pendulums, and a subsystem of US Illinois climate. 

\subsection{Autonomous Vehicles} \label{expav}
\begin{wrapfigure}{r}{0.50\textwidth}
\vspace{-0.9cm}
  \begin{center}
    \includegraphics[width=0.50\textwidth]{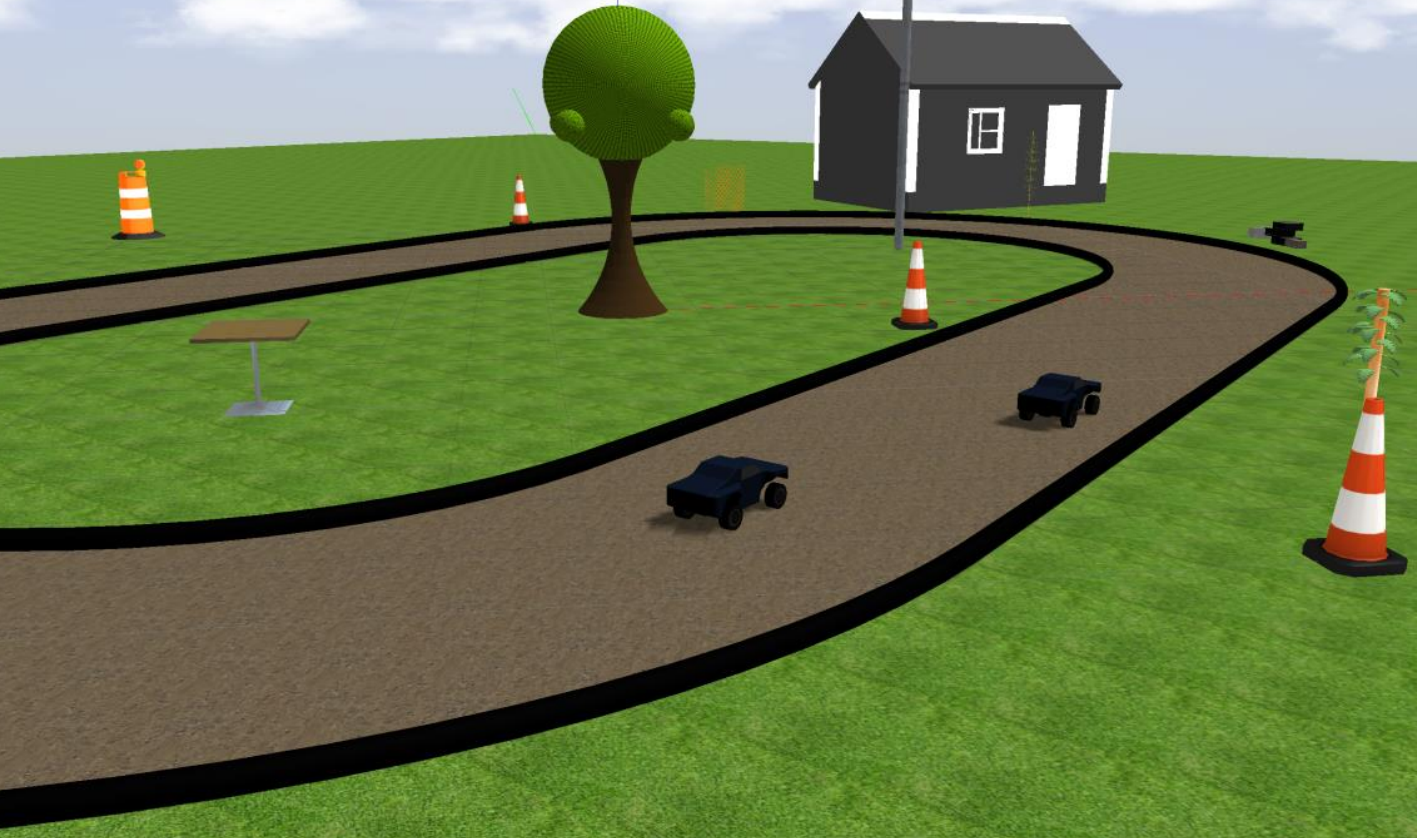}
  \end{center}
  \vspace{-0.5cm}
  \caption{Vehicle's driving environment.}
  \vspace{-0.5cm}
  \label{env}
\end{wrapfigure}
The first experiment performs the demonstration of two functions: (i) the learning of vehicle’s conjunctive lateral and longitudinal dynamics via Phy-Taylor, and (ii) the safe velocity regulation via self-correcting Phy-Taylor. The vehicle's driving environment is shown in Figure~\ref{env}, operating in the AutoRally platform \cite{goldfain2019autorally}.

\subsubsection{Vehicle Dynamics Learning} \label{expppgf}
We first leverage our available physical knowledge to identify the known model-substructure parameters for the physics-guided NN editing. According to the Newton’s second law for motion along longitudinal and lateral axes \cite{rajamani2011vehicle}, we have the following governing equations:
\begin{align}
\bar{m}\ddot {\mathrm{p}} = {F_{\mathrm{p}f}} + {F_{\mathrm{p}r}} - {F_{\text{aero}}} - {R_{\mathrm{p}f}} - {R_{\mathrm{p}r}}, \quad \quad \bar{m} (\ddot {\mathrm{y}} + \dot{\psi} {v_{\mathrm{p}}}) = {F_{\mathrm{y}f}} + {F_{\mathrm{y}r}}, \label{balancequation}
\end{align}
where $\mathrm{p}$ is the longitudinal position, $\mathrm{y}$ is the lateral position, $\psi$ is the vehicle yaw angle, $\bar{m}$ is the vehicle mass, $v_{\mathrm{p}} \triangleq \dot {\mathrm{p}}$ 
is the longitudinal velocity, ${F_{\mathrm{p}f}}$ and ${F_{\mathrm{p}r}}$ denote the longitudinal tire force at the front and rear tires, respectively, ${R_{\mathrm{p}f}}$ and ${R_{\mathrm{p}r}}$ denote the rolling resistance at the front and rear tires, respectively, ${F_{\text{aero}}}$ represents the longitudinal aerodynamic drag force, ${F_{\mathrm{y}f}}$ and ${F_{\mathrm{y}r}}$ are the lateral tire forces of the front and rear wheels, respectively. With the notations of lateral velocity $v_{\mathrm{y}} \triangleq \dot {\mathrm{y}}$ and yaw velocity $v_{\psi} \triangleq \dot{\psi}$, the following state space model is derived from the force balance equation \eqref{balancequation} in the literature \cite{rajamani2011vehicle}. 
\begin{align}
\frac{\mathrm{d}}{{\mathrm{d}t}} \left[ \begin{gathered}
\mathrm{p} \hfill \\
\mathrm{y} \hfill \\
\psi \hfill \\
{v_{\mathrm{p}}} \hfill \\
{v_{\mathrm{y}}} \hfill \\
{v_\psi } \hfill \\
\end{gathered} \right] = \left[ {\begin{array}{*{20}{c}}
0&0&0&1&0&0 \\
0&0&0&0&1&0 \\
0&0&0&0&0&1 \\
0&0&0&*&0&0 \\
0&0&0&0&*&* \\
0&0&0&0&*&*
\end{array}} \right]\underbrace{\left[ \begin{gathered}
\mathrm{p} \hfill \\
\mathrm{y} \hfill \\
\psi \hfill \\
{v_{\mathrm{p}}} \hfill \\
{v_{\mathrm{y}}} \hfill \\
{v_\psi } \hfill \\
\end{gathered} \right]}_{\triangleq {\mathbf{x}}} + \left[ \begin{gathered}
0 \hfill \\
0 \hfill \\
0 \hfill \\
* \hfill \\
0 \hfill \\
0 \hfill \\
\end{gathered} \right]\theta + \left[ \begin{gathered}
0 \hfill \\
0 \hfill \\
0 \hfill \\
0 \hfill \\
* \hfill \\
* \hfill \\
\end{gathered} \right]\delta, \label{statmodel}
\end{align}
where `*' can represent a state-dependent or time-dependent function or mixed of them or just a scalar, \underline{but is unknown to us}, and $\theta$ and $\delta$ denote throttle and steering, respectively. Given the practical physical knowledge that \textit{the throttle computation depends on the longitudinal velocity and position only, while the dependencies of steering are unknown}, the state space model \eqref{statmodel} updates with 
\begin{align}
\dot {{\mathbf{x}}} = \left[ {\begin{array}{*{20}{c}}
0&0&0&1&0&0 \\
0&0&0&0&1&0 \\
0&0&0&0&0&1 \\
*&0&0&*&0&0 \\
*&*&*&*&*&* \\
*&*&*&*&*&*
\end{array}} \right]{\mathbf{x}}. \nonumber
\end{align}
The sampling technique, with sampling period denoted by $T$, converts the continuous-time state-space model above to the following discrete-time one: 
\begin{align}
{\mathbf{x}}\left( {k + 1} \right) = \left[ {\begin{array}{*{20}{c}}
1&0&0&T&0&0 \\
0&1&0&0&T&0 \\
0&0&1&0&0&T \\
*&0&0&*&0&0 \\
*&*&*&*&*&* \\
*&*&*&*&*&*
\end{array}} \right]{\mathbf{x}}\left( k \right).\label{discar}
\end{align}

\begin{wrapfigure}{r}{0.55\textwidth}
\vspace{-0.9cm}
  \begin{center}
    \includegraphics[width=0.55\textwidth]{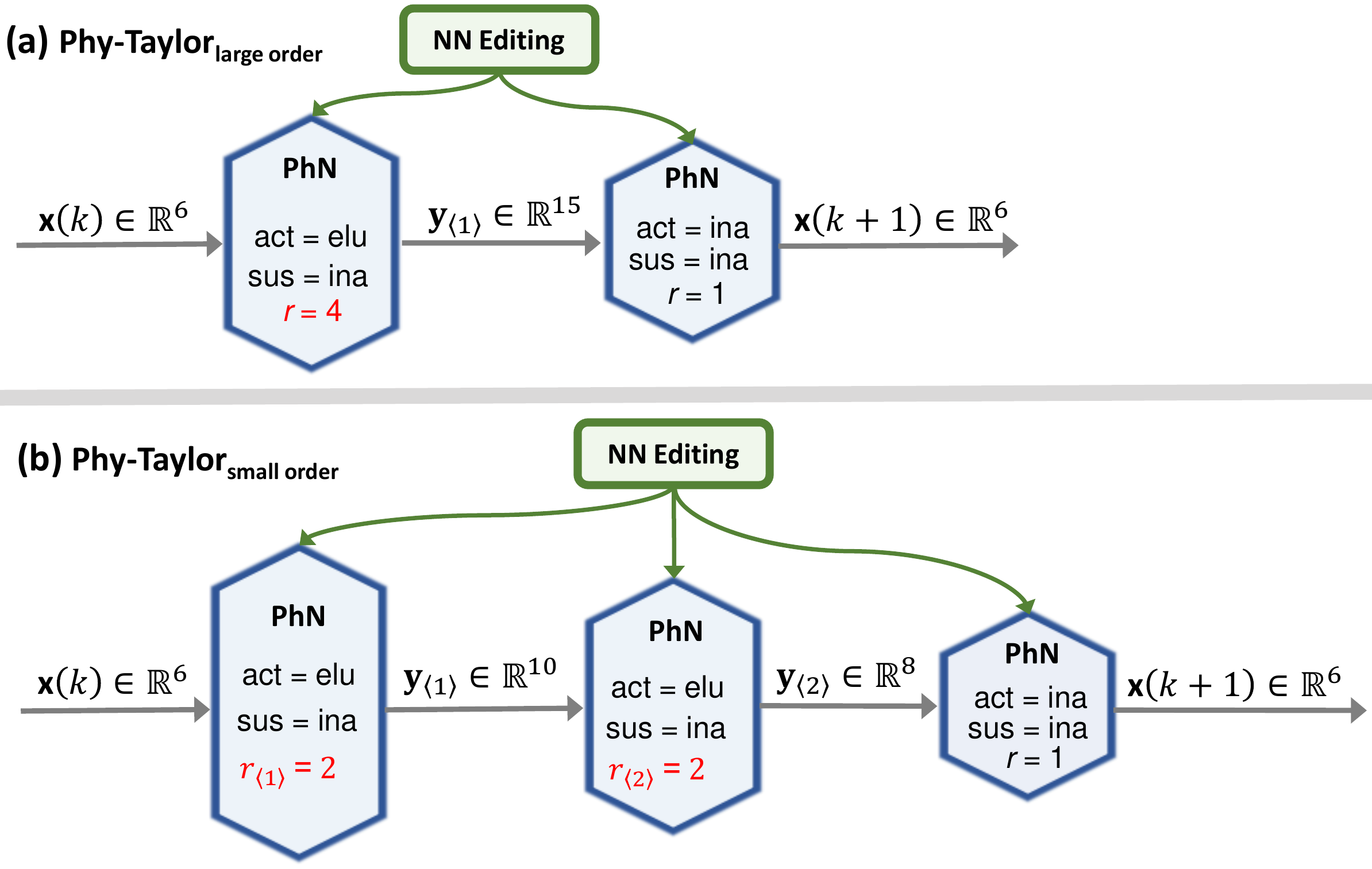}
  \end{center}
  \vspace{-0.5cm}
  \caption{(a): Phy-Taylor$_{\text{large order}}$ has a PhN with large order $r = 4$. (b): Phy-Taylor$_{\text{small order}}$ has cascading PhNs with relatively small orders satisfying $r_{\left\langle 1 \right\rangle} \cdot r_{\left\langle 2 \right\rangle} = 2 \cdot 2 = 4 = r$.}
  \vspace{-0.3cm}
  \label{capspoo}
\end{wrapfigure}
We first consider two Phy-Taylor models, named `Phy-Taylor$_{\text{large order}}$' and `Phy-Taylor$_{\text{small order}}$', which can embed the available knowledge (i.e., the known parameters included in system matrix of model \eqref{discar}). Their architectures are shown in Figure \ref{capspoo} (a) and (b). The Phy-Taylor$_{\text{large order}}$ has one PhN with a large augmentation order while the Phy-Taylor$_{\text{small order}}$ has two cascade PhN layers with two relatively small augmentation orders. Meanwhile, the three orders satisfy the condition \eqref{mcc} for having the same monomials of Taylor series. We also consider the corresponding models without NN editing (i.e., without physical knowledge embedding), which degrades the Phy-Taylor to the deep PhN (DPhN). The two DPhN models are named `DPhN$_{\text{large order}}$' and `DPhN$_{\text{small order}}$'. The final model we considered is the seminal Deep Koopman \cite{lusch2018deep}, following the same model configurations therein. The configurations of five models are summarized in Table \ref{taboo}.

\begin{table*}\footnotesize{
\caption{Model Configurations}
\centering
\begin{tabular}{l cc cc cc c c}
\toprule
 & \multicolumn{2}{c}{Layer 1}  & \multicolumn{2}{c}{Layer 2} & \multicolumn{2}{c}{Layer 3}\\
\cmidrule(lr){2-3} \cmidrule(lr){4-5} \cmidrule(lr){6-7}
Model ID     & \#weights    & \#bias  & \#Weights & \#bias  & \#weights  & \#bias & \#parameter sum & prediction error $e$\\
\midrule
DPhN$_{\text{large order}}$    &  $3135$   &    $15$   &   $90$      &     $6$    &   $-$    &   $-$   &   $3246$    &   nan\\
DPhN$_{\text{small order}}$   &  $270$     &    $10$   &   $520$    &     $8$    &   $48$   &   $6$   &   $862$     &   $45.57277$ \\
Phy-Taylor$_{\text{large order}}$   &  $2313$    &    $12$   &   $30$     &     $3$    &   $-$     &   $-$   &   $2358$   &   $0.047758$ \\
Phy-Taylor$_{\text{small order}}$  &  $167$     &    $7$    &   $265$   &     $5$    &   $18$   &   $3$    &    $465$     &   $0.003605$\\
\end{tabular}
\begin{tabular}{l cc cc cc c c}
\toprule
 & \multicolumn{2}{c}{Encoder}  & \multicolumn{2}{c}{Decoder} & \multicolumn{2}{c}{Auxiliary Networks}\\
\cmidrule(lr){2-3} \cmidrule(lr){4-5} \cmidrule(lr){6-7}
Model ID     & \#weights    & \#bias  & \#weights & \#bias  & \#weights  & \#bias & \#parameter sum & prediction error\\
\midrule
Deep Koopman   &  $2040$   &    $176$   &   $2040$      &     $176$    &   $19920$    &   $486$   &   $24838$    &  $0.232236$\\
\bottomrule
\end{tabular}\label{taboo}}
\end{table*}

\begin{figure*}[!t]
\centering
\subfigure{\includegraphics[scale=0.143]{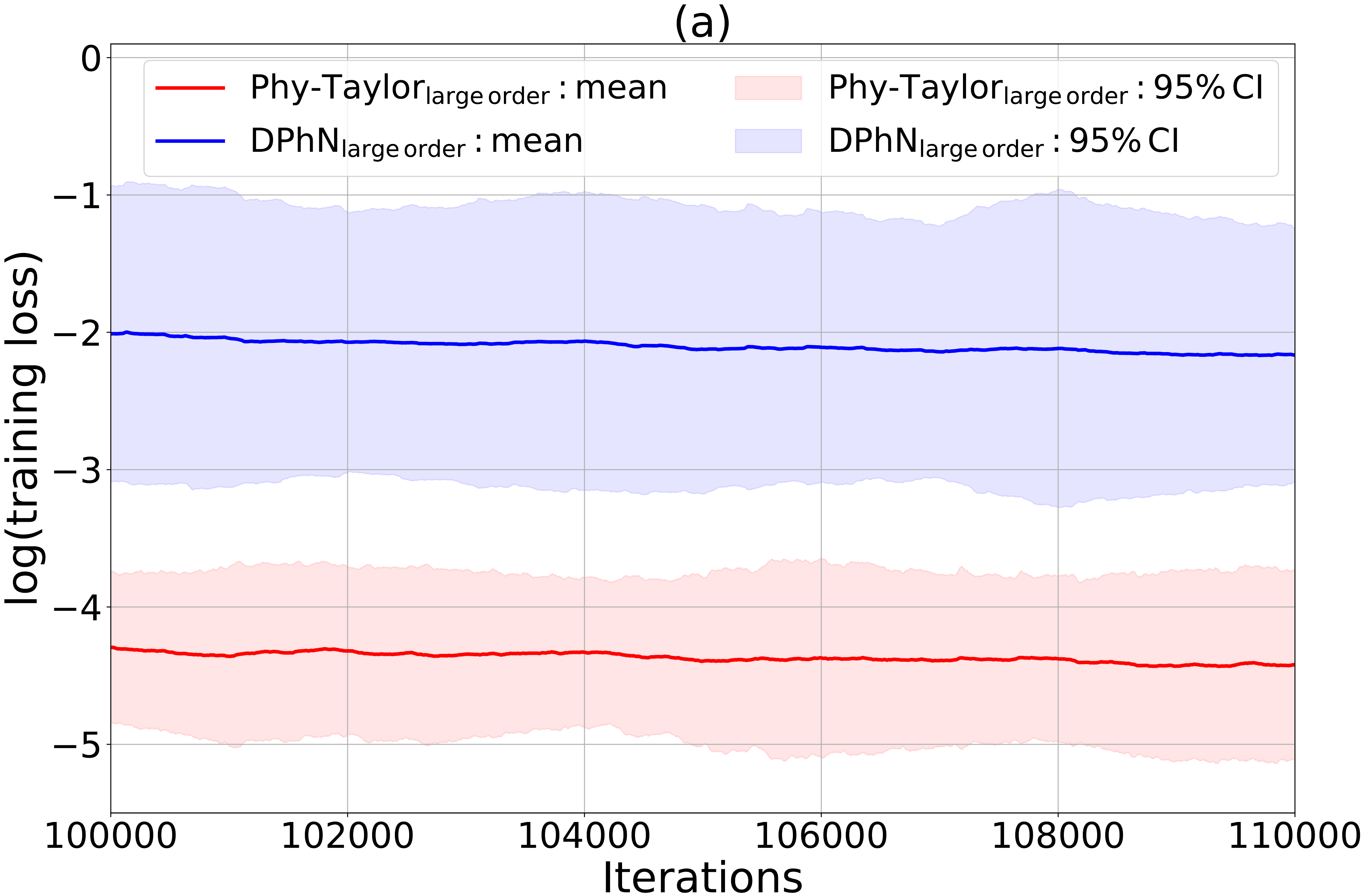}}
\subfigure{\includegraphics[scale=0.143]{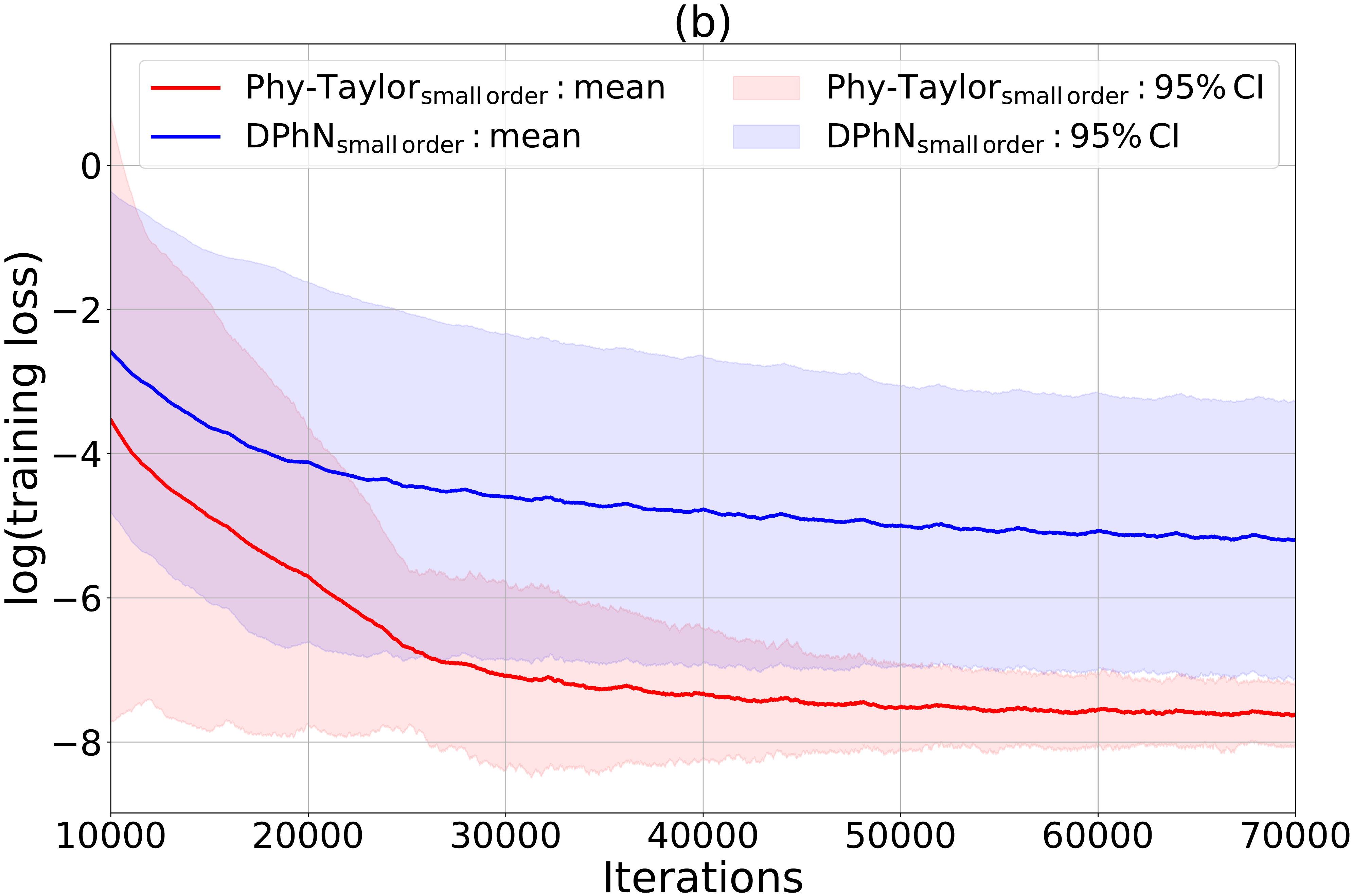}}
\subfigure{\includegraphics[scale=0.143]{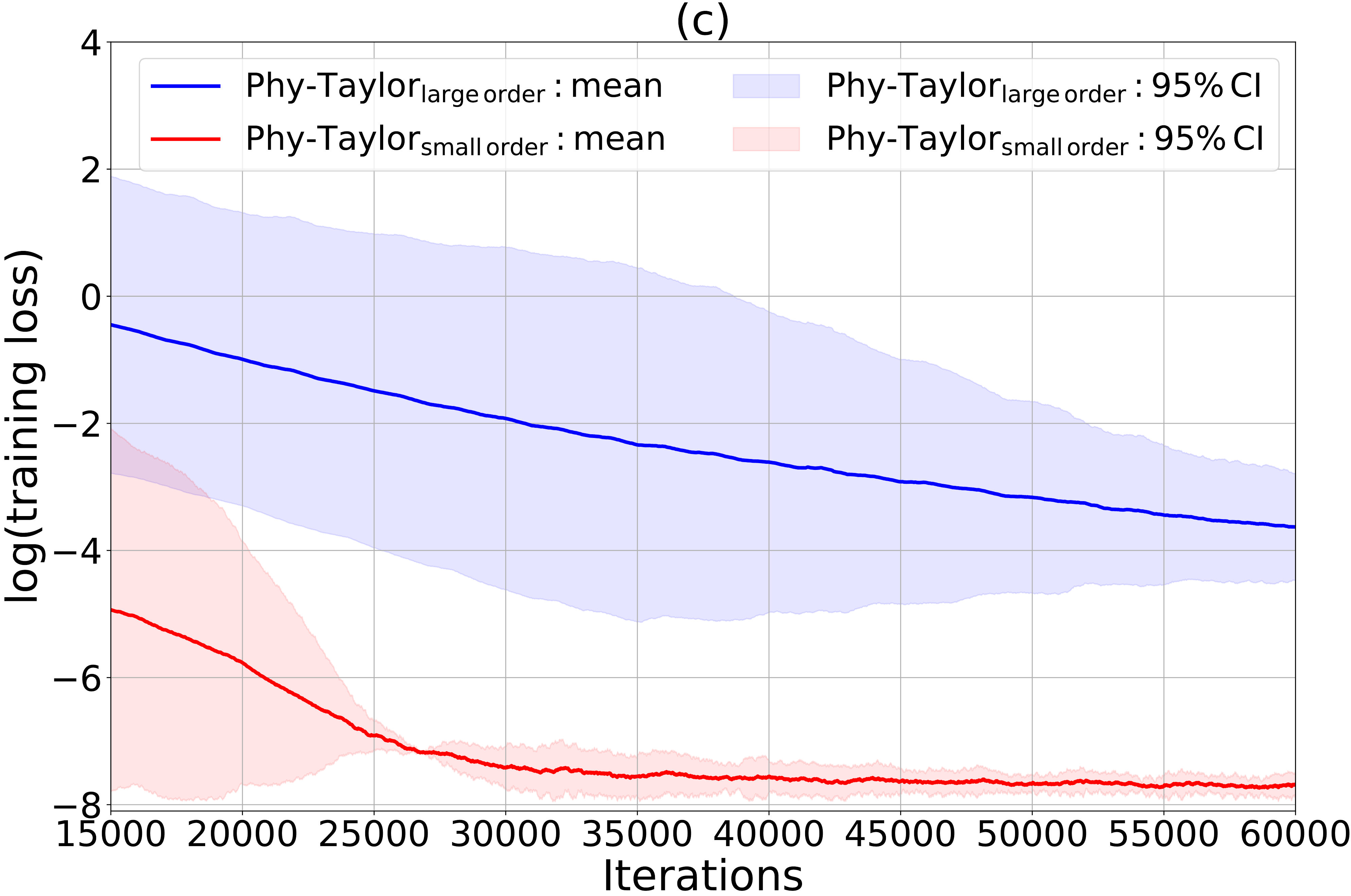}}
\subfigure{\includegraphics[scale=0.178]{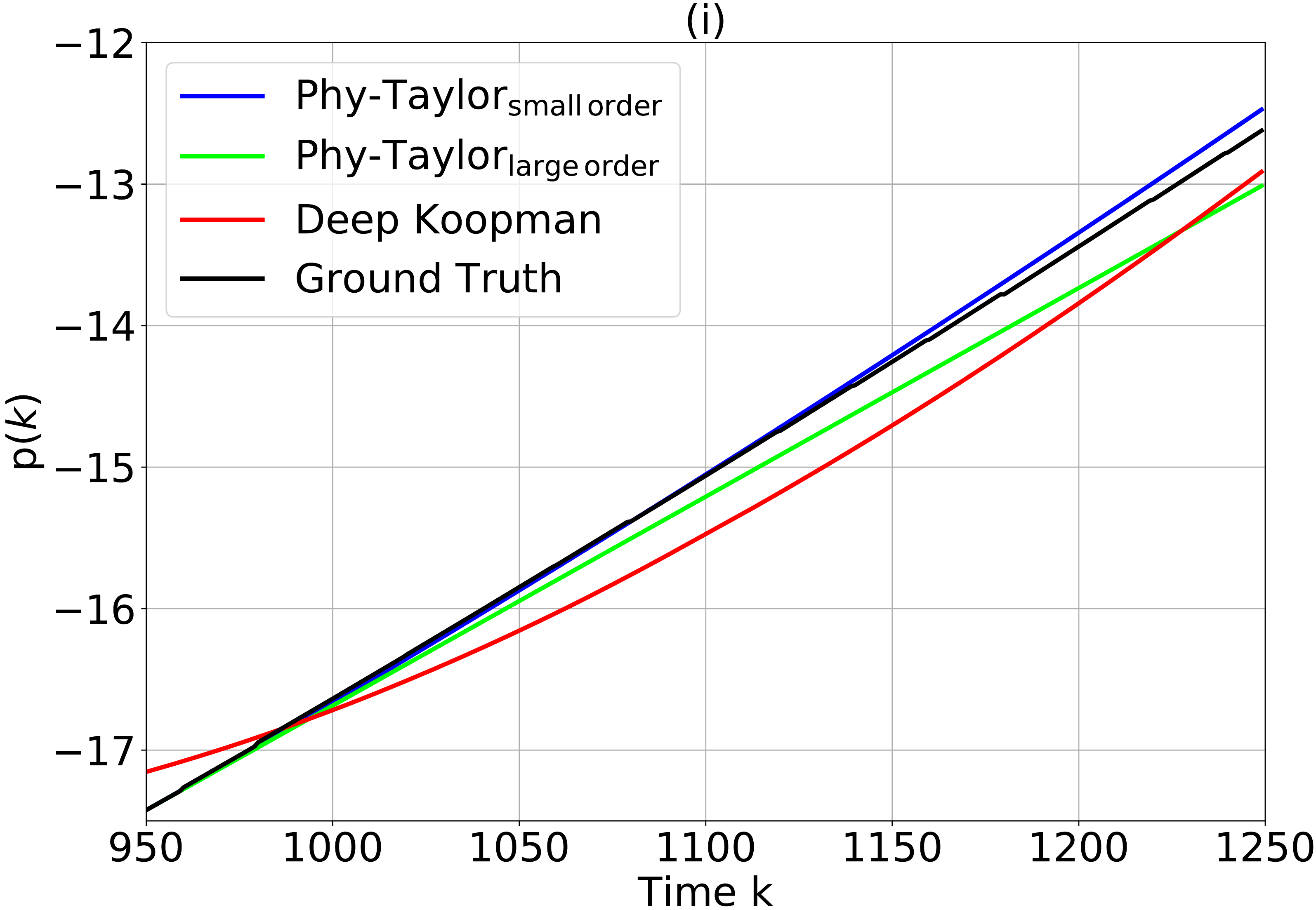}}
\subfigure{\includegraphics[scale=0.178]{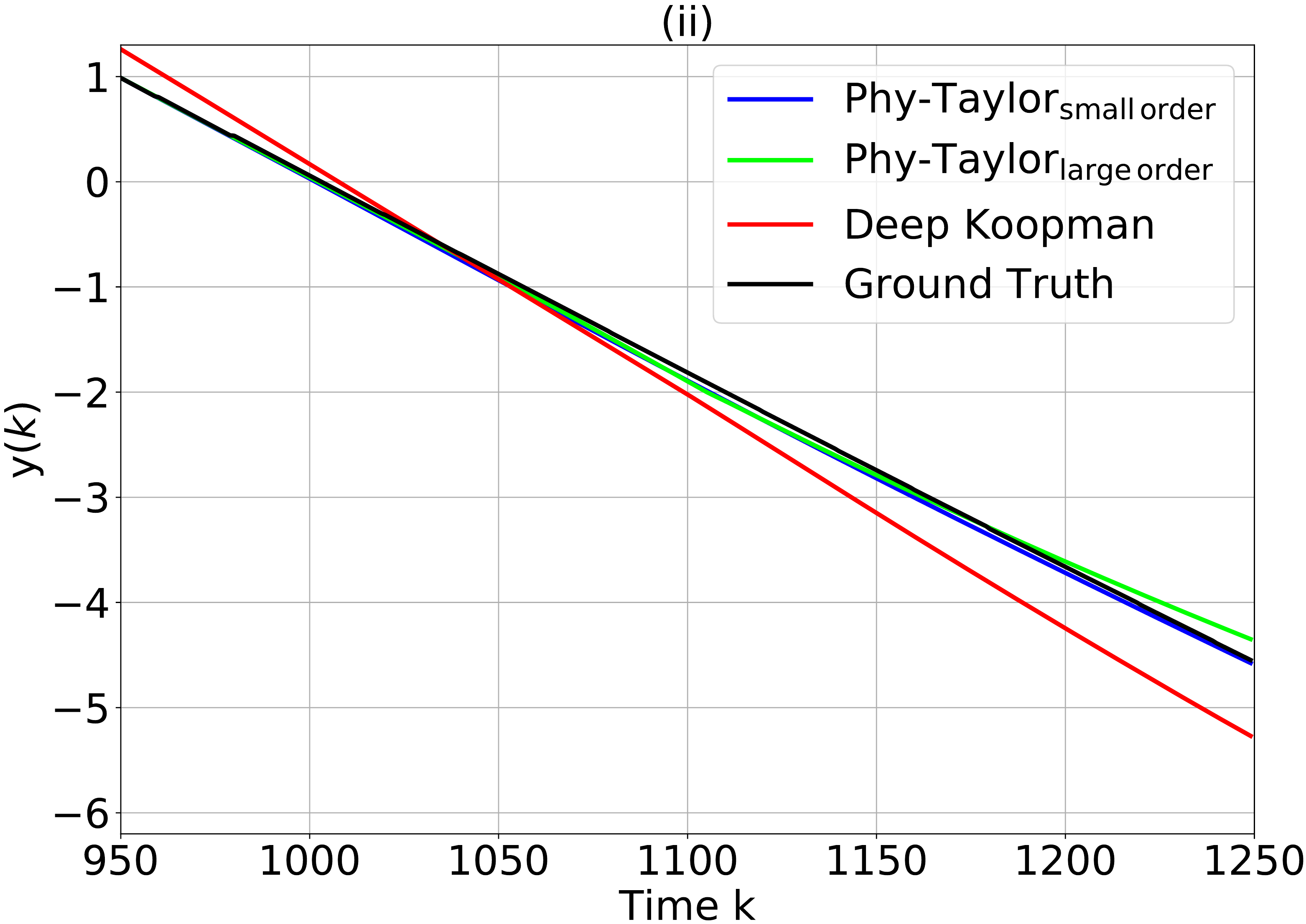}}
\subfigure{\includegraphics[scale=0.178]{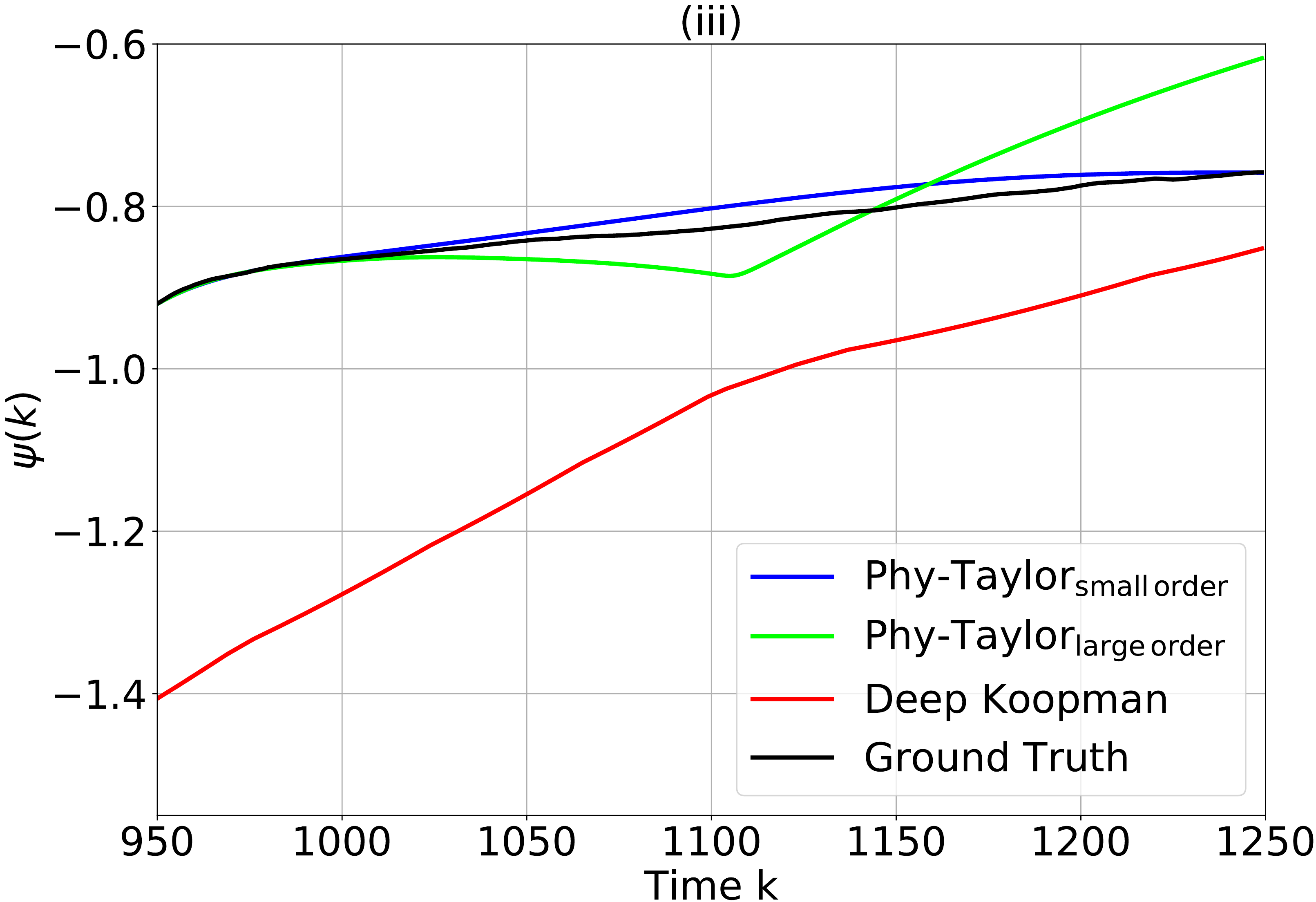}}
\subfigure{\includegraphics[scale=0.178]{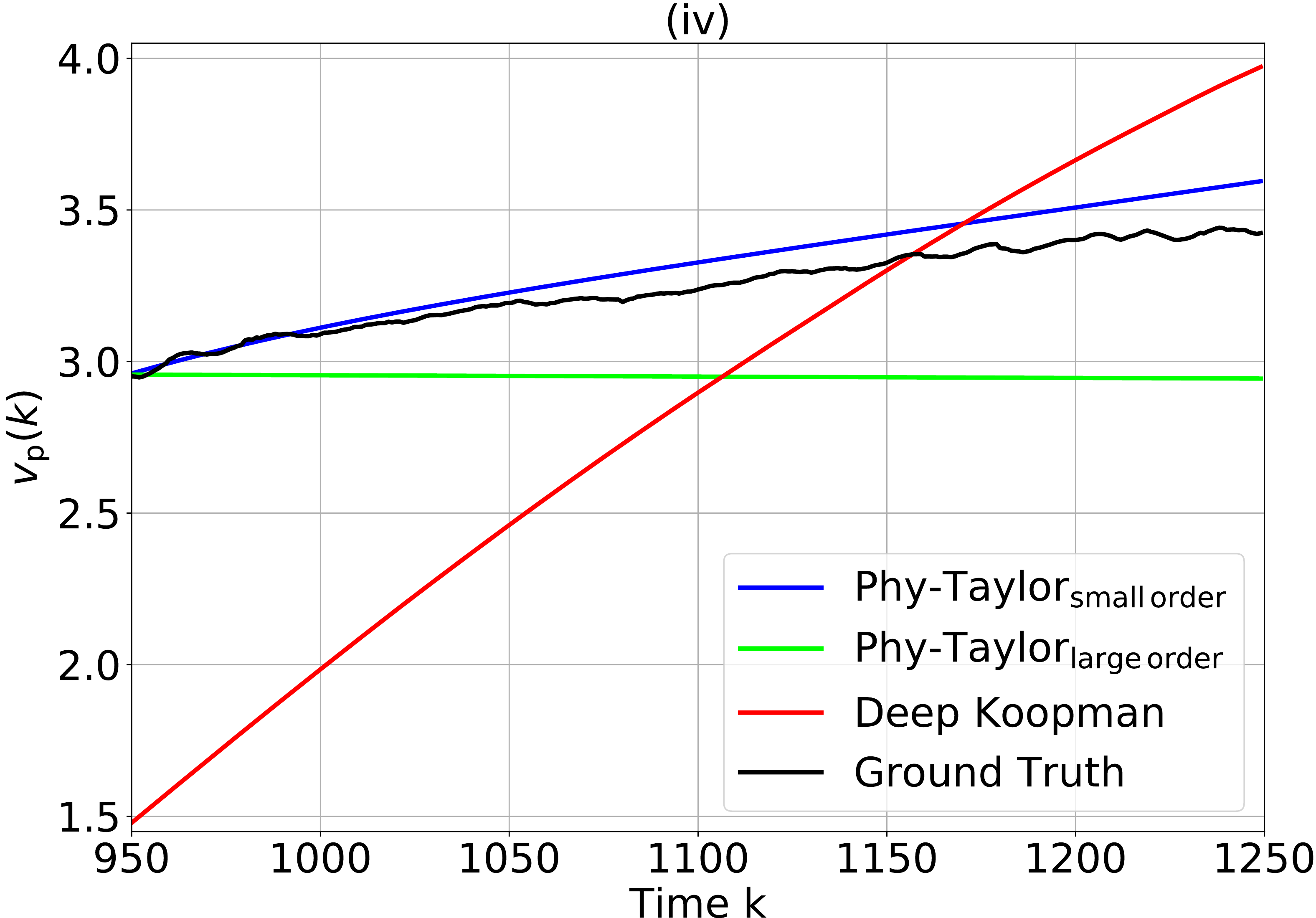}}
\subfigure{\includegraphics[scale=0.178]{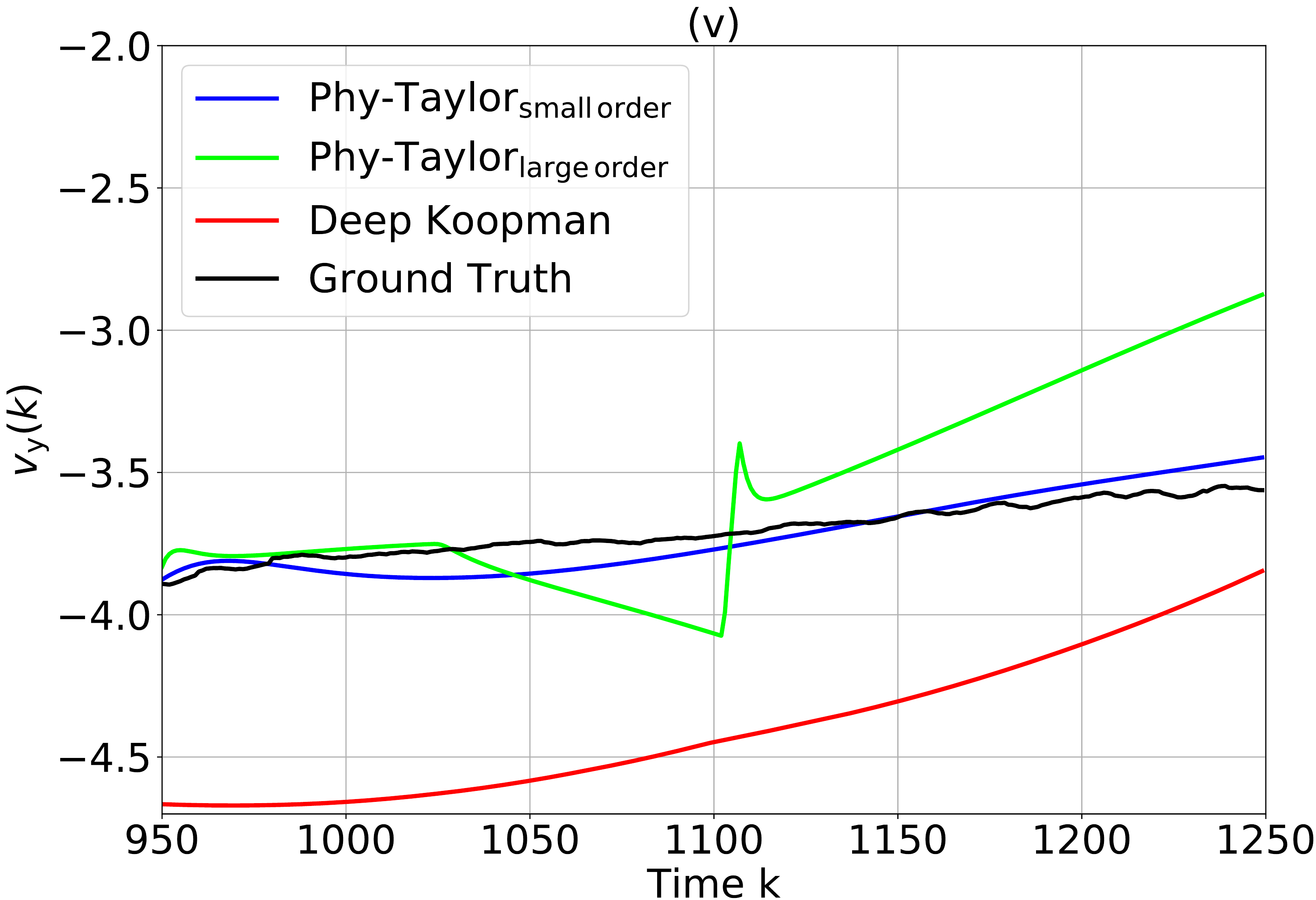}}
\subfigure{\includegraphics[scale=0.178]{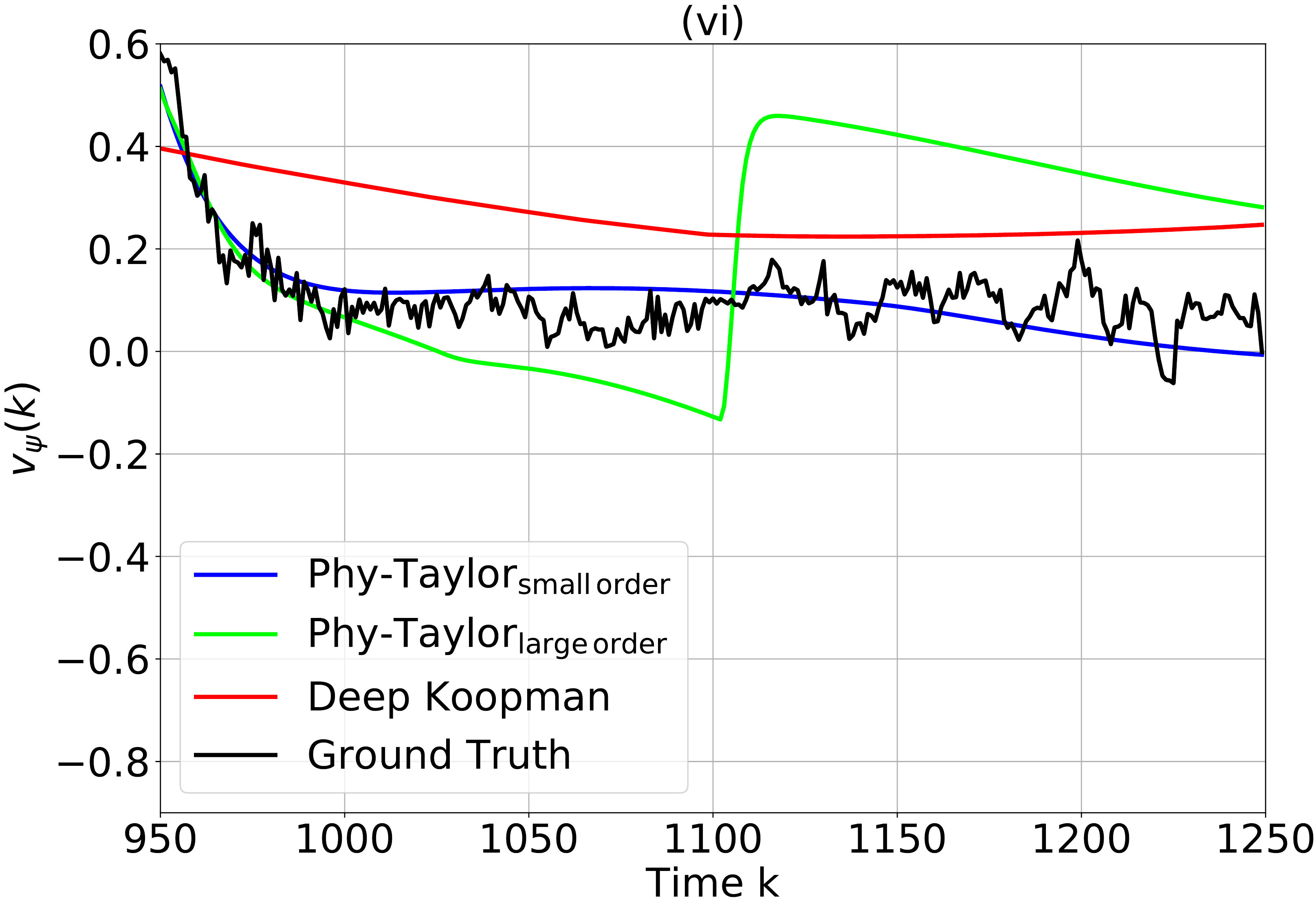}}
\caption{\textbf{Training and Testing}. (a)-(c): The trajectories of averaged training loss (5 random seeds) of different models described in Table \ref{taboo}. (i)-(vi): Ground truth and predicted trajectories via trained models.} 
\label{trvala}
\vspace{-0.4cm}
\end{figure*}

The trajectories of training loss are presented in Figure \ref{trvala} (a)--(c). The (training loss, validation loss) of trained DPhN$_{\text{large order}}$, DPhN$_{\text{small order}}$, Phy-Taylor$_{\text{large order}}$ and Phy-Taylor$_{\text{small order}}$ are (0.00389, 0.00375), (0.000344, 0.000351), (0.000222, 0.000238) and (0.000915, 0.000916), respectively. To perform the testing, we consider the long-horizon prediction of system trajectories, given the same initial conditions. The prediction error is measured by the mean squared error: $e = \frac{1}{\kappa }\sum\limits_{t = k + 1}^{k + \kappa } {\frac{1}{6}} \left\| {\underline{\bf{x}} \left( t \right) - {\bf{x}}\left( t \right)} \right\|~\text{with}~\underline {\bf{x}} \left( k \right) = {\bf{x}}\left( k \right)$, where $\underline {\bf{x}} \left( t \right) $ is the prediction of ground truth ${\bf{x}} \left( t \right)$ at time $t$. The prediction errors over the horizon $\kappa = 300$ and initial time $k = 950$ are summarized in Table \ref{taboo}. The ground-truth trajectories and predicted ones from Deep Koopman and the Phy-Taylor models are presented in Figure \ref{trvala} (i)--(vi). Observing from Table \ref{taboo} and Figure \ref{trvala}, we can conclude:
\begin{itemize}
\vspace{-0.1in}
\item Figure \ref{trvala} (i)--(vi) and Figure \ref{trvala} (a)--(b): the physics-guided NN editing can significantly accelerate model training, reduce validation and training loss as well as improve model accuracy (viewed from long-horizon prediction of trajectory).
\vspace{-0.1in}
\item Figure \ref{trvala} (i)--(vi) and Figure \ref{trvala} (c): with physics-guided NN editing, the cascade PhN with small augmentation orders can further significantly reduce training loss, and increase model accuracy. This can be due to the further removed spurious correlations or NN links contradicting with physics law, via the cascade architecture.
\vspace{-0.1in}
\item Figure \ref{trvala} (i)--(vi) and Table \ref{taboo}: compared with the fully-connected DPhN models, i.e., DPhN$_{\text{large order}}$ and DPhN$_{\text{small order}}$, the seminal Deep Koopman strikingly increases model accuracy, viewed from the perspective of long-horizon prediction of trajectory. Conversely, compared to Deep Koopman, the Phy-Taylor models (both Phy-Taylor$_{\text{large order}}$ and Phy-Taylor$_{\text{small order}}$) notably reduce the model learning parameters (weights and bias) and further remarkably increase model accuracy simultaneously. 
\end{itemize}

\subsubsection{Self-Correcting Phy-Taylor} 
This experiment demonstrates the effectiveness of self-correcting Phy-Taylor in guaranteeing vehicle's safe driving in the environment shown in Figure \ref{env}. The architecture of self-correcting Phy-Taylor is presented in Figure \ref{capsp}. Its real-time input vector is 
${\mathbf{x}}(k) = [w(k); ~\mathrm{p}(k); ~\mathrm{y}(k); ~\psi(k); ~v_{\mathrm{p}}(k); ~v_{\mathrm{y}}(k); v_{\psi}(k)]$, 
where $w(k)$ is the average of four wheels' velocities. The mediate output $\mathbf{u}(k) = \left[\theta(k);~ \gamma(k) \right]$ denotes the vector of control commands, where $\theta(k) \in [-0.156, ~0.156]$ is the throttle command and $\gamma(k) \in [-0.6, ~0.6]$ is the steering command. The considered safety-metric vector in Equation \eqref{mkbzkm} is 
\begin{align}
\mathbf{s}({\mathbf{x}}(k),\mathbf{u}(k),\tau) = \sum\limits_{t = k + 1}^{k + \tau } {\left[ {{{\left( {{v_{\mathrm{p}}}(t) - \mathrm{v}} \right)}^2};~~{{\left( {{v_{\mathrm{p}}}(t) - r \cdot w(k)} \right)}^2}} \right]}  \in {\mathbb{R}^2}, \label{hho} 
\end{align}
\begin{wrapfigure}{r}{0.80\textwidth}
\vspace{-0.9cm}
  \begin{center}
    \includegraphics[width=0.80\textwidth]{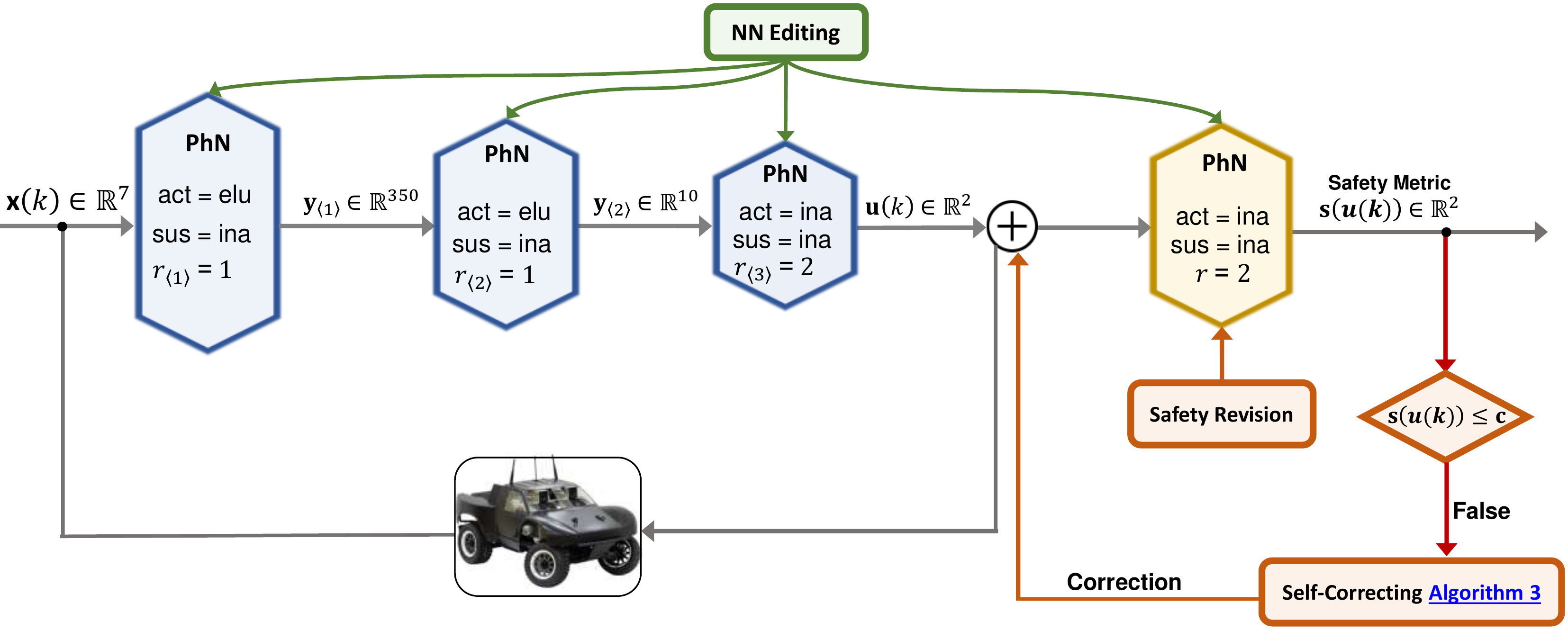}
  \end{center}
  \vspace{-0.0cm}
  \caption{Self-correcting Phy-Taylor for safe control of autonomous vehicle.}
  \vspace{-0.0cm}
  \label{capsp}
\end{wrapfigure}
where $\mathrm{v}$ and $r$ denote the reference of longitudinal velocity and the wheel radius, respectively. The safety metric \eqref{hho} together with the driving scenario in Figure \ref{env} indicate the objective of safe control command is to simultaneously steer vehicle's longitudinal velocity to reference $\mathrm{v}$, and constrain the slip (i.e., $(v_{\mathrm{x}}(t) - r \cdot w(k))^2$) to prevent slipping and sliding. The hyperparameters in the training loss function \eqref{tranloss} are set to $\alpha = \beta = 1$.

The testing results of trained model are presented in Figure \ref{inff} (a)--(d) (blue curves), which in conjunction with the ground truth (orange curves) demonstrate the trustfulness of trained model. We next output the learned safety relationships for off-line verification and necessary revision:
\begin{subequations}
\begin{align}
[{\bf{s}}({\bf{u}}(k))]_1 &=  0.00111007 + {\left[\! \begin{array}{l}
\theta(k)\\
\gamma(k)
\end{array} \!\right]^\top}   \left[ {\begin{array}{*{20}{c}}
{-0.04581441}&{0.00100625}\\
{0.00100625}&{0.00342825}
\end{array}} \right] \left[ \begin{array}{l}
\theta(k)\\
\gamma(k)
\end{array} \right], \label{choo1}\\
[{\bf{s}}({\bf{u}}(k))]_2 &=  0.14376973 - {\left[ \begin{array}{l}
\theta(k)\\
\gamma(k)
\end{array} \right]^\top}\left[ {\begin{array}{*{20}{c}}
{6.06750536}&{ 0.02701398}\\
{0.02701398}&{0.00601609}
\end{array}} \right]\left[\begin{array}{l}
\theta(k)\\
\gamma(k)
\end{array} \right].\label{choo2}
\end{align}\label{choo}
\end{subequations}
 
The safety metrics \eqref{hho} of ground truth are always non-negative. We thus need to verify that given the ranges of control commands (i.e., $\theta(k) \in [-0.156, 0.156]$, $\gamma(k) \in [-0.6, 0.6]$, $\forall k \in \mathbb{N}$), if both 
$[{\bf{s}}({\bf{u}}(k))]_1$ and $[{\bf{s}}({\bf{u}}(k))]_2$ in Equation \eqref{choo} can always be non-negative. If a violation occurs, we will make revisions to the relationships. We can verify from \eqref{choo} that the non-negativity constraint does not always hold, such as $\theta(k) = 0.156$ and $\gamma(k) = 0$. Therefore, the revision of safety relationship is needed before working on the self-correcting procedure. The regulated safety relationships (revisions are highlighted in \textcolor[rgb]{1.00,0.00,0.00}{red} color) are presented below. 
\begin{subequations}
\begin{align}
[{\bf{s}}({\bf{u}}(k))]_1 &=  \underbrace{0.00\textcolor[rgb]{1.00,0.00,0.00}{02}1007}_{{[\mathbf{b}]_1}} + {\left[ \begin{array}{l}
\theta \left( k \right)\\
\gamma \left( k \right)
\end{array} \right]^\top}   \underbrace{\left[ {\begin{array}{*{20}{c}}
{\textcolor[rgb]{1.00,0.00,0.00}{0.0018}1441}&{0.00100625}\\
{0.00100625}&{0.00342825}
\end{array}} \right]}_{{\triangleq {\bf{P}}_1}}    \left[ \begin{array}{l}
\theta \left( k \right)\\
\gamma \left( k \right)
\end{array} \right], \label{rchoo1}\\
[{\bf{s}}({\bf{u}}(k))]_2 &=  \underbrace{0.14376973}_{{[\mathbf{b}]_2}} - {\left[ \begin{array}{l}
\theta \left( k \right)\\
\gamma \left( k \right)
\end{array} \right]^\top}\underbrace{\left[ {\begin{array}{*{20}{c}}
{\textcolor[rgb]{1.00,0.00,0.00}{5.90769724}}&{ \textcolor[rgb]{1.00,0.00,0.00}{0.01201398}}\\
{\textcolor[rgb]{1.00,0.00,0.00}{0.01201398}}&{0.00601609}
\end{array}} \right]}_{\triangleq {\bf{P}}_2}\left[ \begin{array}{l}
\theta \left( k \right)\\
\gamma \left( k \right)
\end{array} \right],\label{rchoo2}
\end{align}\label{rchoo}
\end{subequations}
\!\!which satisfy $\left[ {\bf{s}}({\bf{u}}(k)) \right]_1 \ge 0$ and $\left[ {\bf{s}}({\bf{u}}(k)) \right]_2 \ge 0$, for any $\theta(k) \in [-0.156, 0.156]$ and $\gamma(k) \in [-0.6, 0.6]$, and can be demonstrated by the green curves in Figure \ref{inff} (c) and (d). 


\begin{algorithm} \footnotesize{ 
\caption{Self-Correcting Procedure for Safe Control Commands}  \label{ALG4}
\KwIn{Real-time control-command vector ${\bf{u}}(k) = [\theta(k);~\gamma(k)]$, safety bounds $[\mathbf{c}]_1$ and $[\mathbf{c}]_2$, and learned matrices $\mathbf{P}_1$ and $\mathbf{P}_2$ and bias $[\mathbf{b}]_1$ and $[\mathbf{b}]_2$ defined in Equation \eqref{rchoo}.}
Update original safety relationship with off-line verified and revised one: ${\bf{s}}(\bf{u}(k)) \leftarrow \eqref{rchoo}$; \label{ALG4-1}\\
\eIf{$[{\bf{s}}({\bf{u}}(k))]_1 > [\mathbf{c}]_1$ or $[{\bf{s}}({\bf{u}}(k))]_2 > [\mathbf{c}]_2$ \label{ALG4-2}}
{\eIf{$[{\bf{s}}({\bf{u}}(k))]_{i} \ge [\mathbf{c}]_i$,  $i \in \{1,2\}$ \label{ALG4-3}}
{Update safety metric: $[\widehat{\mathbf{c}}]_{i} \leftarrow [\mathbf{c}]_i, i \in \{1,2\}$; \label{ALG4-4}}
{Maintain safety metric: $[\widehat{\mathbf{c}}]_{i} \leftarrow [{\bf{s}}({\bf{u}}(k))]_{i}, i \in \{1,2\}$; \label{ALG4-6}}
Compute orthogonal matrix $\mathbf{P}_1$ and eigenvalues $\lambda_1$ and $\lambda_1$ according to \eqref{rch1}; \label{ALG4-8}\\
Compute matrix: $\mathbf{S} \leftarrow \mathbf{Q}_1 \cdot \mathbf{P}_2  \cdot \mathbf{Q}_1$; \label{ALG4-9}\\
Compute $\widehat{\theta}(k)$ and $\widehat{\gamma}(k)$ according to \eqref{rch43}; \label{ALG4-10}\\
Correct real-time control commands: \begin{align}
\hspace{-1cm}\theta(k) \leftarrow \mathop {\arg \min }\limits_{\left\{ {\widehat \theta(k), - \widehat \theta(k)} \right\}} \left\{ {|{\theta(k) - \widehat \theta (k)}|,| {\theta(k) + \widehat \theta(k)}|} \right\}, ~~\gamma(k) \leftarrow \mathop {\arg \min }\limits_{\left\{ {\widehat \gamma(k), - \widehat \gamma(k)} \right\}} \left\{ {| {\gamma(k) - \widehat \gamma(k)}|, | {\gamma(k) + \widehat \gamma(k)}|} \right\}.\nonumber
\end{align}\label{ALG4-12}}
{Maintain real-time control commands: $\theta(k) \leftarrow \theta(k)$ and $\gamma(k) \leftarrow \gamma(k)$.}}
\end{algorithm}

\begin{figure*}
\centering
\vspace{-0.10cm}
\subfigure{\includegraphics[scale=0.35]{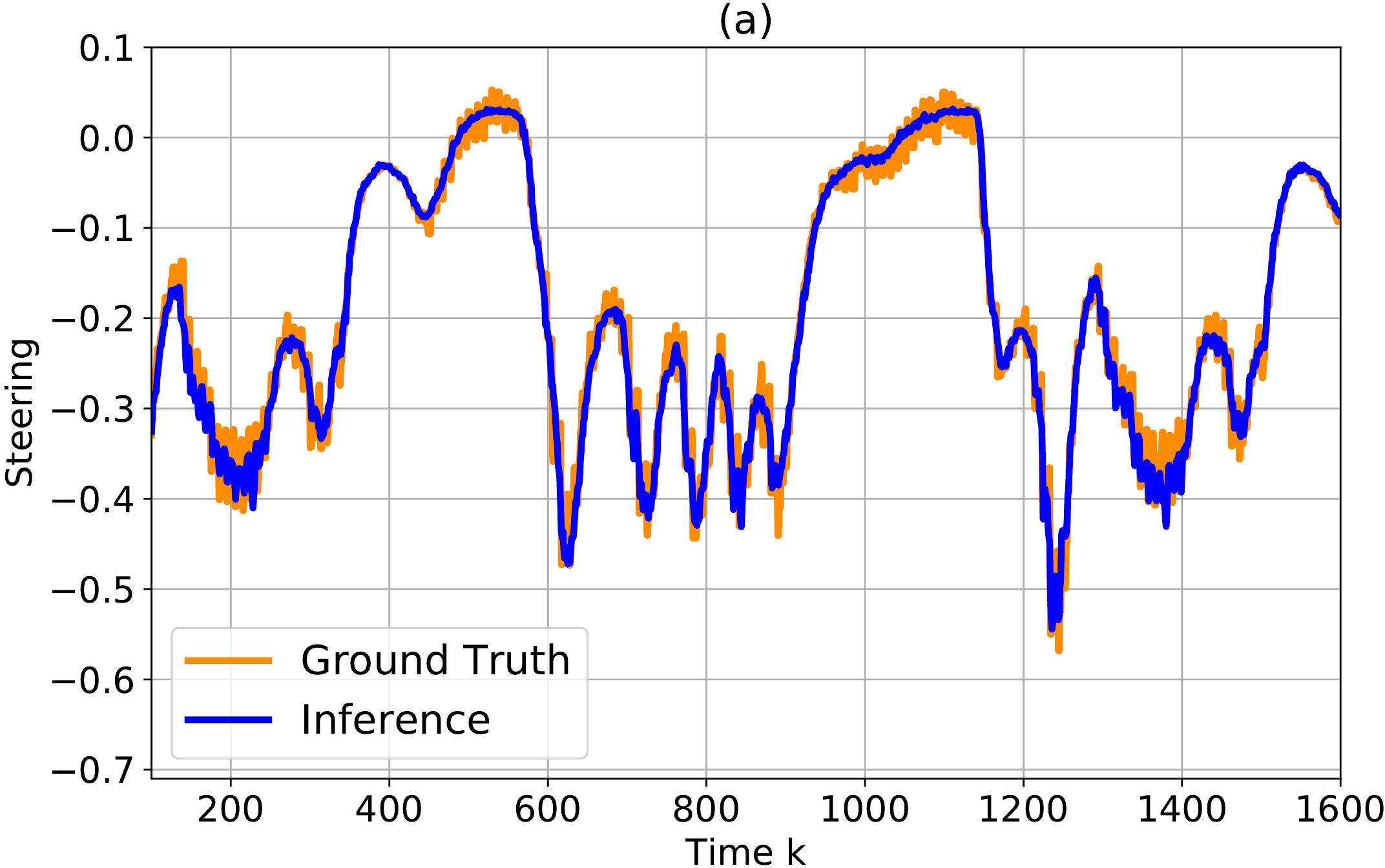}}
\subfigure{\includegraphics[scale=0.350]{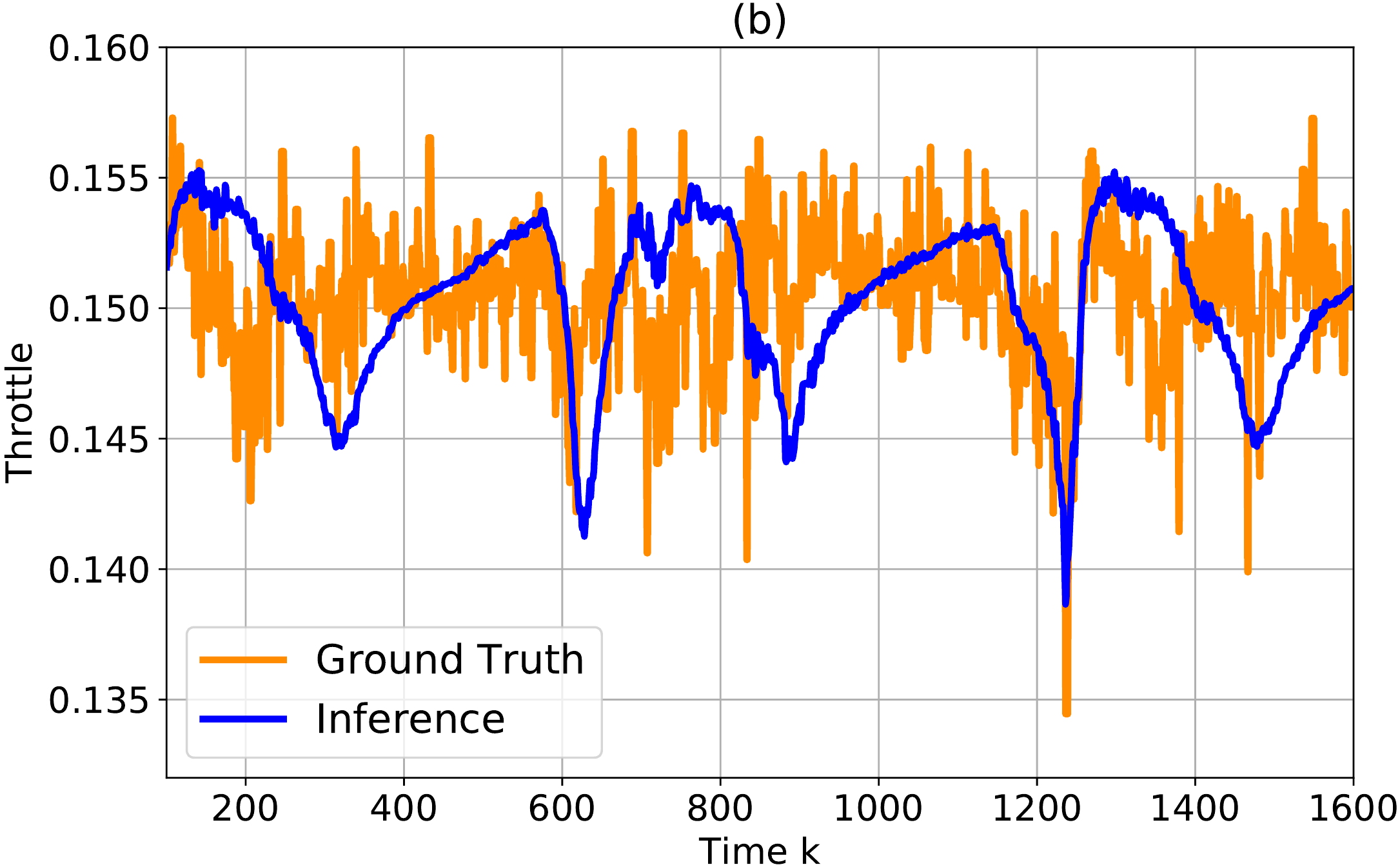}}
\vspace{-0.10cm}
\subfigure{\includegraphics[scale=0.35]{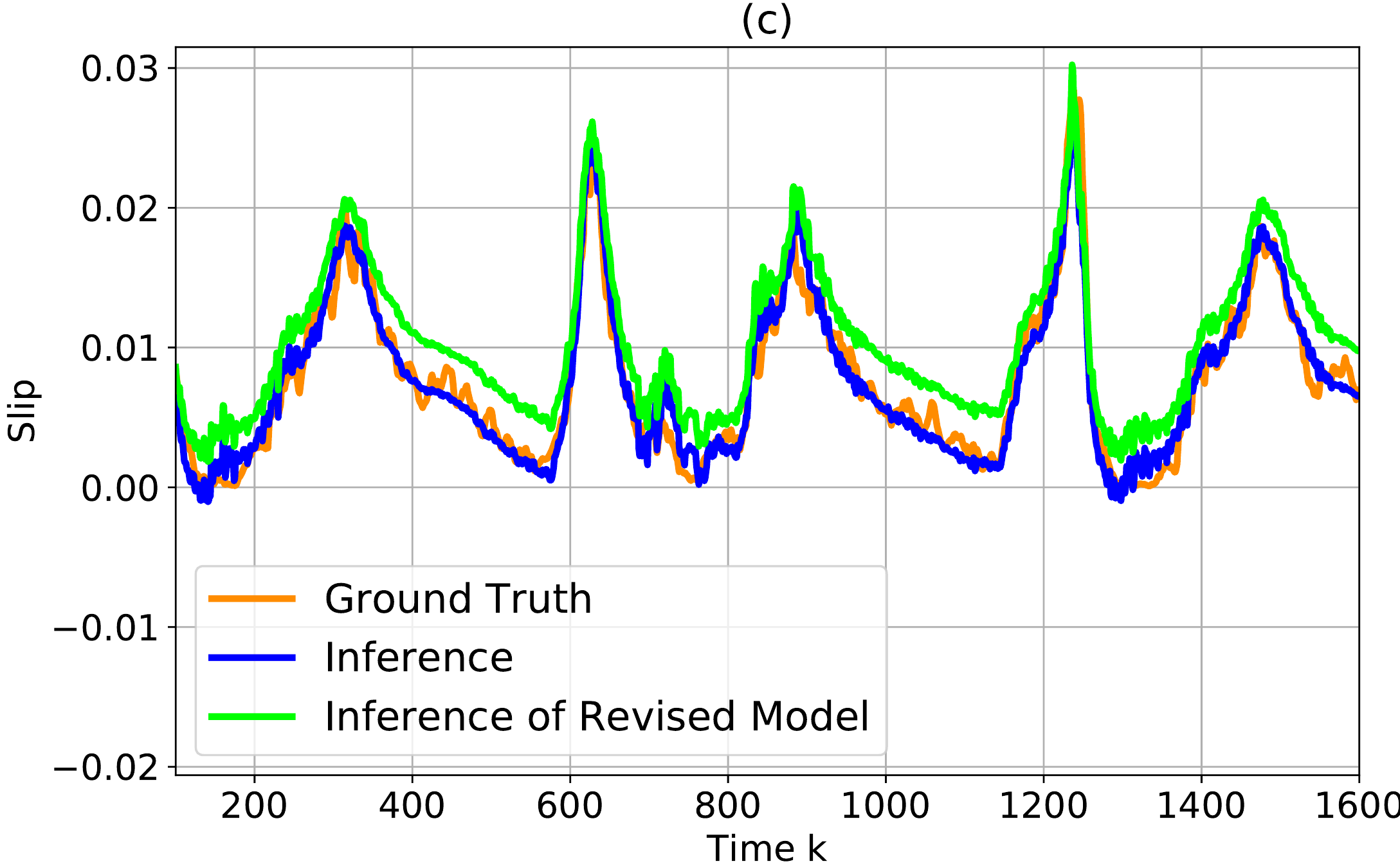}}
\subfigure{\includegraphics[scale=0.35]{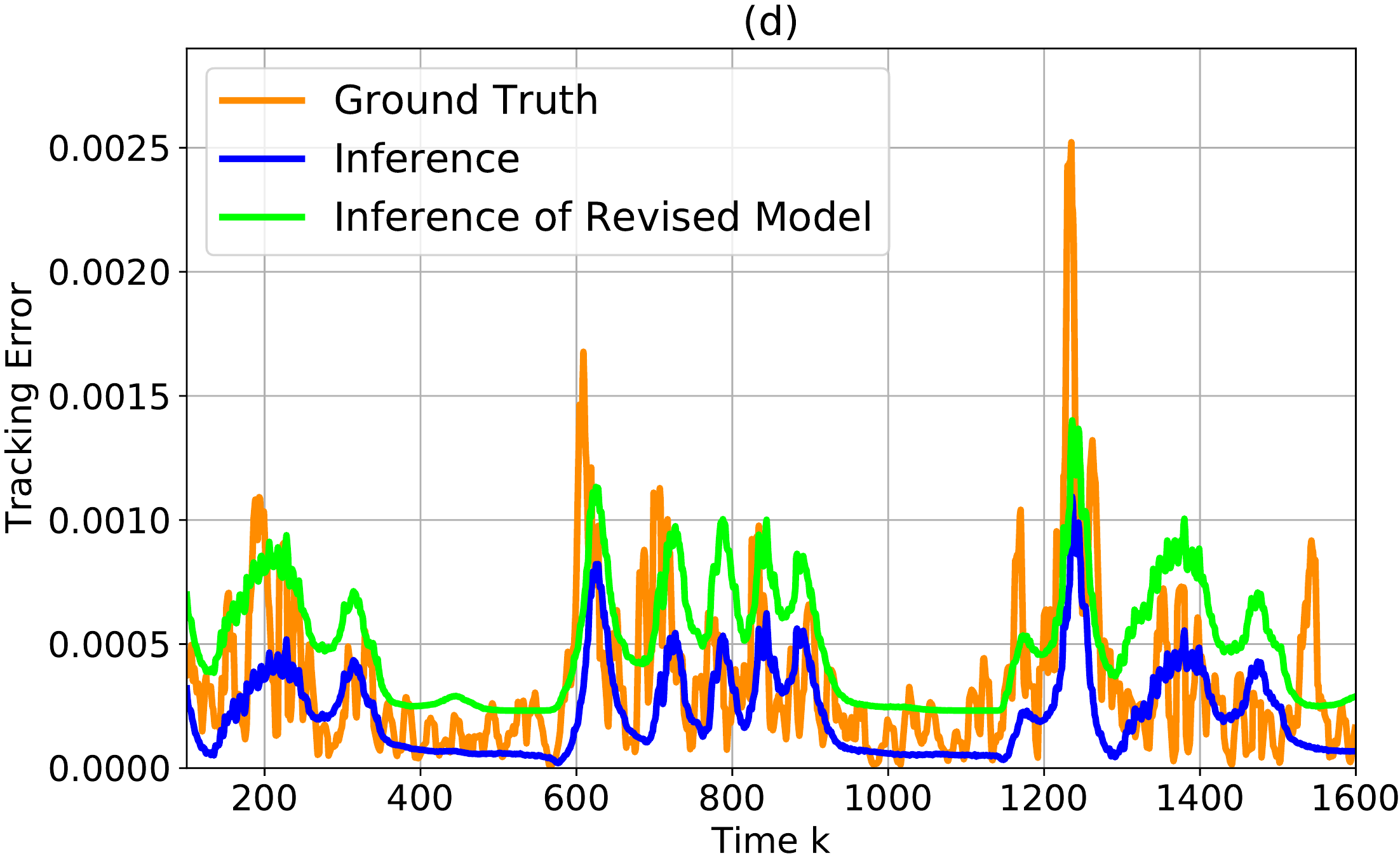}}
\vspace{-0.10cm}
\subfigure{\includegraphics[scale=0.245]{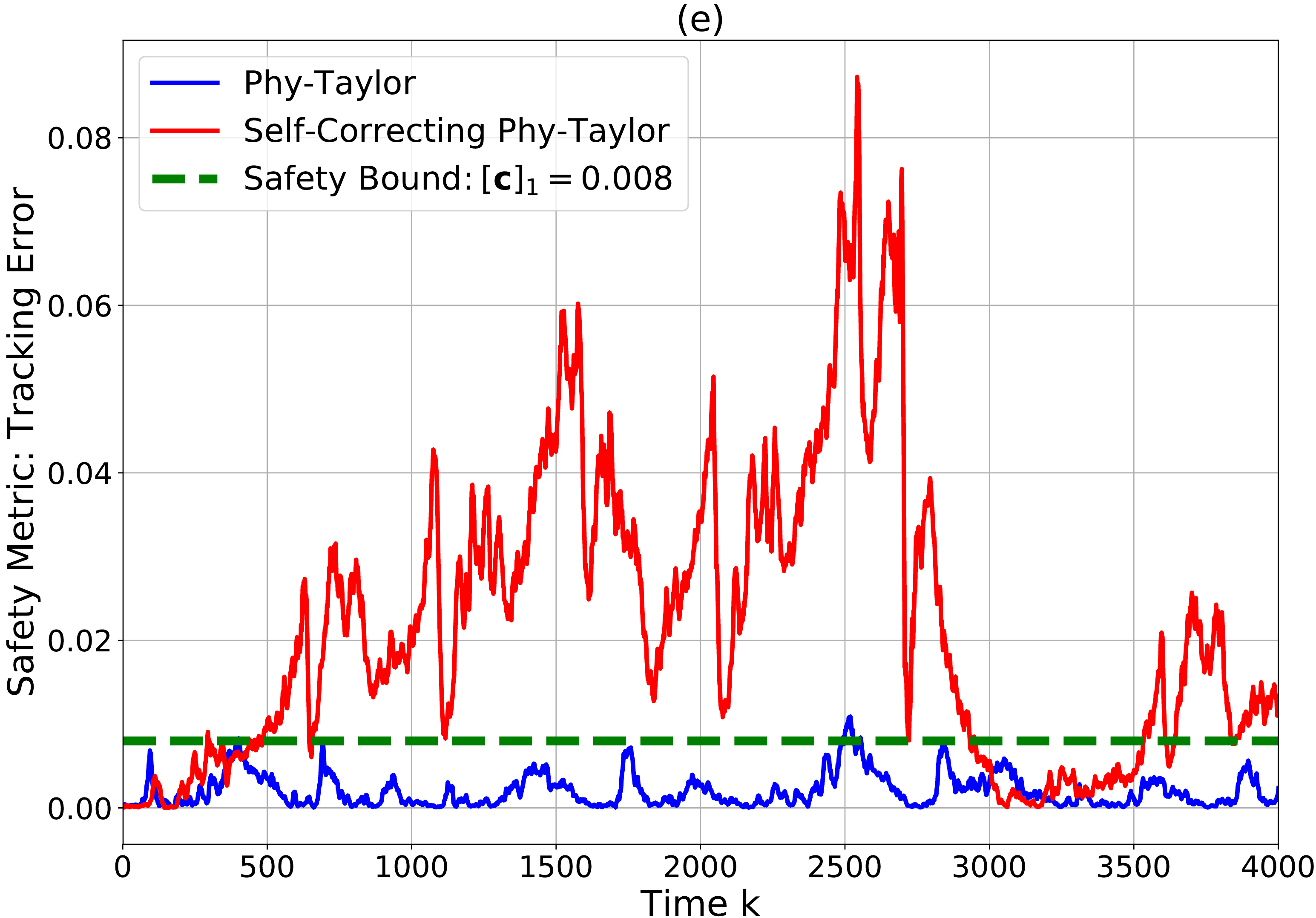}}
\subfigure{\includegraphics[scale=0.245]{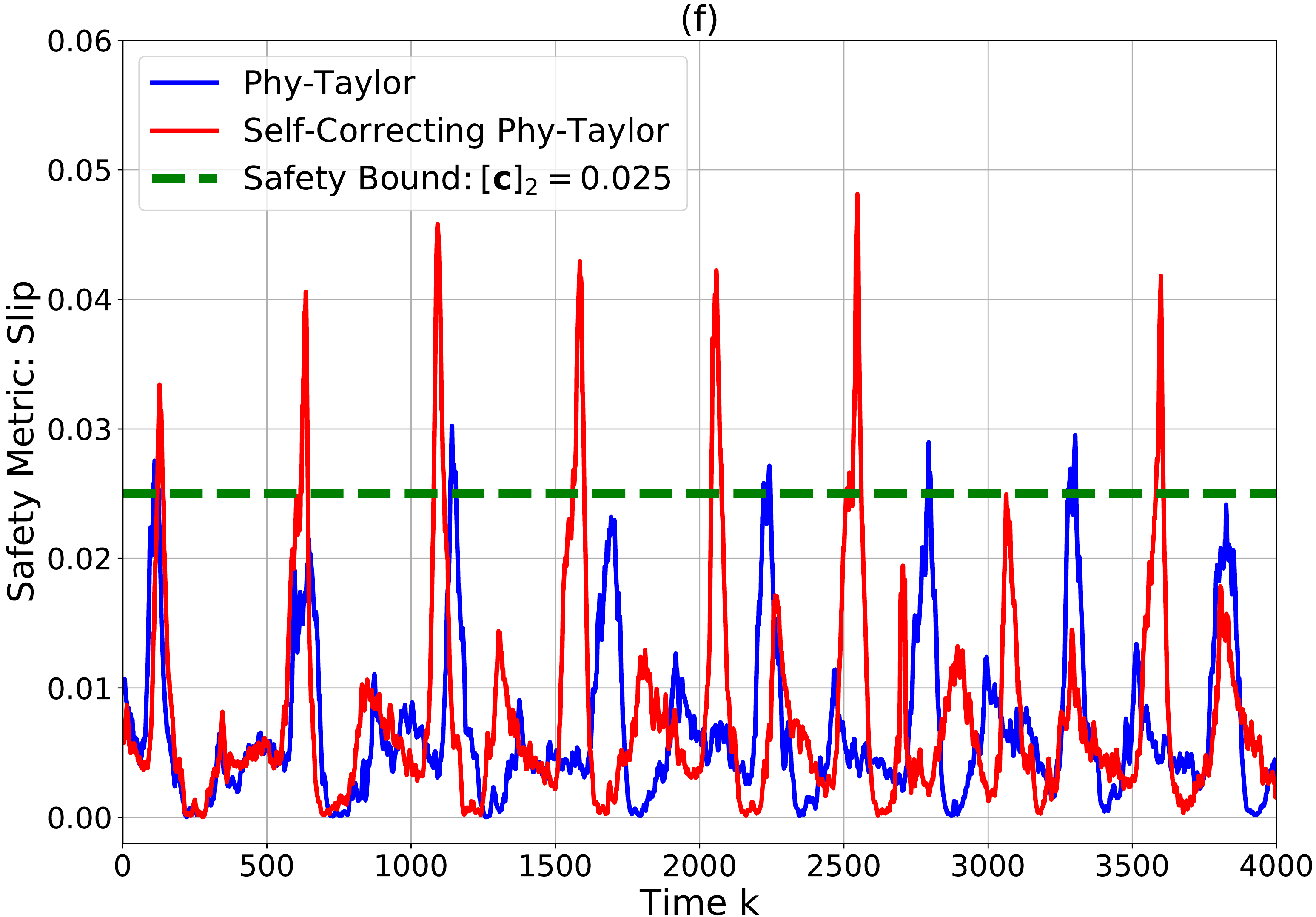}}
\vspace{-0.10cm}
\subfigure{\includegraphics[scale=0.245]{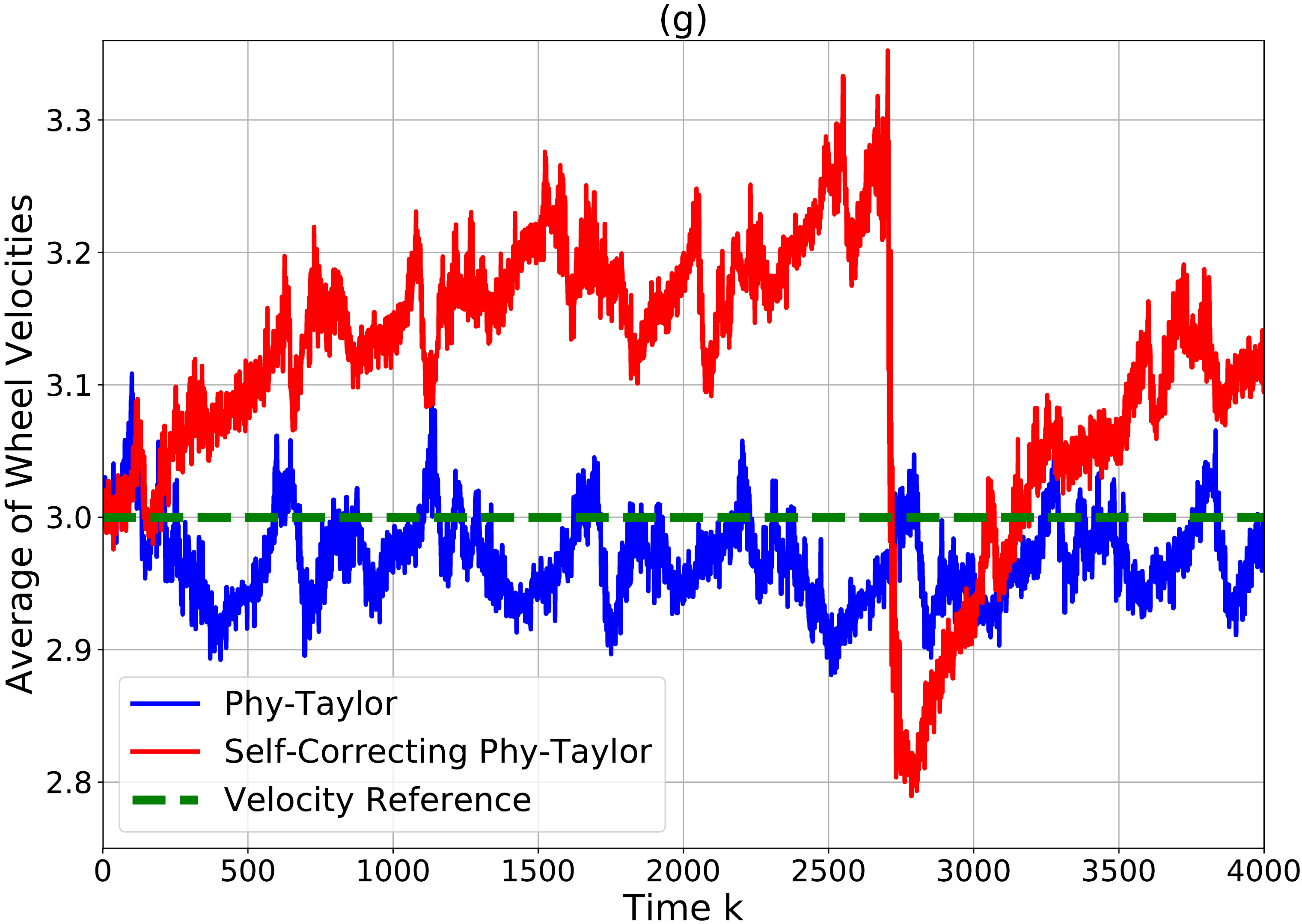}}
\subfigure{\includegraphics[scale=0.245]{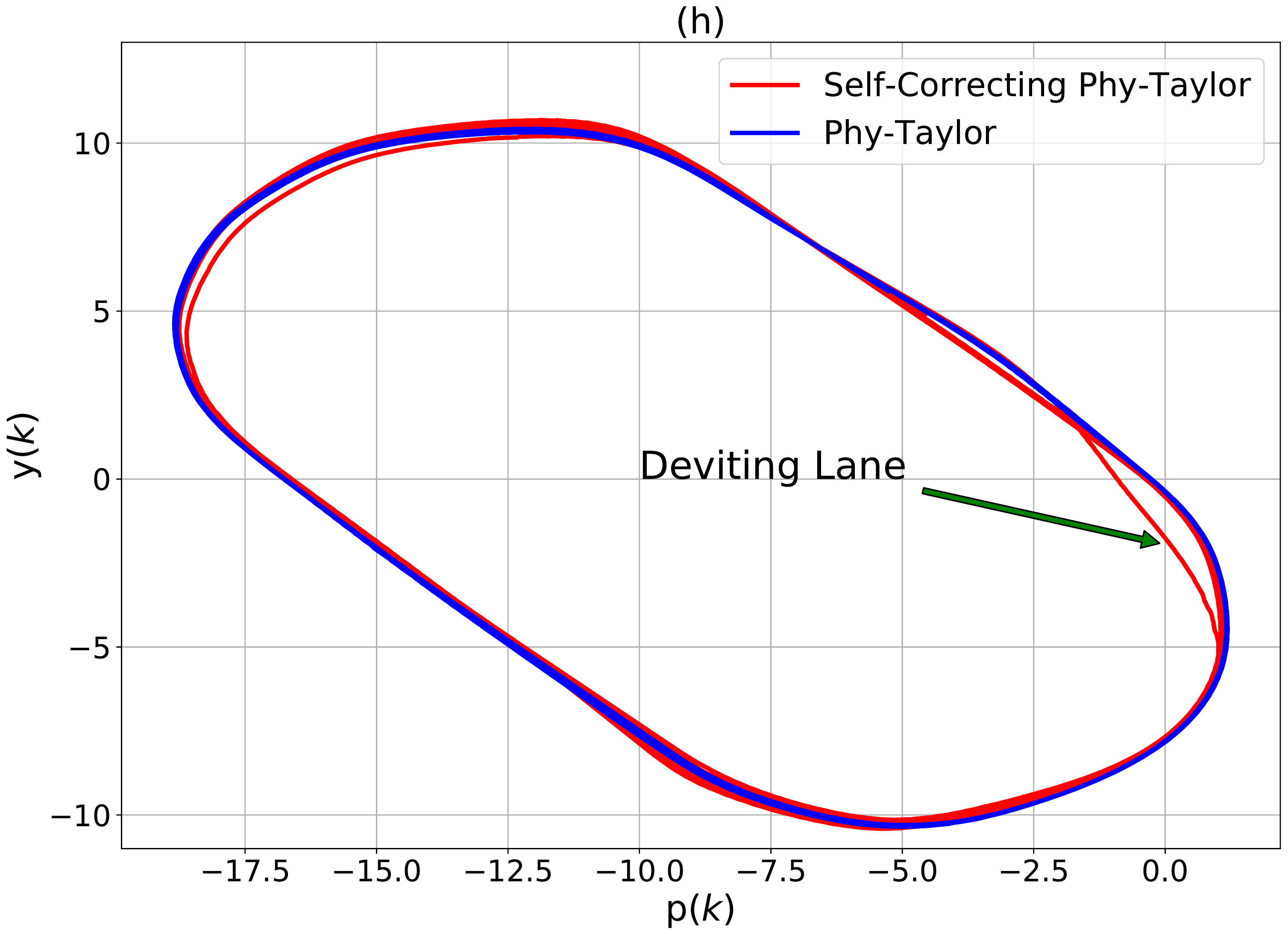}}
\caption{(a)--(b): Ground truth and inference of control commands. (c)--(d): Ground truth, original inference and inference (according to the revised safety relationships) of safety metrics (tracking error ${( {{v_{\mathrm{p}}}(t) - \mathrm{v}})}^2$ and wheel slip $(v_{\mathrm{p}}(t) - r \cdot w(k))^2$).  {(e)}: Safety metric: tracking error. {(f)}: Safety metric: wheel slip. (g): Average of wheel velocities. (h): Vehicle's driving curve (phase plot).}
\label{inff}
\end{figure*}

With the revised safety relationships \eqref{rchoo} at hand, we are ready to develop the self-correcting procedure. Considering the two matrices ${\bf{P}}_1$ and ${\bf{P}}_2$ defined in Equation \eqref{rchoo} are symmetric, we have  
\begin{align}
{\bf{P}}_1 &=  \underbrace{\left[ {\begin{array}{*{20}{c}}
  { - 0.934}&{ - 0.3572} \\
  { - 0.3572}&{0.934}
\end{array}} \right]}_{\triangleq {\bf{Q}}_1} \left[ {\begin{array}{*{20}{c}}
  {\underbrace{0.0008}_{\triangleq {\lambda}_{1}}}&0 \\
  0&{\underbrace{0.0038}_{\triangleq {\lambda}_{2}}}
\end{array}} \right] \underbrace{\left[ {\begin{array}{*{20}{c}}
  { - 0.934}&{ - 0.3572} \\
  { - 0.3572}&{0.934}
\end{array}} \right]}_{= {\bf{Q}}^{\top}_1}, \label{rch1}\\
{\bf{P}}_2 &= {\bf{Q}}_1 \cdot {\bf{Q}}_1 \cdot {\bf{P}}_2  \cdot {\bf{Q}}_1 \cdot {\bf{Q}}_1, \label{rch2}
\end{align}
based on which, we further define:
\begin{align}
\widehat{\mathbf{u}}(k) \triangleq {\bf{Q}}_1 \left[ \begin{array}{l}
\theta \left( k \right)\\
\gamma \left( k \right)
\end{array} \right], ~~~~~~~~~~\bf{S} \triangleq {\bf{Q}}_1 \cdot {\bf{P}}_2  \cdot {\bf{Q}}_1 = \left[ {\begin{array}{*{20}{c}}
{{s_{11}}}&{{s_{12}}}\\
{{s_{12}}}&{{s_{22}}}
\end{array}} \right]. \label{rch3}
\end{align}
We let $[\widehat{\mathbf{c}}]_1$ and $[\widehat{\mathbf{c}}]_2$ denote the two assigned safety metrics. According to the derivations appearing in Supplementary Information \ref{SI10}, the control commands included in the safety formulas $[{\bf{s}}({\bf{u}}(k))]_1 = [\widehat{\mathbf{c}}]_1$ and $\left[ {\bf{s}}({\bf{u}}(k)) \right]_2 = [\widehat{\mathbf{c}}]_2$ are obtained as 
\begin{align}
\left[ \begin{array}{l}
 \pm \widehat{\theta}(k)\\
 \pm \widehat{\gamma}(k)
\end{array} \right] \triangleq  {\mathbf{Q}_1}\left[ \begin{array}{l}
 \pm \sqrt {\frac{{{[\widehat{\mathbf{c}}]_1} - {[\mathbf{b}]_1}}}{{{\lambda _1}}} - \frac{{{\lambda _2}}}{{{\lambda _1}}}\frac{{\sqrt {\varpi _2^2 - 4{\varpi _1}{\varpi _3}}  - {\varpi _2}}}{{2{\varpi _1}}}} \\
 \pm \sqrt {\frac{{\sqrt {\varpi _2^2 - 4{\varpi _1}{\varpi _3}}  - {\varpi _2}}}{{2{\varpi _1}}}}
\end{array} \right],  \label{rch43}
\end{align}
where 
\begin{align}
{\varpi _1} &\triangleq {\left( {\frac{{{\lambda _2}}}{{{\lambda _1}}}} \right)^2}s_{11}^2 + s_{22}^2 + \frac{{\left( {4s_{12}^2 - 2{s_{11}}{s_{22}}} \right){\lambda _2}}}{{{\lambda _1}}}, \label{pko1}\\
{\varpi _2} &\triangleq \frac{{2( {{[\widehat{\mathbf{c}}]_1} - {[\mathbf{b}]_1}}){s_{11}}{s_{22}} - 4( {{[\widehat{\mathbf{c}}]_1} - {[\mathbf{b}]_1}})s_{12}^2 + 2{\lambda _2}( {{[\mathbf{b}]_2} - {[\widehat{\mathbf{c}}]_2}}){s_{11}}}}{{{\lambda _1}}} \nonumber\\
&\hspace{7.5cm} - \frac{{2( {{[\widehat{\mathbf{c}}]_1} - {[\mathbf{b}]_1}}){\lambda _2}s_{11}^2}}{{\lambda _1^2}} - 2( {{[\mathbf{b}]_2} - {[\widehat{\mathbf{c}}]_2}}){s_{22}},\label{pko2}\\
{\varpi _3} &\triangleq {\left( {{[\mathbf{b}]_2} - {[\widehat{\mathbf{c}}]_2}} \right)^2} + {\left( {\frac{{{[\widehat{\mathbf{c}}]_1} - {[\mathbf{b}]_1}}}{{{\lambda _1}}}} \right)^2}s_{11}^2 - \frac{{2\left( {{[\widehat{\mathbf{c}}]_1} - {[\mathbf{b}]_1}} \right)\left( {{[\mathbf{b}]_2} - {[\widehat{\mathbf{c}}]_2}} \right){s_{11}}}}{{{\lambda _1}}}.\label{pko3}
\end{align}

The solution \eqref{rch43} has paved the way to delivering the self-correcting procedure: Algorithm \ref{ALG4}. The algorithm can be summarized as if the real-time safety metric $[{\bf{s}}({\bf{u}}(k))]_1$
or $[{\bf{s}}({\bf{u}}(k))]_2$ is larger than the corresponding safety bound $[\mathbf{c}]_1$ or $[\mathbf{c}]_2$, the real-time safety metric will be updated with the corresponding safety bound (indicated by Line \ref{ALG4-4} of Algorithm \ref{ALG4}). The corrected control commands are then computed according to \eqref{rch43} (see Lines \ref{ALG4-8}-\ref{ALG4-10}). The solutions however are not unique. To address the problem, the Line \ref{ALG4-12} of Algorithm \ref{ALG4} picks up the control commands that are most close to current ones.

Under the control of Phy-Taylor, with and without the self-correcting procedure, the system performances are presented in Figure \ref{inff} (e)--(h). We can observe from the figure that the self-correcting procedure can significantly enhance the safe assurance of velocity regulation. The demonstration video of implementing the self-correcting Phy-Taylor in AutoRally is available at \url{https://ymao578.github.io/pubs/taylorcontrol2.mp4}.

\subsection{Coupled Pendulums} \label{expe}
\begin{wrapfigure}{r}{0.30\textwidth}
\vspace{-0.9cm}
  \begin{center}
    \includegraphics[width=0.30\textwidth]{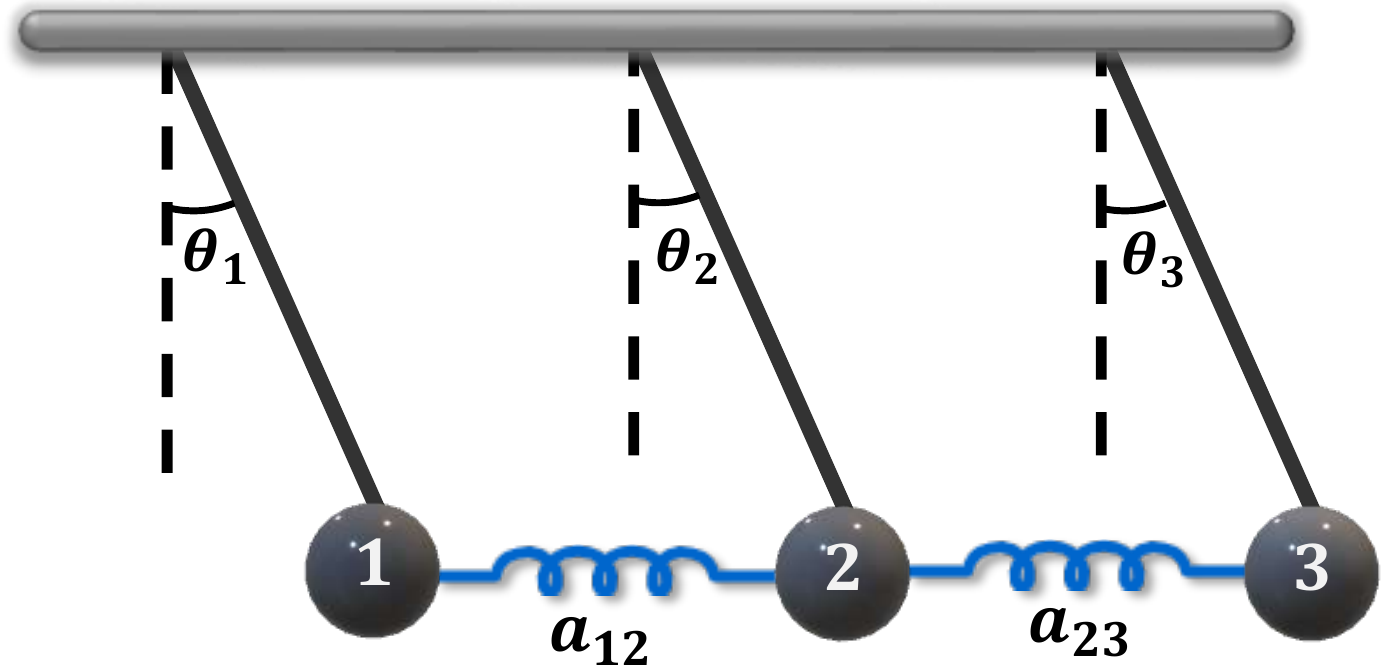}
  \end{center}
  \vspace{-0.3cm}
  \caption{Mechanical analog of three coupled pendulums.}
  \vspace{-0.3cm}
  \label{df}
\end{wrapfigure}
The second experiment demonstrates the proposed approach to learn the dynamics of three coupled pendulums, whose mechanical analog is shown in Figure \ref{df}. Each pendulum is characterized by its phase angle $\theta_i$ and velocity $\dot \theta _i$. The physical topology is described by $a_{12} = a_{21} = a_{23} = a_{32} = 1$ and $a_{13} = a_{31} = 0$. We let ${\upsilon _i} \triangleq {\dot \theta _i}$. Referring to Figure \ref{df}, the available physical knowledge for NN editing are summarized in Table \ref{tab11}.  

\begin{table} \scriptsize{
\caption{Models with Different Degrees of Embedded Physical Knowledge}
\centering
\begin{tabular}{l cccc c c}
\toprule
 & \multicolumn{4}{c}{Available Physical Knowledge}  \\
\cmidrule(lr){2-5}
Model ID     & \makecell{Physics Law: \\ $\upsilon = \dot \theta$}    & \makecell{Sampling Period: \\$T$}  & \makecell{Coupling Topology: \\$1 \leftrightsquigarrow 2 \leftrightsquigarrow 3$} & \makecell{Force Dependency}  & Training Loss  & \makecell{Out-of-Distribution:\\Prediction Error $e$}\\
\midrule
Phy-Taylor-1         &  $\surd$   &    $\surd$       &    $\surd$      &     $\surd$    &     $1.75146\!\cdot\!10^{-6}$   &     $0.06486147$\\
Phy-Taylor-2       &  $\surd$   &    $\times$     &     $\surd$     &     $\times$   &         $2.42426\!\cdot\!10^{-6}$   &      $1.46647886$\\
FDNN         &  $\times$  &    $\times$      &    $\times$    &     $\times$   &     $6.63564\cdot10^{-7}$   &      $3.65772883$\\
\bottomrule
\end{tabular}
\label{tab11}}
\end{table}

\begin{figure*}[!t]
\centering
\subfigure{\includegraphics[scale=0.48]{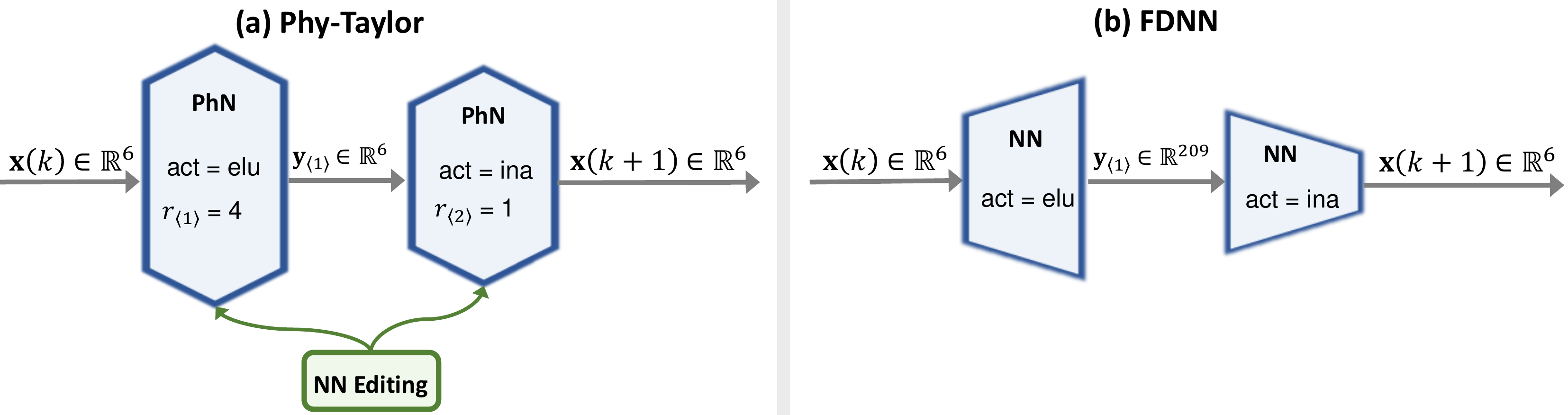}}
\subfigure{\includegraphics[scale=0.24]{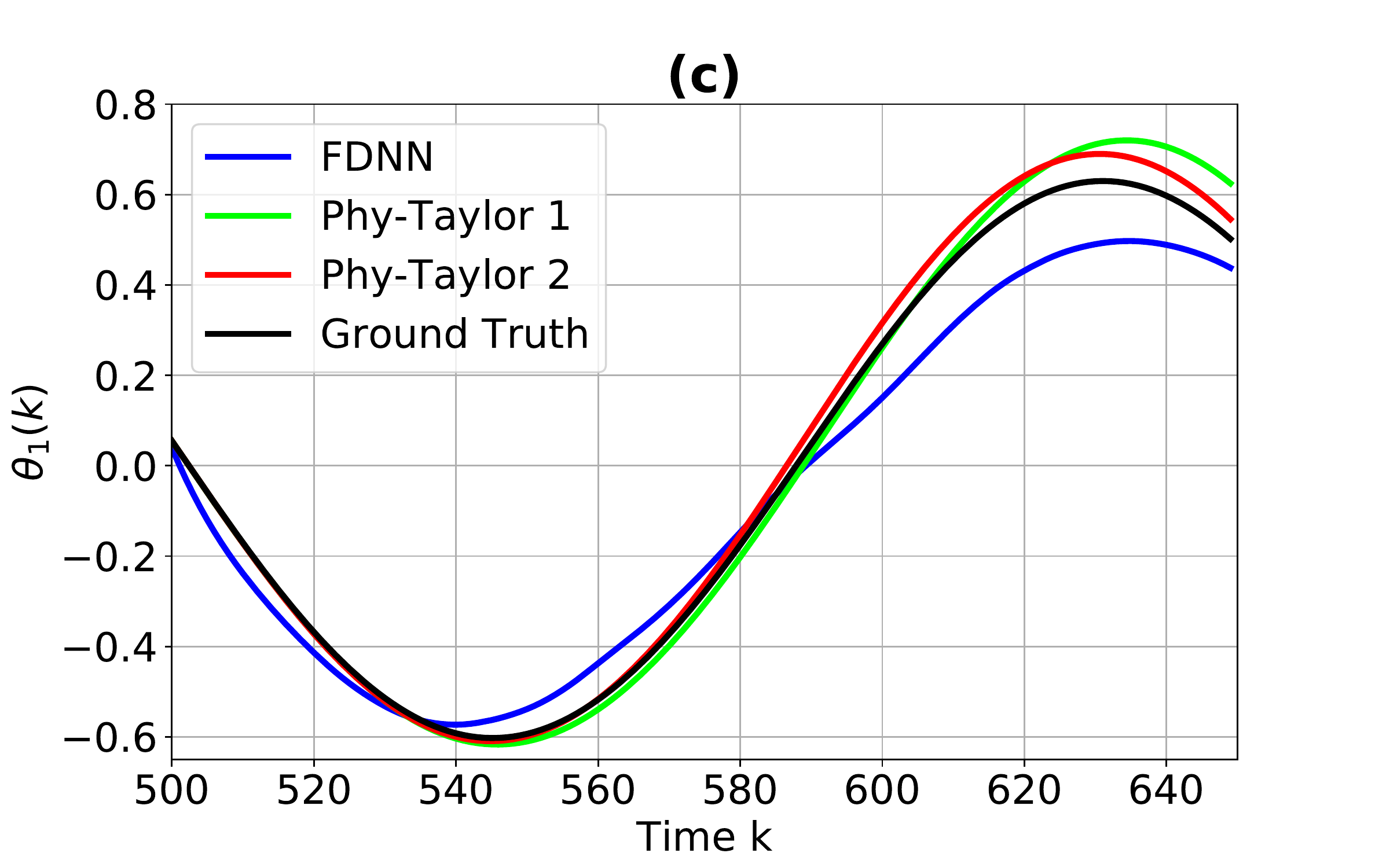}}
\subfigure{\includegraphics[scale=0.24]{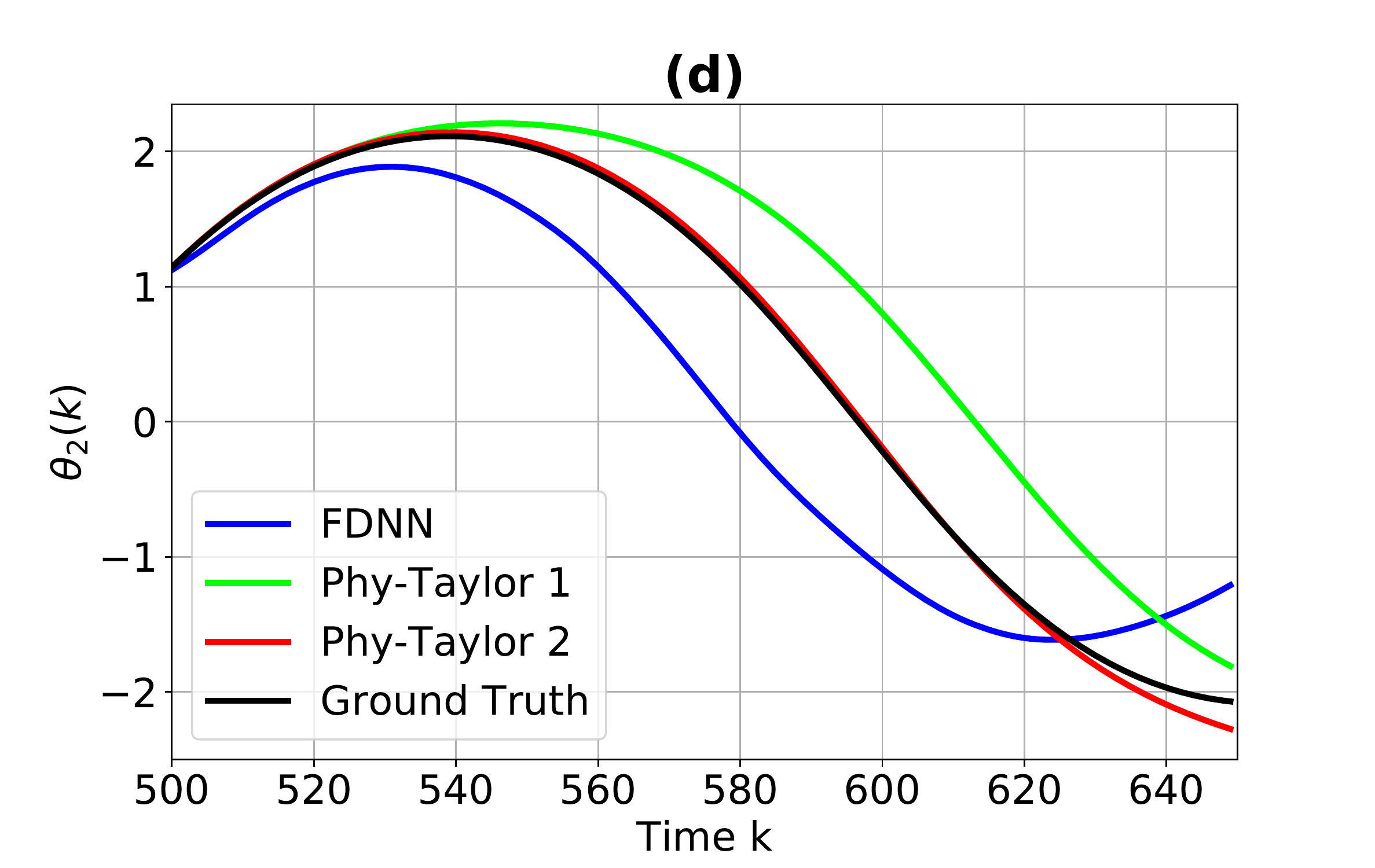}}
\subfigure{\includegraphics[scale=0.24]{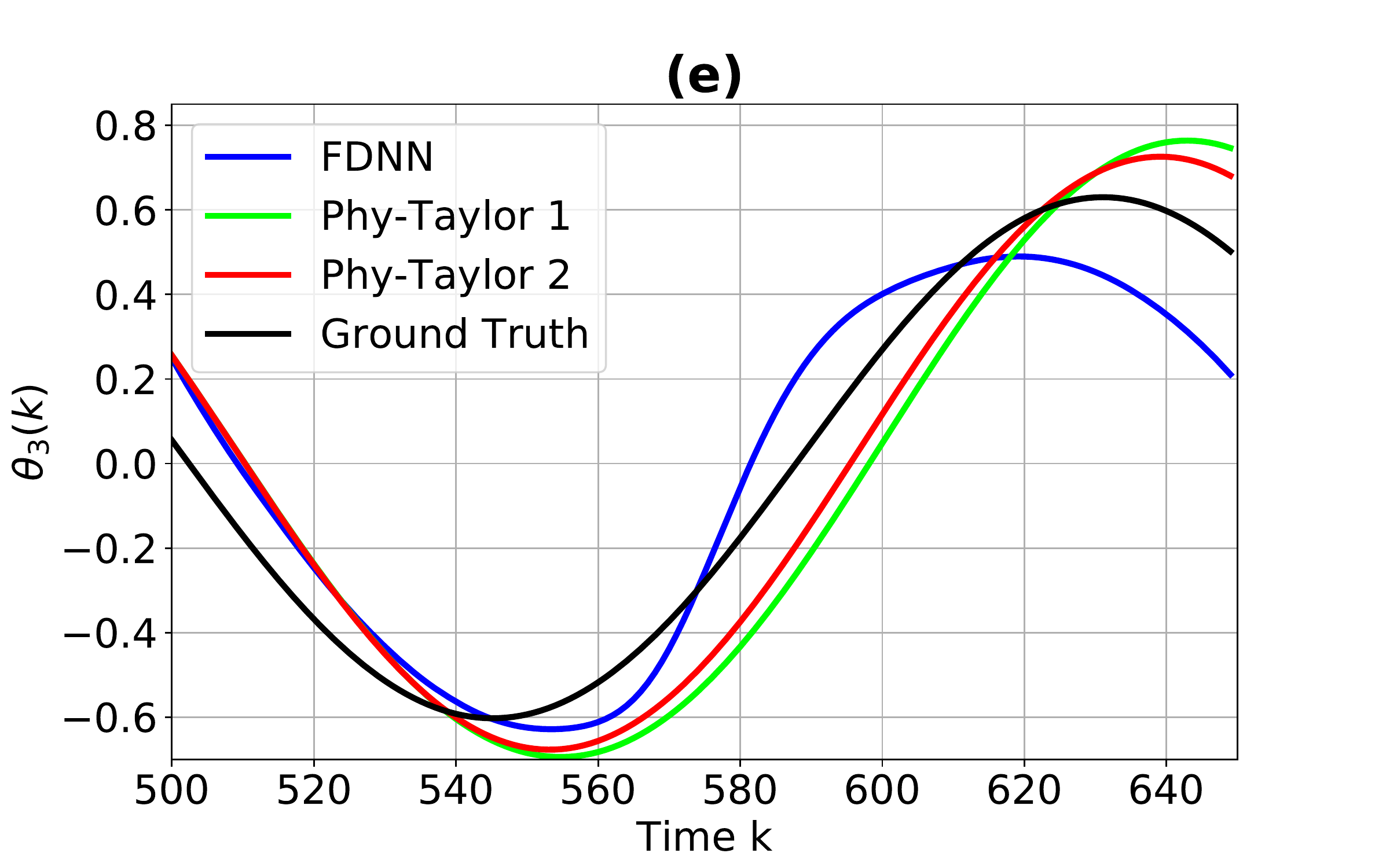}}
\subfigure{\includegraphics[scale=0.24]{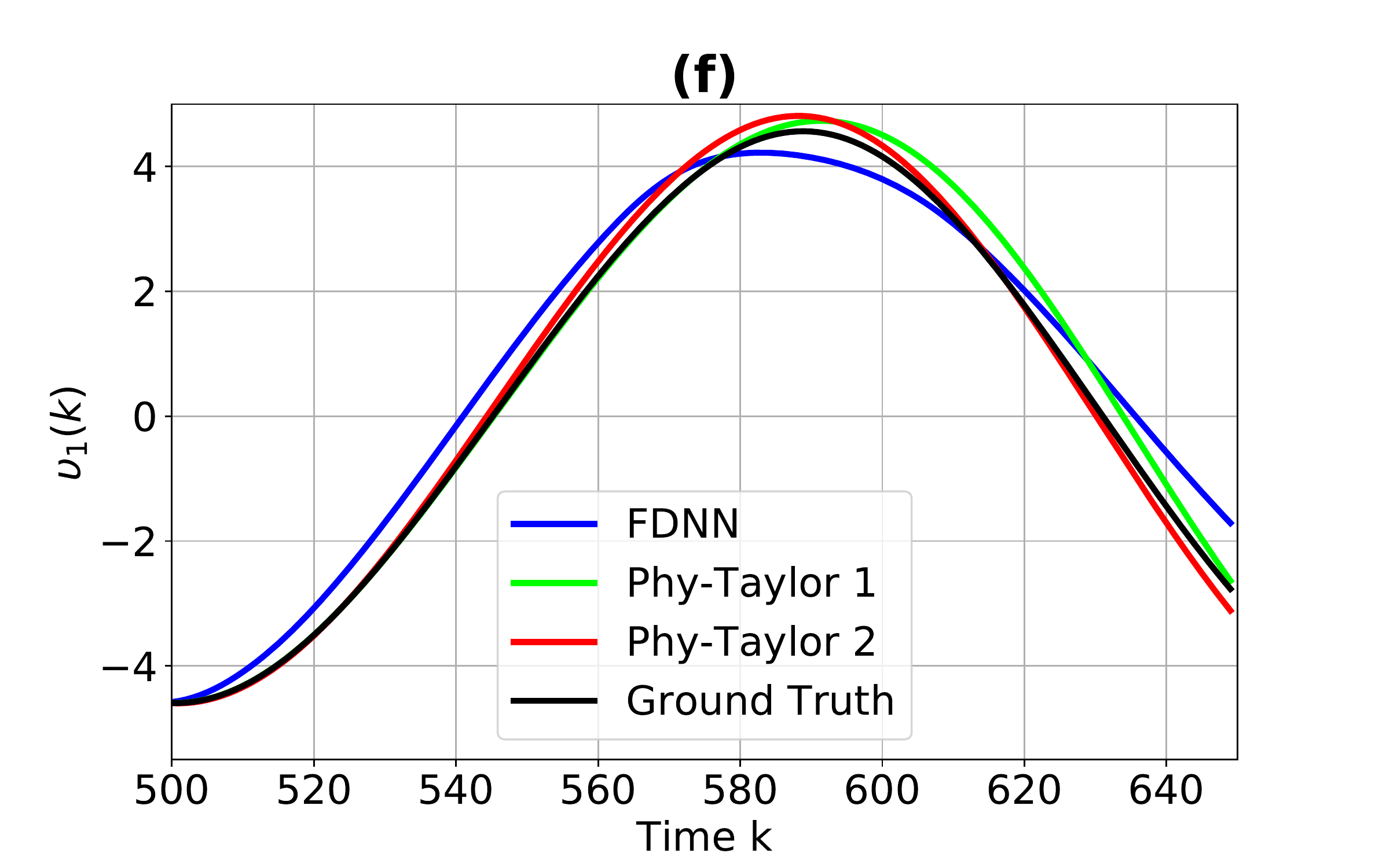}}
\subfigure{\includegraphics[scale=0.24]{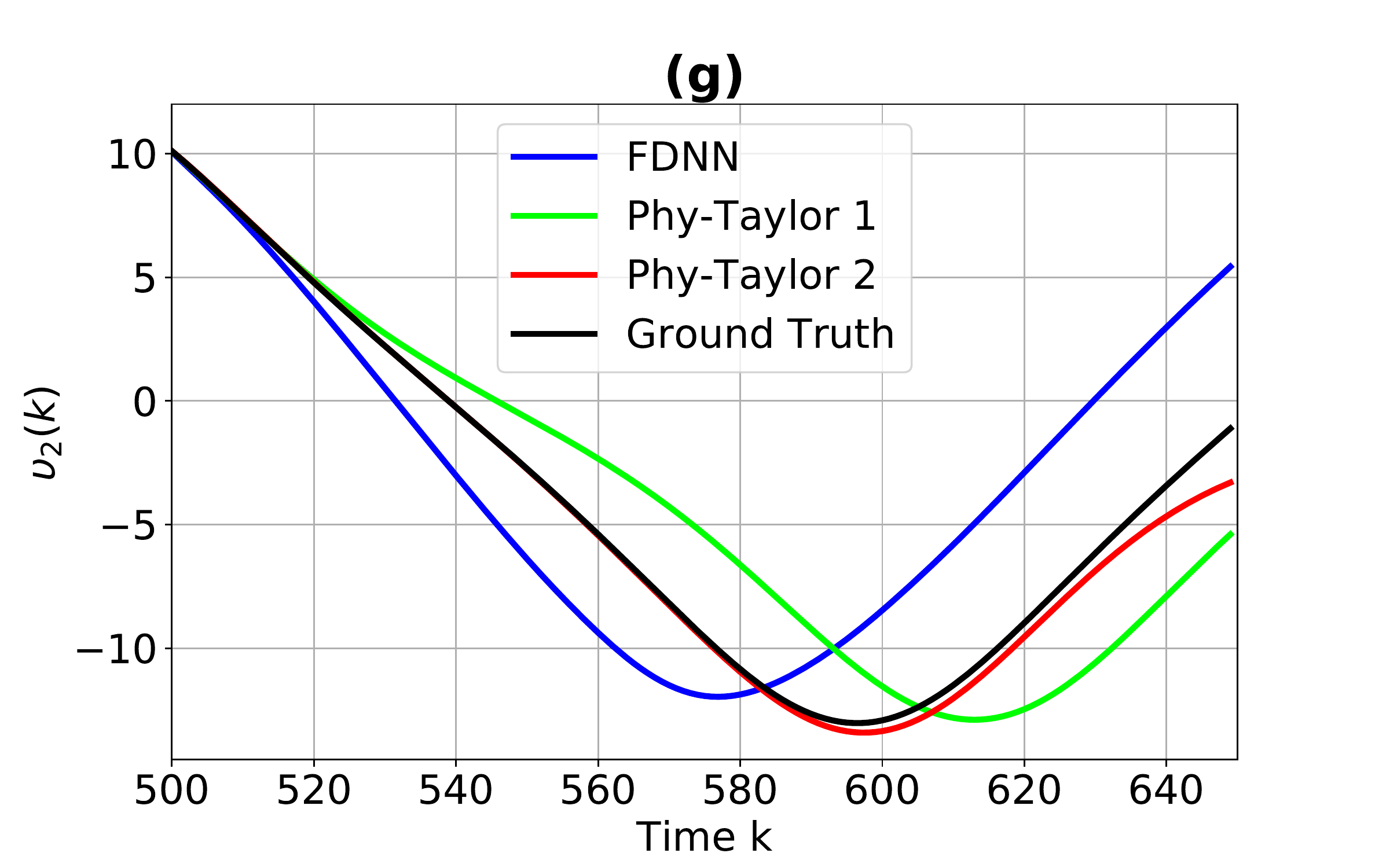}}
\subfigure{\includegraphics[scale=0.24]{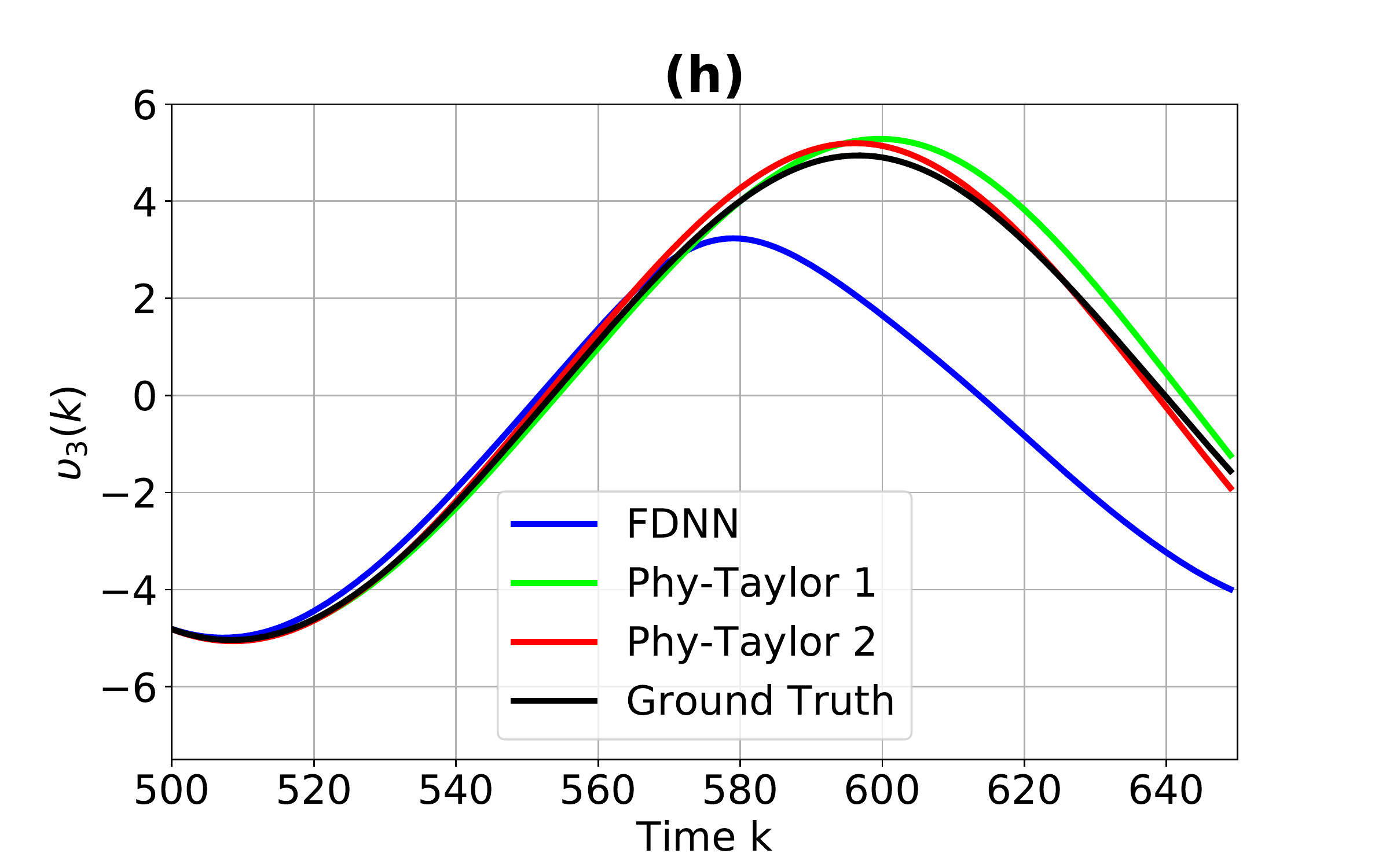}}
\caption{(i): Phy-Taylor architecture.  (ii): FDNN architecture. (a)--(f): Ground truth and predicted trajectories.}
\label{gfd}
\end{figure*}

We consider the dynamics learning via Phy-Taylor and fully-connected DNN (FDNN), whose architectures are shown in Figure \ref{gfd} (a) and (b). The inputs of both networks are identical as $\mathbf{x}(k)$ $=$ [${\theta _1}( {k})$; ${\theta _2}( {k})$; ${\theta _3}( {k})$; ${\upsilon _1}( {k})$; ${\upsilon _2}( {k})$; ${\upsilon _3}( {k})$]. For a fair comparison, we let each layer of FDNN has the same number of nodes as the Phy-Taylor (after augmentation), which suggests the network configuration in Figure \ref{gfd} (b).

We aim to demonstrate the robustness of Phy-Taylor owning to NN editing. To achieve this, we consider the testing data that is out-of-distribution. Specifically, the initial conditions of phase angles for generating training data are ${\theta _1}(0), {\theta _2}(0), {\theta _3}(0) \in [-1,1]$, while the initial conditions for generating testing data are outside the range: ${\theta _1}(0), {\theta _2}(0), {\theta _3}(0) \in [-1.5,-1)$. The training loss (mean square error) of the considered models with different degrees of embedded physical knowledge are summarized in Table \ref{tab11}. To review the robustness of trained models, we consider the prediction of long-horizon trajectory, given the same initial input. The prediction error is measured by $e = \frac{1}{6}\sum\limits_{i = 1}^3 {\frac{1}{\tau }\left( {\sum\limits_{t = k + 1}^{k + \tau } {\left( {{{\widehat \theta }_i}(t) - {\theta _i}(t)} \right)^2} + \sum\limits_{t = k + 1}^{k + \tau } {\left( {{{\widehat \upsilon }_i}(t) - {\upsilon _i}(t)} \right)^2} } \right)}$, where ${{\widehat \theta }_i}(t)$ and ${{\widehat \upsilon }_i}(t)$ denote the predicted angle and angular speed of $i$-th pendulum at time $t$, respectively. The prediction errors are summarized in Table \ref{tab11}, which, in conjunction with ground-truth and predicted trajectories in Figure \ref{gfd} (c)--(h), demonstrate that (i) the NN editing can significantly enhance the model robustness, and (ii) the more embedded physical knowledge can lead to stronger robustness, viewed from the perspective of long-horizon (out-of-distribution) predictions of trajectories.

\subsection{US Illinois Climate} \label{clima}
In final example, we consider a climate system without any available physical knowledge for NN editing, which degrades Phy-Taylor to deep PhN. The dataset is the hourly climate normals in Illinois state, including five stations\footnote{\tiny{\hspace{-0.0cm} NOAA: \hspace{-0.0cm} \url{https://www.ncdc.noaa.gov/cdo-web/datasets/NORMAL_HLY/locations/FIPS:17/detail}}}. The period of record is 01/01/2010--12/31/2010. The locations of considered stations are shown in Figure \ref{location}, where $S_{1}$-$S_{4}$ denote stations GHCND:USW00094846, GHCND:USW00094822, GHCND:USW00014842 and GHCND:USW00093822, respectively. The input of  deep PhN is
\begin{align}
\mathbf{x}(k) = \left[ {{x_1}(k);~{x_2}(k);~{x_3}(k);~{x_4}(k);~{{\dot x}_1}(k);~{{\dot x}_2}(k);~} \right. {{\dot x}_3}(k); ~{{\dot x}_4}(k);~{{\ddot x}_1}(k);~{{\ddot x}_2}(k);~{{\ddot x}_3}(k);~\left. {{{\ddot x}_4}(k)} \right], \nonumber
\end{align}
\begin{wrapfigure}{r}{0.285\textwidth}
\vspace{-0.8cm}
  \begin{center}
    \includegraphics[width=0.285\textwidth]{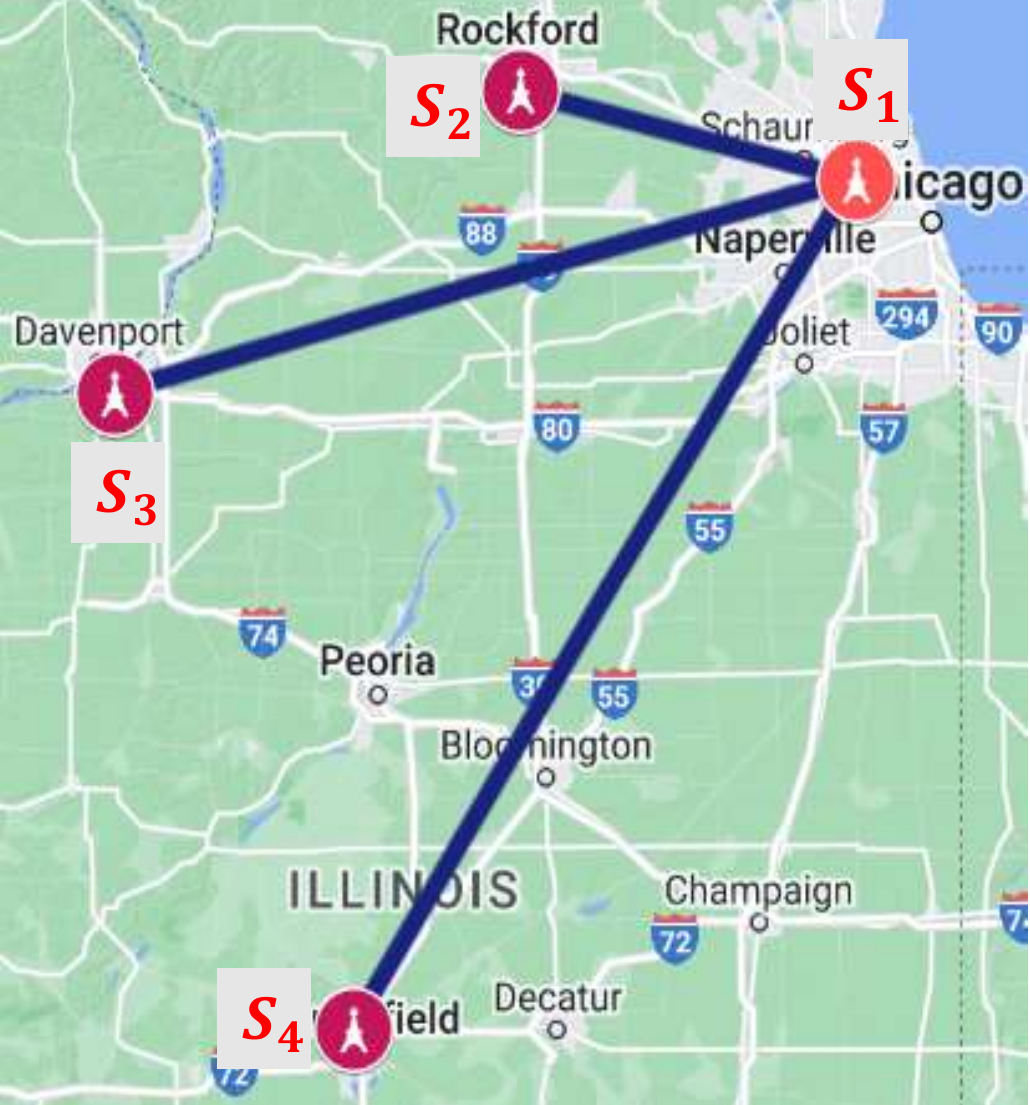}
  \end{center}
  \vspace{-0.75cm}
  \caption{Station locations.}
  \vspace{-0.0cm}
  \label{location}
\end{wrapfigure}
where ${{\dot x}_i}\left( k \right) = {{x}_i}\left( k \right) - {{x}_i}\left( k-1 \right)$ and ${{\ddot x}_i}\left( k \right) = {\dot{x}_i}\left( k \right) - {\dot{x}_i}\left( k-1 \right) =
({{x}_i}\left( k \right) - {{x}_i}\left( k-1 \right)) - ({{x}_i}\left( k-1 \right) - {{x}_i}\left( k-2 \right))$, $i = 1,2,3,4$, and the $x_{1}$, $x_{2}$, $x_{3}$ and $x_{4}$ denote the dew point mean, the heat index mean, the wind chill mean and the average wind speed, respectively. The training loss function is 
\begin{align}
\mathcal{L} = {\left\| {{[\widehat{\mathbf{y}}]_1} - {x_2}\left( {k + 1} \right)} \right\|} + {\left\| {{[\widehat{\mathbf{y}}]_2} - {x_4}\left( {k + 1} \right)} \right\|}, \nonumber
\end{align}
which indicates the network models are to predict the heat index mean and the average wind speed for next hour. 

\subsubsection{Phy-Augmentation} 
\begin{wrapfigure}{r}{0.60\textwidth}
\vspace{-0.9cm}
  \begin{center}
    \includegraphics[width=0.60\textwidth]{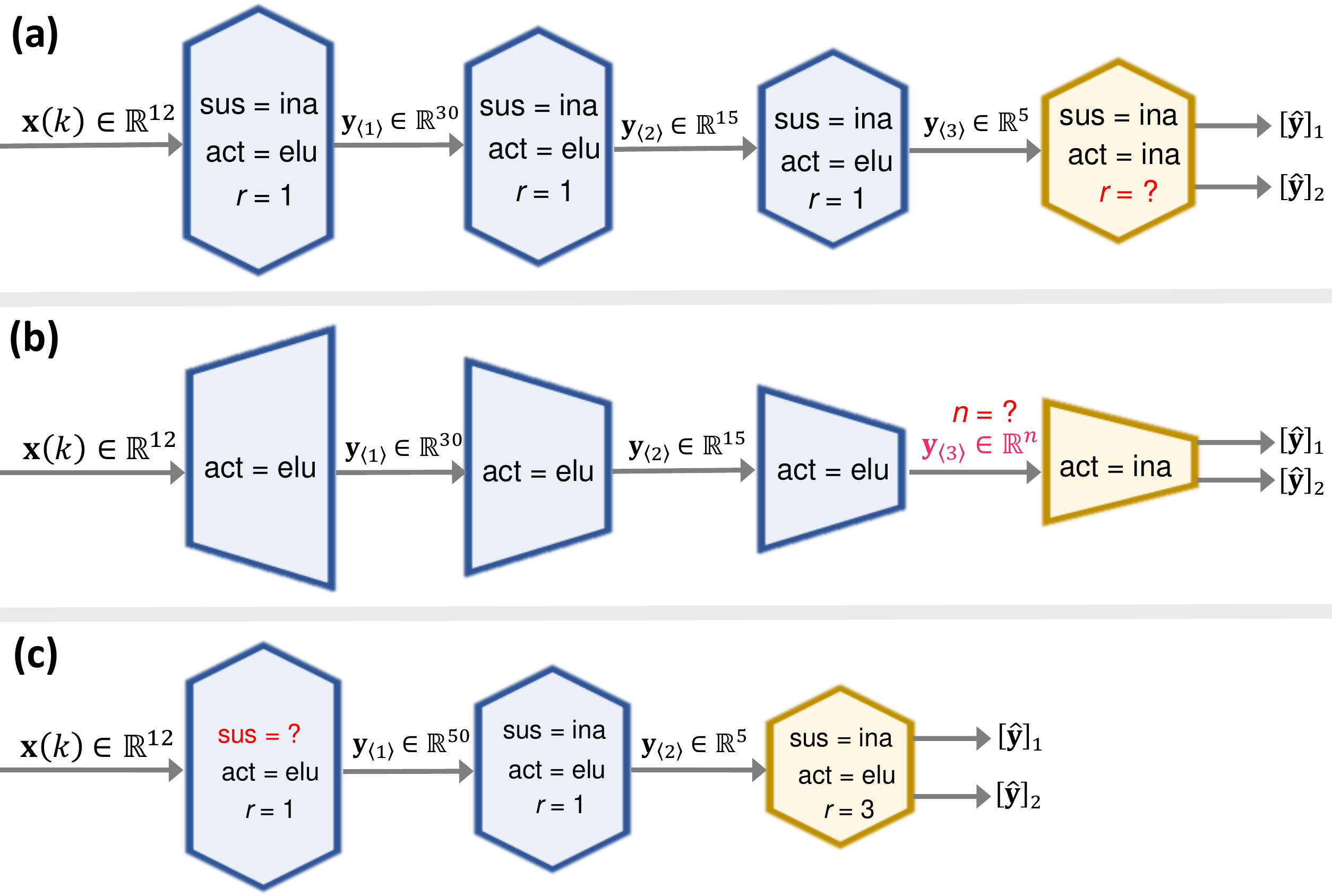}
  \end{center}
  \vspace{-0.45cm}
  \caption{Architectures of deep PhN and classical DNN.}
  \vspace{-0.5cm}
  \label{et}
\end{wrapfigure}
We use the data of station $S_1$ to study the influence of augmentation order $r$ from the perspective of training loss. The considered network model is shown in Figure \ref{et} (a). 
The trajectories of training loss under two different augmentation orders are presented in Figure \ref{ghj} (a). It is straightforward to observe from the figure that the input augmentation significantly speeds up the training process and reduces the training loss. The intuitive explanation is the input augmentation enlarges node number. 

\noindent To validate the statement, we perform the comparisons with the classical DNN, whose structure is given in Figure \ref{et} (b). To guarantee the comparisons are carried in the fair settings, we let the input dimension (i.e., $n$) of the final layer of DNN in Figure \ref{et} (b) equates to the output dimension of input augmentation of PhN in Figure \ref{et} (a). According to \eqref{cpeq1}, we let $n = {\rm{len}}(\mathfrak{m}(\mathbf{x},r = 5)) = 253$. The comparison results are presented in Figure \ref{ghj} (b), which indicates that given the same number of nodes, the deep PhN still has much faster convergence speed and smaller training loss of mean. The phenomenon can be explained by the fact that Phy-Augmentation well captures physical features.
\begin{figure*}
\centering
\subfigure{\includegraphics[scale=0.285]{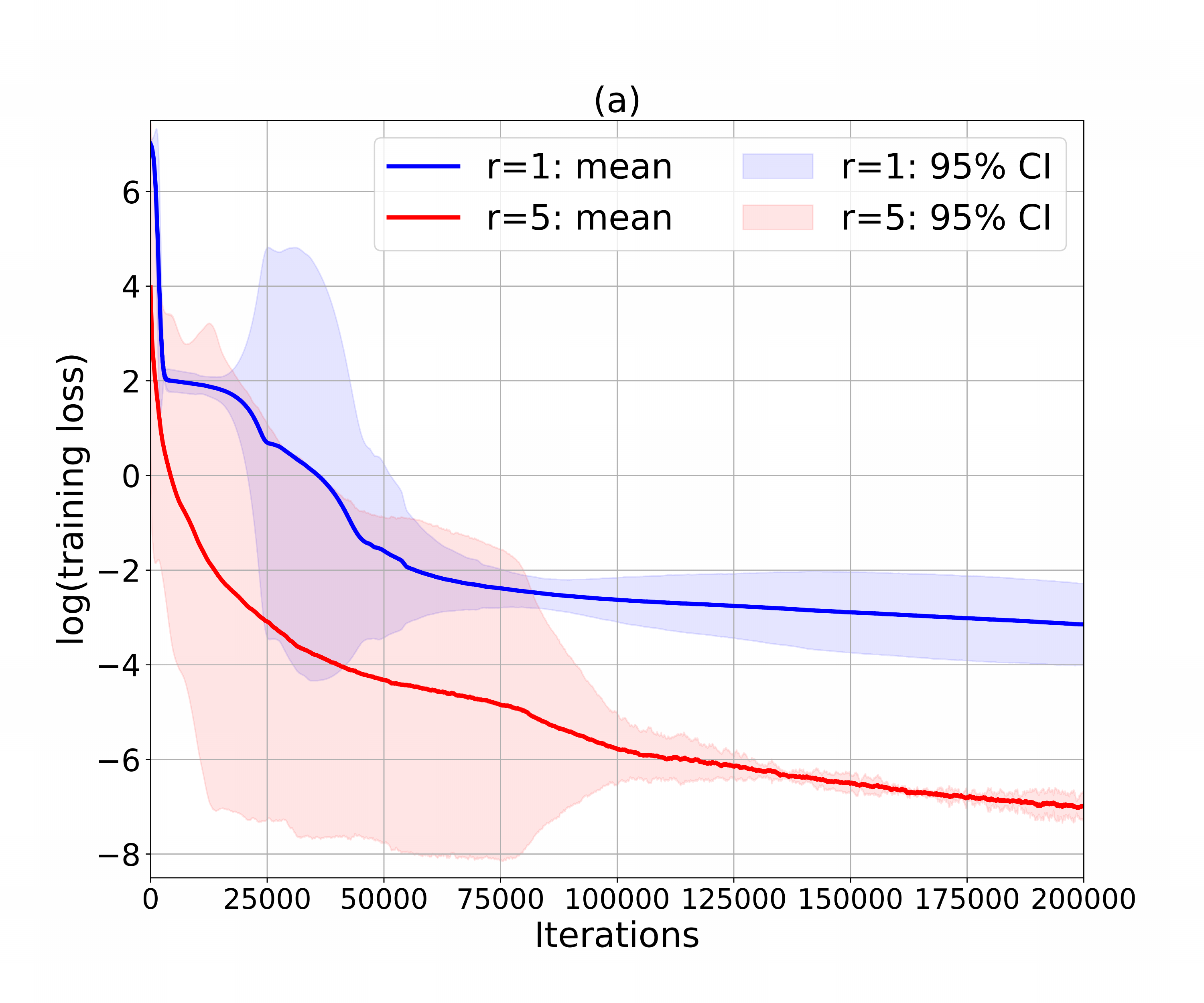}}
\subfigure{\includegraphics[scale=0.285]{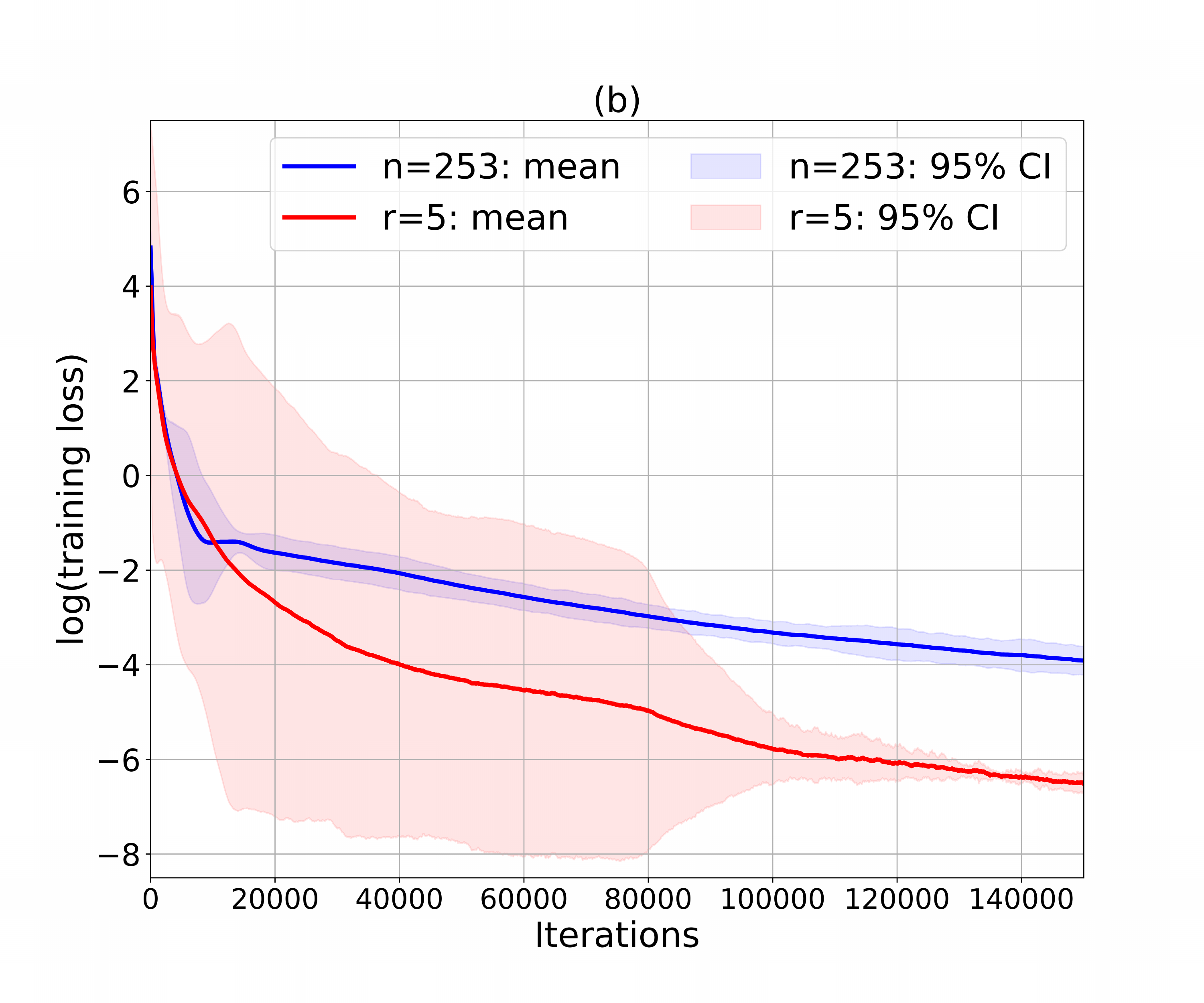}}
\caption{Deep PhN v.s. classical DNN: log of training loss (5 random seeds).}
\label{ghj}
\end{figure*}

\subsubsection{Large Noise} 
To test the robustness of deep PhN faced with large noise, we consider the network structure in Figure \ref{et} (c). To introduce noisy datasets for training, we use the inputs $\mathbf{x}(k)$ from stations $S_{2}$--$S_{4}$, while the dataset of outputs $[\widehat{\mathbf{y}}]_1$ and $[\widehat{\mathbf{y}}]_2$ is from the station $S_{1}$. This setting means we will use station $S_{1}$'s neighbors' inputs to predict its heat index mean and average wind speed. For the suppressor of deep PhN in \eqref{compb}, we let $\beta = 90$ and $\alpha = -1$. The trajectories of training loss in Figure \ref{ddo} together with station locations in Figure \ref{location} show that the training loss decreases as the distance with station $S_1$ increases. Considering the fact that noise of training (input) datasets can increase as distance with (truth) station $S_1$ increases, the result demonstrates the superiority of suppressor in mitigating the influence of large noise in augmented input features. 
\begin{figure*}
\centering
\subfigure{\includegraphics[scale=0.235]{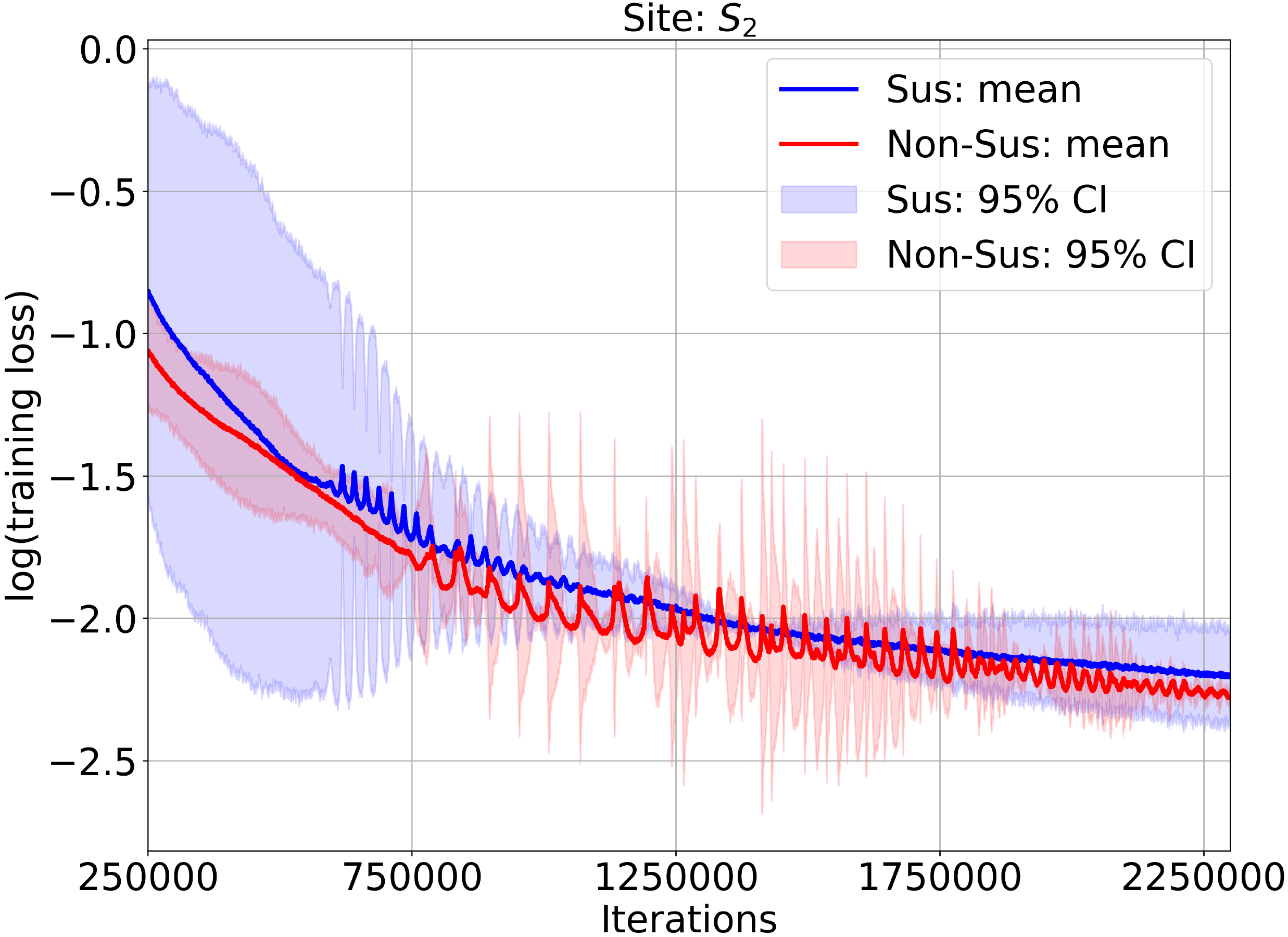}}
\subfigure{\includegraphics[scale=0.235]{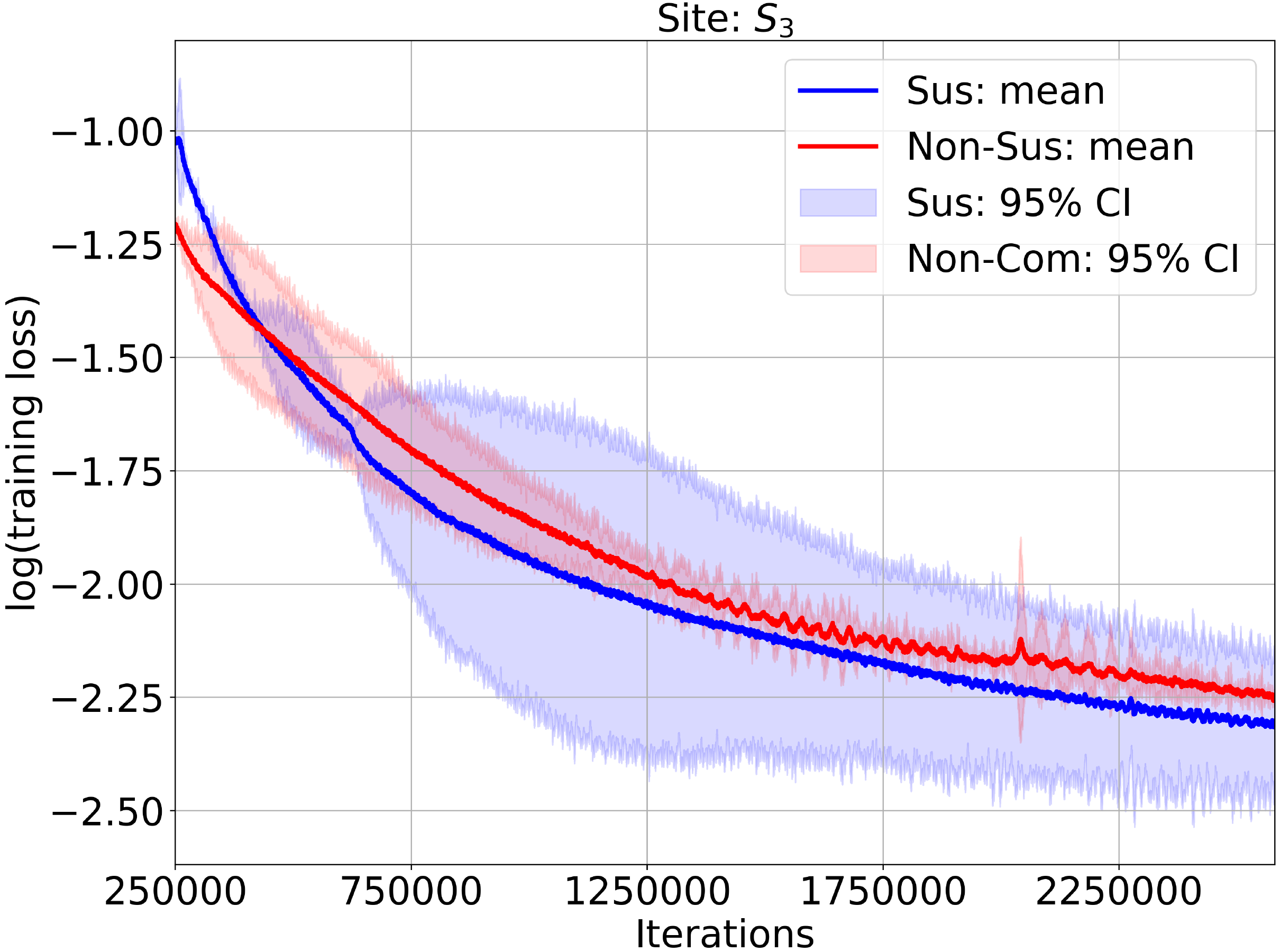}}
\subfigure{\includegraphics[scale=0.235]{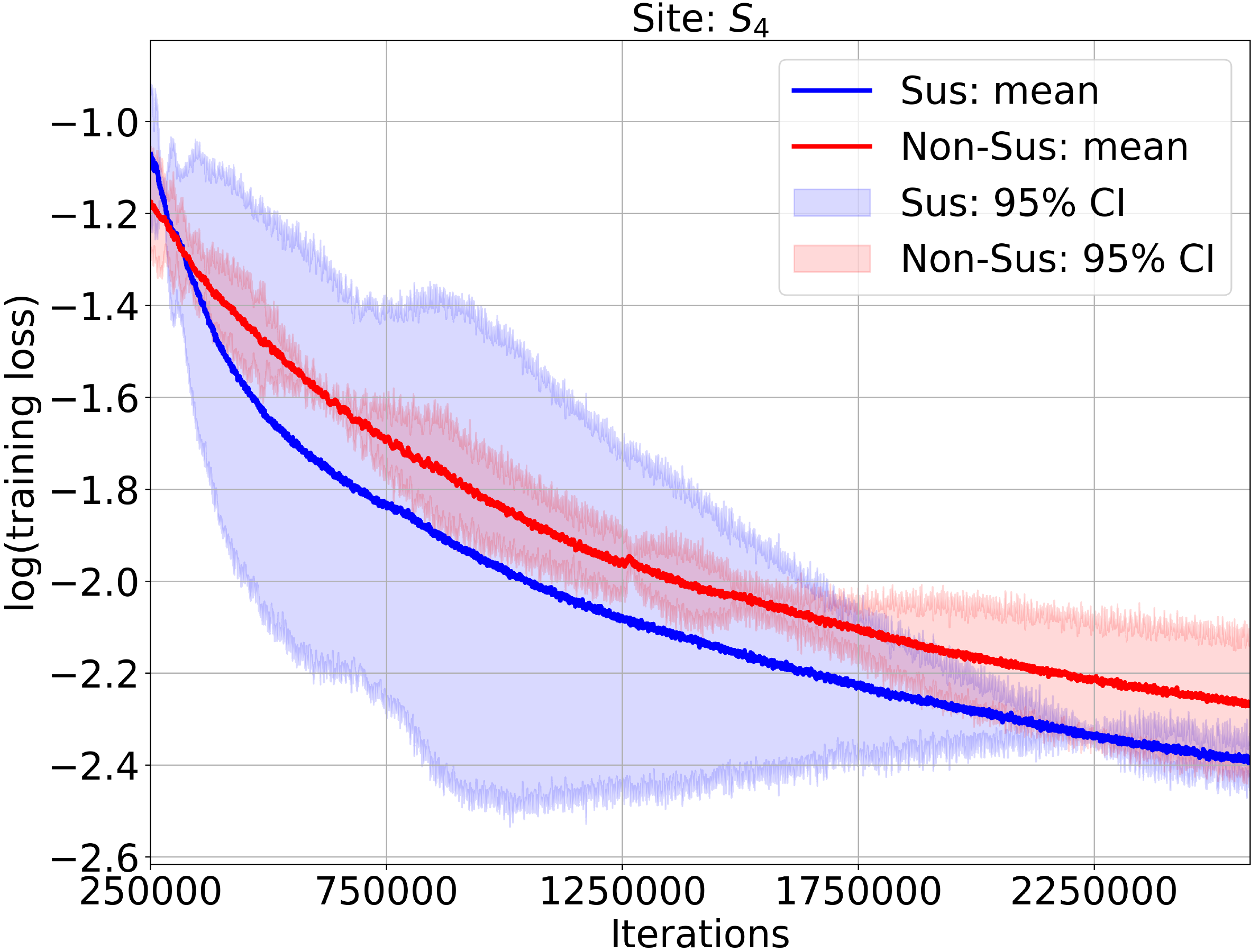}}
\caption{Training loss with and without suppressor function (5 random seeds).}
\label{ddo}
\end{figure*}

\section{Discussion}
In this article, we have proposed a physics-model-based deep neural network framework, called Phy-Taylor, that addresses the challenge the purely data-driven deep DNNs are facing: the violation of inferred relations with physics laws. The Phy-Taylor framework introduces two contributions: the deep physics-compatible neural networks (PhNs) and a physics-guided neural network (NN) editing mechanism, aiming at ensuring strict compliance with prior physical knowledge. The PhN aims to directly capture nonlinearities inspired by physical quantities, such as kinetic energy and aerodynamic drag force. The NN editing mechanism further modifies network links and activation functions consistently with physical knowledge to avoid spurious correlations. As an extension, we have also proposed a self-correcting Phy-Taylor framework that introduces a core capability of automatic output correction when violation of safety occurs. The self-correcting Phy-Taylor successfully addresses the dilemma of prediction horizon and computation time that nonlinear model-predictive control and control barrier function are facing in safety- and time-critical systems.

The experiments show that through physics-guided NN editing and direct capture of hard-to-learn nonlinearities, the Phy-Taylor exhibits a considerable reduction in learning parameters, a remarkably accelerated training process, and greatly enhanced model robustness and accuracy. This suggests that building on the Phy-Taylor framework, the concurrently energy-efficient, robustness and high-accurate DNNs are promising for the energy-constraint physical engineering systems (see e.g., internet of things and battery-powered drones), which constitutes a part of our future research. The current Phy-Taylor however suffers from curse of dimensionality, so it can hardly be applied to high-dimensional data, such as image and text. Overcoming the dimensionality curse problem will thus make the Phy-Taylor scalable. The Tucker Decomposition is a potential to address this problem, since it can decompose the higher-order derivatives in Taylor expansions parameterized by DNNs into small core tensors and a set of factor matrices.  

\section{Methods}

\subsection{Datasets} All the datasets used in the experiments are publicly available at \url{https://waynestateprod-my.sharepoint.com/:f:/g/personal/hm9062_wayne_edu/EoXO99WN8zJEidtGRj-dISIBQPBWCnL_Ji6QOZ1uJACjug}.

The datasets in Experiment \ref{expav} are collected from the AutoRally platform \cite{goldfain2019autorally}. The vehicle therein is set to run 6 times, each running time is 30 minutes, which can generate around 1100 samples of ellipse trajectories. The datasets in Experiment \ref{expe} are generated by solving the differential equations of coupled pendulum in MATLAB using the ode45 solver, given 1000 initial conditions. The climate datasets in Experiment \ref{clima} are from the Climate Data Online--National Oceanic and Atmospheric Administration\footnote{\tiny{\hspace{-0.0cm} NOAA: \hspace{-0.0cm} \url{https://www.ncdc.noaa.gov/cdo-web/datasets/NORMAL_HLY/locations/FIPS:17/detail}}}, whose period of record is 01/01/2010--12/31/2010. The datasets in Experiments \ref{expav} and \ref{expe} are evenly distributed to 6 files: two files are preserved for testing and validating, the remaining  4 files are used for training.  For the Experiment \ref{clima}, the data over the final 100 hours is preserved for testing, the data over another 864 hours are used for validating, the remaining data are used for training. 

\subsection{Code} 
For the code, we use the Python API for the TensorFlow framework \cite{kingma2014adam} and the Adam optimizer \cite{abadi2016tensorflow} for training. The Python version is 2.7.12. The TensorFlow version is 1.4.0. Our source code is publicly available at GitHub: \url{https://github.com/ymao578/Phy-Taylor}. The source code of implementing self-correcting Phy-Taylor in the AutoRally platform is publicly available at  \url{https://waynestateprod-my.sharepoint.com/:f:/g/personal/hm9062_wayne_edu/EnDmRmmbKlJFmlQfC3qLMHwBf4KG9FRGVVo3FFVk1TrZeg?e=TQZSgr}.

\subsection{Training}
We set the batch-size to 200 for the Experiments \ref{expav} and \ref{expe}, while the  batch-size of Experiment \ref{clima} is 150. The learning rates of Experiments \ref{expav}--\ref{clima} are set to 0.0005, 0.00005 and 0.00005, respectively. In all the experiments,  each weight matrix is initialized randomly from a (truncated) normal distribution with zero mean and standard deviation, discarding and re-drawing any samples that are more than two standard deviations from the mean. We initialize each bias according to the normal distribution with zero mean and standard deviation.

\bibliography{Edited_ref}

\section{Supplementary Information} \label{sec:SI}


\subsection{Auxiliary Theorems} \label{Aux}
\begin{thm}
The DNR magnitude of high-order monomial $[\bar{\mathbf{x}}]_i^p[\bar{\mathbf{x}}]_j^q$, $p$, $q \in \mathbb{N}$,  is strictly increasing with respect to $|\mathrm{DNR}_i|$ and $|\mathrm{DNR}_j|$, if 
\begin{align}
\mathrm{DNR}_i, ~\mathrm{DNR}_j \in (-\infty, -1] ~~~~~\text{or}~~~~~ \mathrm{DNR}_i, ~\mathrm{DNR}_j \in [-\frac{1}{2}, 0) ~~~~~\text{or}~~~~~ \mathrm{DNR}_i, ~\mathrm{DNR}_j \in (0, \infty). \label{concon}
\end{align}\label{colthm}
\end{thm}
\begin{proof}
In view of definition \ref{def1}, the true data can be equivalently expressed as $[{\mathbf{h}}]_i = \mathrm{DNR}_i \cdot [{\mathbf{w}}]_i$, according to which we have $[\bar{\mathbf{x}}]_i = (1 + \mathrm{DNR}_i){[{\mathbf{w}}]_i}$ such that 
\begin{align}
[\bar{\mathbf{x}}]_i^p[\bar{\mathbf{x}}]_j^q = {( {1 + \mathrm{DNR}_i})^p}{( {1 + \mathrm{DNR}_j} )^q}[\mathbf{w}]_i^p[\mathbf{w}]_j^q, ~~~~~~\text{and}~~~[\mathbf{h}]_i^p[\mathbf{h}]_j^q = \mathrm{DNR}_i^p \cdot \mathrm{DNR}_j^q \cdot [\mathbf{w}]_i^p [\mathbf{w}]_j^q. \label{reformula}
\end{align}
We note the true data of high-order monomial $[\bar{\mathbf{x}}]_i^p[\bar{\mathbf{x}}]_j^q$ is $[{\mathbf{h}}]_i^p[{\mathbf{h}}]_j^q$, the corresponding noise can thus be derived from the formula~\eqref{reformula} as  
\begin{align}
[\bar{\mathbf{x}}]_i^p[\bar{\mathbf{x}}]_j^q - [\mathbf{h}]_i^p[\mathbf{h}]_j^q = \left[{( {1 + \mathrm{DNR}_i})^p}{( {1 + \mathrm{DNR}_j} )^q} -  \mathrm{DNR}_i^p \cdot \mathrm{DNR}_j^q\right] [\mathbf{w}]_i^p[\mathbf{w}]_j^q, \nonumber
\end{align}
which, in conjunction with the second formula in Equation \eqref{reformula}, lead to 
\begin{align}
\left| \mathrm{DNR}^{p+q}_{ij} \right| \triangleq \left|\frac{{[\mathbf{h}]_i^p[\mathbf{h}]_j^q}}{{[\bar{\mathbf{x}}]_i^p[\bar{\mathbf{x}}]_j^q - [\mathbf{h}]_i^p[\mathbf{h}]_j^q}}\right| = \left|\frac{1}{{{{\left( {1 + \frac{1}{{\mathrm{DNR}_i}}} \right)}^p}{{\left( {1 + \frac{1}{{\mathrm{DNR}_j}}} \right)}^q} - 1}}\right|, ~~~~~p,~q \in \mathbb{N}. \label{dnrfinal}
\end{align}

We can straightforwardly verify from formula \eqref{dnrfinal} that if $\mathrm{DNR}_i, ~\mathrm{DNR}_j \in (0, \infty)$, we have 
\begin{align}
\left| \mathrm{DNR}^{p+q}_{ij} \right| = \frac{1}{{{{\left( {1 + \frac{1}{{|\mathrm{DNR}_i|}}} \right)}^p}{{\left( {1 + \frac{1}{{|\mathrm{DNR}_j|}}} \right)}^q} - 1}}, \label{dnrfinal0}
\end{align}
which implies $|\mathrm{DNR}^{p+q}_{ij}|$ is strictly increasing with respect to $|\mathrm{DNR}_i|$ and $|\mathrm{DNR}_j|$ under this condition. 

The condition $\mathrm{DNR}_i, ~\mathrm{DNR}_j \in (-\infty, -1]$ means that
\begin{align}
\frac{{1}}{{\mathrm{DNR}_j}} \in [-1, 0),~~~~~~~~\frac{{1}}{{\mathrm{DNR}_j}} \in [-1, 0),~~~~~~~~1 + \frac{{1}}{{\mathrm{DNR}_i}} \in [0, 1), ~~~~~~~~1 + \frac{{1}}{{\mathrm{DNR}_j}} \in [0, 1), \label{dnrpp1}
\end{align}
considering which, the formula \eqref{dnrfinal} equivalently transforms to 
\begin{align}
\left| \mathrm{DNR}^{p+q}_{ij} \right| = \frac{1}{{{{1 - \left( {1 + \frac{1}{{\mathrm{DNR}_i}}} \right)}^p}{{\left( {1 + \frac{1}{{\mathrm{DNR}_j}}} \right)}^q}}} = \frac{1}{{{{1 - \left( {1 - \frac{1}{{|\mathrm{DNR}_i|}}} \right)}^p}{{\left( {1 - \frac{1}{{|\mathrm{DNR}_j|}}} \right)}^q}}}, \label{dnrfina2}
\end{align}
which reveals that $|\mathrm{DNR}^{p+q}_{ij}|$ is strictly increasing with respect to $|\mathrm{DNR}_i|$ and $|\mathrm{DNR}_j|$. 

The condition $\mathrm{DNR}_i, ~\mathrm{DNR}_j \in [-\frac{1}{2}, 0)$ means
\begin{align}
\frac{{1}}{{\mathrm{DNR}_j}} \in (-\infty, -2], ~~\frac{{1}}{{\mathrm{DNR}_j}} \in (-\infty, -2], ~~1 + \frac{{1}}{{\mathrm{DNR}_i}} \in (-\infty, -1], ~~1 + \frac{{1}}{{\mathrm{DNR}_j}} \in (-\infty, -1], \label{dnrpp2}
\end{align}
in light of which, the formula \eqref{dnrfinal} can equivalently express as 
\begin{itemize}
  \item if $p + q$ is even,
  \begin{align}
\left| \mathrm{DNR}^{m+n}_{ij} \right| = \frac{1}{{{{\left| {\frac{1}{{|\mathrm{DNR}_i|}}-1} \right|}^p}{{\left| {\frac{1}{{|\mathrm{DNR}_j|}}-1} \right|}^q }} - 1}, \label{dnrfina3}
\end{align}
  \item if $p + q$ is odd, 
  \begin{align}
\left| \mathrm{DNR}^{p+q}_{ij} \right| = \frac{1}{{{{1 + \left| {\frac{1}{{|\mathrm{DNR}_i|}}-1} \right|}^p}{{\left| {\frac{1}{{|\mathrm{DNR}_j|}}-1} \right|}^q }}}. \label{dnrfina4}
\end{align}
\end{itemize}
We note both the functions \eqref{dnrfina3} and \eqref{dnrfina4} imply $|\mathrm{DNR}^{m+n}_{ij}|$ is strictly increasing with respect to $|\mathrm{DNR}_i|$ and $|\mathrm{DNR}_j|$, which completes the proof.  
\end{proof}

\begin{thm}\cite{stanley1986enumerative}  For any pair of positive integers $n$ and $k$, the number of $n$-tuples of non-negative integers whose sum is $r$ is equal to the number of multisets of cardinality $n - 1$ taken from a set of size $n + r - 1$, i.e., $\left( \begin{array}{l}
n + r - 1\\
n - 1
\end{array} \right) = \frac{{\left( {n + r - 1} \right)!}}{{\left( {n - 1} \right)!r!}}$. \label{ath3}
\end{thm}

\begin{thm}
The space complexity of Phy-Augmentation, i.e., the dimension of terminal output generated by Algorithm~\ref{ALG1}, in PhN layer \eqref{sTNO}  is 
\begin{align}
\mathrm{len}(\mathfrak{m}({\mathbf{x}},r)) = \sum\limits_{s = 1}^r {\frac{{\left( {n + s - 1} \right)!}}{{\left( {n - 1} \right)!s!}}} + 1. \label{cpeq1}
\end{align}\label{thm3}
\end{thm} 
\begin{proof}
We denote the output from Line \ref{ALG1-102} of Algorithm \ref{ALG1} by $\overline{\mathbf{x}}$. Let us first consider the case $r = 1$. In this case, the Algorithm~\ref{ALG1} skips the Lines \ref{alg1-3}--\ref{alg1-14} and arrives at $\mathfrak{m}(\mathbf{x},r) = [1; ~ \overline{\mathbf{x}}]$ in Line \ref{alg1-16}. Noticing from  Line \ref{ALG1-102} that  $\overline{\mathbf{x}} \in \mathbb{R}^{n}$, we obtain $\text{len}(\mathfrak{m}(\mathbf{x},r))  = n + 1$, which verifies the correctness of \eqref{cpeq1} with $r = 1$. 

We next consider the case $r \ge 2$. Given the input dimension $n$ and an order $s \in \{2, \ldots, r-1, r\}$, the Lines \ref{alg1-4}--\ref{alg1-13} of Algorithm \ref{ALG1} are to generate all the non-missing and non-redundant monomials included in ${\left( {\sum\limits_{i = 1}^n {{[\overline{\mathbf{x}}]_i}} } \right)^s}$. The problem of the number of generated monomials via Algorithm \ref{ALG1} is equivalent to the problem that for any pair of positive integers $n$ and $s$, the number of $n$-tuples of non-negative integers (whose sum is $s$) is equal to the number of multisets of cardinality $n - 1$ taken from a set of size $n + s - 1$. Additionally, we note that the vector generated in Line \ref{alg1-12} of Algorithm \ref{ALG1}, denoted by $\widetilde{\mathfrak{m}}(\mathbf{x},s)$, stacks all the generated monomials. According to the auxiliary Theorem \ref{ath3} in Supplementary Information \ref{Aux}, we then have $\text{len}(\widetilde{\mathfrak{m}}(\mathbf{x},s)) = \frac{{\left( {n + s - 1} \right)!}}{{\left( {n - 1} \right)!s!}}$. Finally, we note that the Lines \ref{alg1-3}, and \ref{alg1-16} of Algorithm \ref{ALG1} imply that the generated vector $\mathfrak{m}(\mathbf{x},r)$ stack the $1$ with $\widetilde{\mathfrak{m}}(\mathbf{x},s)$ over $s \in \{1, \ldots, r-1, r\}$, respectively. We thus can obtain \eqref{cpeq1}.
\end{proof}

\subsection{Proof of Theorem \ref{th2}} \label{SI02}
We note the $[\widetilde{\mathbf{h}}]_i$ given in Equation \eqref{cohgq3} can be written as  
\begin{align}
[\widetilde{\mathbf{h}}]_i = \begin{cases}
		[\mathbf{h}]_i, & \text{if}~[\mathbf{h}]_i + [\mathbf{w}]_i < 0\\
		[\mathbf{h}]_i, & \text{if}~[\mathbf{h}]_i + [\mathbf{w}]_i \ge 0 ~\text{and}~[\mathbf{w}]_i < 0 \\
		[\mathbf{h}]_i \cdot \kappa_i + \rho_i, & \text{if}~[\mathbf{h}]_i + [\mathbf{w}]_i \ge 0 ~\text{and}~[\mathbf{w}]_i > 0
	\end{cases}, \label{pth32}
\end{align}
subtracting $[{\mathbf{h}}]_i$ from which yields 
\begin{align}
\left| [\widetilde{\mathbf{h}}]_i - [{\mathbf{h}}]_i\right| &= \begin{cases}
		0, & \text{if}~[\mathbf{h}]_i + [\mathbf{w}]_i < 0\\
		0, & \text{if}~[\mathbf{h}]_i + [\mathbf{w}]_i \ge 0 ~\text{and}~[\mathbf{w}]_i < 0\\
		\left| {[\mathbf{h}]_i \cdot \left( {\kappa_i  - 1} \right) + \rho_i } \right|, & \text{if}~[\mathbf{h}]_i + [\mathbf{w}]_i \ge 0 ~\text{and}~[\mathbf{w}]_i > 0 \\
	\end{cases}. \label{pth31}
\end{align}

Referring to the output $\chi([\bar{\bf{x}}]_i)$ of suppressor in Equation \eqref{compbadd}, we can conclude the $[\widetilde{\mathbf{h}}]_i$ given in Equation \eqref{pth32} is the true data of suppressor output. Subtracting the $[\widetilde{\mathbf{h}}]_i$ from the 
$\chi([\bar{\bf{x}}]_i)$ results in the noise $[\widetilde{\mathbf{w}}]_i$ of suppressor output, given in Equation \eqref{cohgq3}.

To prove the property \eqref{ckko}, we consider the following three cases:
\begin{itemize}
  \item If  $[\mathbf{h}]_i + [\mathbf{w}]_i < 0$, we obtain from the the first item of $[\widetilde{\mathbf{h}}]_i$ in Equation \eqref{pth32} and $[\widetilde{\mathbf{w}}]_i$ in Equation \eqref{cohgq3} that $\mathrm{DNR}_{i} = \frac{[\widetilde{\mathbf{h}}]_i}{[\widetilde{\mathbf{w}}]_i} = -1$.
  \item If $[\mathbf{h}]_i + [\mathbf{w}]_i \ge 0$ and $[\mathbf{w}]_i < 0$, we have $[\mathbf{h}]_i > 0$ and $[\mathbf{h}]_i > -[\mathbf{w}]_i| > 0$. We then obtain from the second item of $[\widetilde{\mathbf{h}}]_i$ in Equation \eqref{pth32} and $[\widetilde{\mathbf{w}}]_i$ in Equation \eqref{cohgq3} that $\mathrm{DNR}_{i} = \frac{[\widetilde{\mathbf{h}}]_i}{[\widetilde{\mathbf{w}}]_i} = \frac{[{\mathbf{h}}]_i}{[{\mathbf{w}}]_i} < -1$. 
  \item If $[\mathbf{h}]_i + [\mathbf{w}]_i \ge 0$ and $[\mathbf{w}]_i > 0$, we obtain from the third item of $[\widetilde{\mathbf{h}}]_i$ in Equation \eqref{pth32} and $[\widetilde{\mathbf{w}}]_i$ in Equation \eqref{cohgq3} that $\mathrm{DNR}_{i} = \frac{[\widetilde{\mathbf{h}}]_i}{[\widetilde{\mathbf{w}}]_i} = \frac{{[\mathbf{h}]_i \cdot \kappa_i  + \rho_i}}{{[\mathbf{w}]_i \cdot \kappa_i}}$. Recalling $[\widetilde{\mathbf{w}}]_i > 0$, if $\kappa_i > 0$, the $\frac{{[\mathbf{h}]_i \cdot \kappa_i  + \rho_i}}{{[\mathbf{w}]_i \cdot \kappa_i}} \le -1$ is equivalent to
  \begin{align}
\rho_i  \le  - ([\mathbf{h}]_i + [\mathbf{w}]_i)\kappa_i  < 0, ~~~\text{with}~~\kappa_i > 0,~~[\mathbf{h}]_i + [\mathbf{w}]_i \ge 0. \label{pu1}
\end{align}
If $\kappa_i < 0$, the $\frac{{[\mathbf{h}]_i \cdot \kappa_i  + \rho_i}}{{[\mathbf{w}]_i \cdot \kappa_i}} \le -1$ is equivalent to
\begin{align}
\rho_i  \ge  - \left( {[\mathbf{w}]_i + [\mathbf{h}]_i} \right) \kappa_i  \ge 0, ~~~\text{with}~ \kappa_i < 0,~~ [\mathbf{h}]_i + [\mathbf{w}]_i \ge 0.   \label{pu2}
\end{align}
\end{itemize}
We finally conclude from Equations \eqref{pu1} and \eqref{pu2} that $\mathrm{DNR}_{i} = \frac{[\widetilde{\mathbf{h}}]_i}{[\widetilde{\mathbf{w}}]_i} \in (-\infty, -1]$ under the condition \eqref{compb}. Then, according to Theorem \ref{colthm}, we arrive in the property \eqref{ckko}, which completes the proof.

\subsection{Proof of Theorem \ref{thmmm2}} \label{SI03}
Let us first consider the first PhN layer, i.e., the case $t = 1$. The Line \ref{alg2-3} of Algorithm~\ref{ALG2} means that the knowledge matrix $\mathbf{K}_{\left\langle 1 \right\rangle }$ includes all the known model-substructure parameters, whose corresponding entries in the masking matrix $\mathbf{M}_{\left\langle 1 \right\rangle }$ (generated in the Line \ref{alg2-4} of Algorithm~\ref{ALG2}) are frozen to be zeros. Consequently, both $\mathbf{M}_{\left\langle 1 \right\rangle } \odot \mathbf{A}$ and $\mathbf{U}_{\left\langle 1 \right\rangle } = \mathbf{M}_{\left\langle 1 \right\rangle } \odot \mathbf{W}_{\left\langle 1 \right\rangle }$ excludes all the known model-substructure parameters (included in $\mathbf{K}_{\left\langle 1 \right\rangle }$). With the consideration of Definition \ref{defj}, we thus conclude that $\mathbf{M}_{\left\langle 1 \right\rangle } \odot \mathbf{A} \cdot \mathfrak{m}(\mathbf{x},r_{\left\langle 1 \right\rangle}) + \mathbf{f}(\mathbf{x})$ in the ground-truth model \eqref{exppf1} and $\mathbf{a}_{\left\langle 1 \right\rangle } \odot \text{act}\left( \mathbf{U}_{\left\langle 1 \right\rangle} \cdot {\mathfrak{m}}\left(\mathbf{x},  r_{\left\langle 1 \right\rangle}  \right) \right)$ in the output computation \eqref{exppf2} are independent of the term $\mathbf{K}_{\left\langle 1 \right\rangle } \cdot \mathfrak{m}(\mathbf{x}, r_{\left\langle 1 \right\rangle})$.  Moreover, the activation-masking vector (generated in Line \ref{alg2-5} of Algorithm \ref{ALG2}) indicates that the activation function corresponding to the output's $i$-th entry is inactive, if the all the entries in the $i$-th row of masking matrix are zeros (implying all the entries in the $i$-th row of weight matrix are known model-substructure parameters). Finally, we arrive in the conclusion that the input/output (i.e., $\mathbf{x}$/$\mathbf{y}_{\left\langle 1 \right\rangle }$) of the first PhN layer strictly complies with the available physical knowledge pertaining to the ground truth \eqref{exppf1}, i.e., if the $[\mathbf{A}]_{i,j}$ is a known model-substructure parameter, the $\frac{{\partial {[\mathbf{y}_{\left\langle 1 \right\rangle }]_{i}}}}{{\partial [{\mathfrak{m}}\left( {\mathbf{x},r} \right)]_j}} \equiv \frac{{\partial {[\mathbf{y}]_i}}}{{\partial {[\mathfrak{m}}( {\mathbf{x},r})]_j}} \equiv {[\mathbf{A}]_{i,j}}$ always holds.

We next consider the remaining PhN layers. Considering the Line \ref{alg2-12} Algorithm \ref{ALG2}, we have
\begin{subequations}  
 \begin{align}
\hspace{-0.05cm}[\mathbf{y}_{\left\langle p \right\rangle }]_{1:\text{len}(\mathbf{y})} &= [\mathbf{K}_{\left\langle p \right\rangle } \cdot \mathfrak{m}(\mathbf{y}_{\left\langle p-1 \right\rangle }, r_{\left\langle p \right\rangle})]_{1:\text{len}(\mathbf{y})} +  [\mathbf{a}_{\left\langle p \right\rangle } \odot \text{act}\!\left( {\mathbf{U}_{\left\langle p \right\rangle } \cdot {\mathfrak{m}}( {\mathbf{y}_{\left\langle p-1 \right\rangle },  r_{\left\langle p \right\rangle } })} \right)]_{1:\text{len}(\mathbf{y})} \nonumber\\ 
&= \bigid_{\text{len}(\mathbf{y})} \cdot [\mathfrak{m}(\mathbf{y}_{\left\langle p-1 \right\rangle }, r_{\left\langle p \right\rangle})]_{2:(\text{len}(\mathbf{y}) + 1)}  +  [\mathbf{a}_{\left\langle p \right\rangle } \odot \text{act}\!\left( {\mathbf{U}_{\left\langle p \right\rangle } \cdot {\mathfrak{m}}( {\mathbf{y}_{\left\langle p-1 \right\rangle },  r_{\left\langle p \right\rangle } })}  \right)]_{1:\text{len}(\mathbf{y})} \label{pf41}\\
& = \bigid_{\text{len}(\mathbf{y})} \cdot [\mathbf{y}_{\left\langle p-1 \right\rangle }]_{1:\text{len}(\mathbf{y})} +  [\mathbf{a}_{\left\langle p \right\rangle } \odot \text{act}\!\left( {\mathbf{U}_{\left\langle p \right\rangle } \cdot {\mathfrak{m}}( {\mathbf{y}_{\left\langle p-1 \right\rangle },  r_{\left\langle p \right\rangle } })}  \right)]_{1:\text{len}(\mathbf{y})} \label{pf42}\\
& =  [\mathbf{y}_{\left\langle p-1 \right\rangle }]_{1:\text{len}(\mathbf{y})} +  [\mathbf{a}_{\left\langle p \right\rangle } \odot \text{act}\!\left( {\mathbf{U}_{\left\langle p \right\rangle } \cdot {\mathfrak{m}}( {\mathbf{y}_{\left\langle p-1 \right\rangle },  r_{\left\langle p \right\rangle } })}  \right)]_{1:\text{len}(\mathbf{y})} \nonumber\\
&= [\mathbf{K}_{\left\langle p-1 \right\rangle } \cdot \mathfrak{m}(\mathbf{y}_{\left\langle p-2 \right\rangle }, r_{\left\langle p-1 \right\rangle})]_{1:\text{len}(\mathbf{y})} +  [\mathbf{a}_{\left\langle p \right\rangle } \odot \text{act}\!\left( {\mathbf{U}_{\left\langle p \right\rangle } \cdot {\mathfrak{m}}( {\mathbf{y}_{\left\langle p-1 \right\rangle },  r_{\left\langle p \right\rangle } })}  \right)]_{1:\text{len}(\mathbf{y})} \nonumber\\
&= \bigid_{\text{len}(\mathbf{y})} \cdot [\mathfrak{m}(\mathbf{y}_{\left\langle p-2 \right\rangle }, r_{\left\langle p-1 \right\rangle})]_{2:(\text{len}(\mathbf{y})+1)} +  [\mathbf{a}_{\left\langle p \right\rangle } \odot \text{act}\!\left( {\mathbf{U}_{\left\langle p \right\rangle } \cdot {\mathfrak{m}}( {\mathbf{y}_{\left\langle p-1 \right\rangle },  r_{\left\langle p \right\rangle } })}  \right)]_{1:\text{len}(\mathbf{y})} \nonumber\\
&= \bigid_{\text{len}(\mathbf{y})} \cdot [\mathbf{y}_{\left\langle p-2 \right\rangle }]_{1 :\text{len}(\mathbf{y})} +  [\mathbf{a}_{\left\langle p \right\rangle } \odot \text{act}\!\left( {\mathbf{U}_{\left\langle p \right\rangle } \cdot {\mathfrak{m}}( {\mathbf{y}_{\left\langle p-1 \right\rangle },  r_{\left\langle p \right\rangle } })}  \right)]_{1:\text{len}(\mathbf{y})} \nonumber\\
&=  [\mathbf{y}_{\left\langle p-2 \right\rangle }]_{1 :\text{len}(\mathbf{y})} +  [\mathbf{a}_{\left\langle p \right\rangle } \odot \text{act}\!\left( {\mathbf{U}_{\left\langle p \right\rangle } \cdot \breve{\mathfrak{m}}( {\mathbf{y}_{\left\langle p-1 \right\rangle },  r_{\left\langle p \right\rangle } })}  \right)]_{1:\text{len}(\mathbf{y})} \nonumber\\
& = \ldots \nonumber\\
& = [\mathbf{y}_{\left\langle 1 \right\rangle }]_{1:\text{len}(\mathbf{y})} +  [\mathbf{a}_{\left\langle p \right\rangle } \odot \text{act}\!\left( {\mathbf{U}_{\left\langle p \right\rangle } \cdot {\mathfrak{m}}( {\mathbf{y}_{\left\langle p-1 \right\rangle },  r_{\left\langle p \right\rangle } })}  \right)]_{1:\text{len}(\mathbf{y})} \nonumber\\
& = [\mathbf{K}_{\left\langle 1 \right\rangle } \cdot \mathfrak{m}(\mathbf{x}, r_{\left\langle 1 \right\rangle})]_{1:\text{len}(\mathbf{y})}+  [\mathbf{a}_{\left\langle p \right\rangle } \odot \text{act}\!\left( {\mathbf{U}_{\left\langle p \right\rangle } \cdot {\mathfrak{m}}( {\mathbf{y}_{\left\langle p-1 \right\rangle },  r_{\left\langle p \right\rangle } })}  \right)]_{1:\text{len}(\mathbf{y})}, \nonumber\\
&  = \mathbf{K}_{\left\langle 1 \right\rangle } \cdot \mathfrak{m}(\mathbf{x}, r_{\left\langle 1 \right\rangle}) +  [\mathbf{a}_{\left\langle p \right\rangle } \odot \text{act}\!\left( {\mathbf{U}_{\left\langle p \right\rangle } \cdot {\mathfrak{m}}( {\mathbf{y}_{\left\langle p-1 \right\rangle },  r_{\left\langle p \right\rangle } })}  \right)]_{1:\text{len}(\mathbf{y})}, \label{pf46}
 \end{align}
\end{subequations}
where \eqref{pf41} and \eqref{pf42} are obtained from their previous steps via considering the structure of block matrix $\mathbf{K}_{\left\langle t \right\rangle }$ (generated in Line \ref{alg2-7} of Algorithm \ref{ALG2}) and the formula of augmented monomials: $\mathfrak{m}(\mathbf{x},r) = \left[1;~\mathbf{x};~[\mathfrak{m}( {\mathbf{x},r})]_{(\text{len}(\mathbf{x})+2) : \text{len}(\mathfrak{m}(\mathbf{x},r))} \right]$ (generated via Algorithm \ref{ALG1}). The remaining iterative steps follow the same path. 

The training loss function is to push the terminal output of Algorithm \ref{ALG2} (i.e., $\widehat{\mathbf{y}} =  \mathbf{y}_{\left\langle p \right\rangle }$) to approximate the real output $\mathbf{y}$, which in light of \eqref{pf46} yields 
 \begin{align}
\widehat{\mathbf{y}} &=  \mathbf{K}_{\left\langle 1 \right\rangle } \cdot \mathfrak{m}(\mathbf{x}, r_{\left\langle 1 \right\rangle}) +  [\mathbf{a}_{\left\langle p \right\rangle } \odot \text{act}\!\left( {\mathbf{U}_{\left\langle p \right\rangle } \cdot {\mathfrak{m}}( {\mathbf{y}_{\left\langle p-1 \right\rangle },  r_{\left\langle p \right\rangle } })}  \right)]_{1:\text{len}(\mathbf{y})} \nonumber\\
&=  \mathbf{K}_{\left\langle 1 \right\rangle } \cdot \mathfrak{m}(\mathbf{x}, r_{\left\langle 1 \right\rangle}) +  \mathbf{a}_{\left\langle p \right\rangle } \odot \text{act}\!\left( {\mathbf{U}_{\left\langle p \right\rangle } \cdot {\mathfrak{m}}( {\mathbf{y}_{\left\langle p-1 \right\rangle },  r_{\left\langle p \right\rangle } })}  \right), \label{pf5}
 \end{align}
where \eqref{pf5} from its previous step is obtained via considering the fact $\text{len}(\widehat{\mathbf{y}}) = \text{len}(\mathbf{y}) = \text{len}(\mathbf{y}_{\left\langle p \right\rangle })$. Meanwhile, the condition of generating weight-masking matrix in Line \ref{alg2-8} of Algorithm \ref{ALG2} removes all the node-representations' connections with the known model-substructure parameters included in $\mathbf{K}_{\left\langle 1 \right\rangle }$. Therefore, we can conclude that in the terminal output computation \eqref{pf5}, the term  $\mathbf{a}_{\left\langle p \right\rangle } \odot \text{act}\!\left( {\mathbf{U}_{\left\langle p \right\rangle } \cdot {\mathfrak{m}}( {\mathbf{y}_{\left\langle p-1 \right\rangle },  r_{\left\langle p \right\rangle } })}  \right)$ does not have influence on the computing of knowledge term $\mathbf{K}_{\left\langle 1 \right\rangle } \cdot \mathfrak{m}(\mathbf{x}, r_{\left\langle 1 \right\rangle})$. Thus,  the Algorithm \ref{ALG2} strictly embeds and preserves the available knowledge pertaining to the physics model of ground truth \eqref{eq:lobja}, or equivalently the \eqref{exppf1}.

\subsection{Proof of Theorem \ref{cor}} \label{SI06}
Due to Theorem \ref{thm3} in Supplementary Information \ref{Aux}, the number of augmented monomials of $d$ cascading PhNs \eqref{cko} is obtained as 
\begin{align}
\sum\limits_{p = 1}^d {\text{len}(\mathfrak{m}({\mathbf{x}},r_{{\left\langle p \right\rangle }}))}  & =
\underbrace{\sum\limits_{s = 1}^{{r_{\left\langle 1 \right\rangle }}} {\frac{{\left( {n + s - 1} \right)!}}{{\left( {n - 1} \right)s!}}} + 1}_{\text{the first PhN}} + \underbrace{\sum\limits_{v = 1}^{d - 1} {\sum\limits_{s = 1}^{{r_{{\left\langle v+1 \right\rangle }}}} {\frac{{\left( {{n_{\left\langle v \right\rangle }} + s - 1} \right)!}}{{\left( {{n_{\left\langle v \right\rangle }} - 1} \right)!s!}}} } + d-1}_{\text{the remaining PhNs}}. \label{ccff1}
\end{align}
The condition \eqref{mcc} implies that $r > r_{{\left\langle 1 \right\rangle}}$, which in conjunction with \eqref{cpeq1}, lead to 
\begin{align}
\text{len}(\mathfrak{m}({\mathbf{x}},r)) = \sum\limits_{s = 1}^{{r_{\left\langle 1 \right\rangle }}} {\frac{{\left( {n + s - 1} \right)!}}{{\left( {n - 1} \right)!s!}}}  + \sum\limits_{s = {r_{\left\langle 1 \right\rangle }} + 1}^r {\frac{{\left( {n + s - 1} \right)!}}{{\left( {n - 1} \right)!s!}}} + 1. \label{ccff2}
\end{align}
Subtracting \eqref{ccff1} from \eqref{ccff2} yields \eqref{cffc}. 

\subsection{Derivations of Solution (\ref{rch43})} \label{SI10}
With the consideration of \eqref{rchoo}--\eqref{rch3}, the safety formulas: $\left[ {\bf{s}}({\bf{u}}(k)) \right]_1 = [\widehat{\mathbf{c}}]_1$ and $\left[ {\bf{s}}({\bf{u}}(k)) \right]_2 = [\widehat{\mathbf{c}}]_2$ can be rewritten as 
\begin{align}
{\lambda _1}[\widehat{\mathbf{u}}(k)]_1^2 + {\lambda _2}[\widehat{\mathbf{u}}(k)]_2^2 &= {[\widehat{\mathbf{c}}]_1} - {[\mathbf{b}]_1}, \label{rch41}\\
{\left[\begin{array}{l}
\theta (k)\\
\gamma (k)
\end{array} \right]^\top} 
 {\mathbf{P}}_2  \left[ \begin{array}{l}
\theta (k)\\
\gamma (k)
\end{array} \right] &= {s_{11}}[\widehat{\mathbf{u}}(k)]_1^2 + 2{s_{12}}{[\widehat{\mathbf{u}}}(k)]_1{[\widehat{\mathbf{u}}}(k)]_2 + {s_{22}}[\widehat{\mathbf{u}}(k)]_2^2 =  [\mathbf{b}]_2 - {[\widehat{\mathbf{c}}]_2},  \label{rch42}
\end{align}

We now define:  
\begin{align}
\overline{\mu}_1 \triangleq \frac{{{[\widehat{\mathbf{c}}]_1} - {[\mathbf{b}]_1}}}{{{\lambda _1}}}, ~~~~~\overline{\lambda} \triangleq \frac{{{\lambda _2}}}{{{\lambda _1}}}, ~~~~~\bar b \triangleq {[\mathbf{b}]_2} - {[\widehat{\mathbf{c}}]_2}. \label{defop}
\end{align}
leveraging which, the \eqref{rch41} is rewritten as 
\begin{align}
[\widehat{\mathbf{u}}(k)]_1^2 = \overline{\mu}_1 - \overline{\lambda}[\widehat{\mathbf{u}}(k)]_2^2, \label{ppko1}
\end{align}
and we can obtain from \eqref{rch42} that 
\begin{align}
\hspace{-0.7cm} 4s_{12}^2[\widehat{\mathbf{u}}(k)]_2^2[\widehat{\mathbf{u}}(k)]_1^2 = {{\overline{b}}^2} + s_{11}^2[\widehat{\mathbf{u}}(k)]_1^4 + s_{22}^2[\widehat{\mathbf{u}}(k)]_2^4 - 2\overline{b}{s_{11}}[\widehat{\mathbf{u}}(k)]_1^2 - 2\overline{b}{s_{22}}[\widehat{\mathbf{u}}(k)]_2^2 + 2{s_{11}}{s_{22}}[\widehat{\mathbf{u}}( k)]_1^2[\widehat{\mathbf{u}}(k)]_2^2, \nonumber
\end{align}
substituting \eqref{ppko1} into which yields 
\begin{align}
{\varpi _1}[\widehat{\mathbf{u}}]_2^4\left( k \right) + {\varpi_2}[\widehat{\mathbf{u}}]_2^2\left( k \right) + {\varpi _3}\left( k \right) = 0, \label{ppko2}
\end{align}
where 
\begin{align}
{\varpi _1} &\triangleq s_{11}^2{\overline{\lambda}^2} + s_{22}^2 - 2{s_{11}}{s_{22}}\overline{\lambda}  + 4s_{12}^2\overline{\lambda}, \label{ppko21}\\
{\varpi _2} &\triangleq 2\bar b {s_{11}}\overline{\lambda}  - 2s_{11}^2{\overline{\mu}_1}\overline{\lambda}  - 2\bar b{s_{22}} + 2{s_{11}}{s_{22}}{\overline{\mu} _1} - 4s_{12}^2{\overline{\mu}_1}, \label{ppko22}\\
{\varpi _3} &\triangleq {\bar b^2} + s_{11}^2\overline{\mu}_1^2 - 2\bar b{s_{11}}{\overline{\mu}_1}. \label{ppko23}
\end{align}

Considering $\widehat{u}_2^2(k) \ge 0$, the solution of \eqref{ppko2} is 
\begin{align}
[\widehat{\mathbf{u}}]_2^2(k) = \frac{{\sqrt {\varpi_2^2 - 4{\varpi _1}{\varpi _3}}  - {\varpi_2}}}{{2{\varpi_1}}}, \label{ppko3}
\end{align}
substituting which into \eqref{ppko1} yields 
\begin{align}
[\widehat{\mathbf{u}}]_1^2(k) = {{\overline{\mu} }_1} -  \overline{\lambda} \frac{{\sqrt {\varpi _2^2 - 4{\varpi_1}{\varpi _3}}  - {\varpi_2}}}{{2{\varpi _1}}}. \label{ppko4}
\end{align}
The $\widehat{\mathbf{u}}(k)$ is then straightforwardly obtained from \eqref{ppko3} and \eqref{ppko4}: 
\begin{align}
\widehat{\mathbf{u}}(k) = \left[ { \pm \sqrt {{{\overline{\mu}}_1} -  \overline{\lambda} \frac{{\sqrt {\varpi_2^2 - 4{\varpi_1}{\varpi_3}}  - {\varpi_2}}}{{2{\varpi_1}}}}; \pm \sqrt {\frac{{\sqrt {\varpi_2^2 - 4{\varpi _1}{\varpi _3}}  - {\varpi_2}}}{{2{\varpi _1}}}} } \right], \nonumber
\end{align}
which, in conjunction with \eqref{rch3} and $Q^{-1} = Q = Q^\top$, lead to 
\begin{align}
\left[ \begin{array}{l}
\theta \left( k \right)\\
\gamma \left( k \right)
\end{array} \right] = {\mathbf{Q}_1}\left[ \begin{array}{l}
 \pm \sqrt {{{\overline{\mu} }_1} - \overline{\lambda} \frac{{\sqrt {\varpi _2^2 - 4{\varpi_1}{\varpi _3}}  - {\varpi_2}}}{{2{\varpi_1}}}} \\
 \pm \sqrt {\frac{{\sqrt {\varpi _2^2 - 4{\varpi _1}{\varpi_3}}  - {\varpi _2}}}{{2{\varpi_1}}}} 
\end{array} \right]. \label{ppko5}
\end{align}
Substituting the notations defined in \eqref{defop} into \eqref{ppko5} and \eqref{ppko21}--\eqref{ppko23}  results in \eqref{rch43} and \eqref{pko1}--\eqref{pko3}, respectively.

\section*{Acknowledgements}
This work was supported in part by National Science Foundation (award number: CPS-1932529), Hong Kong Research Grants Council (RGC) Theme-based Research Scheme T22-505/19-N (P0031331, RBCR, P0031259, RBCP), and RGC Germany/HK Joint 
Research Scheme under Grant G-PolyU503/16.

\section*{Author Contributions}
Y.M., L.S., and H.S. designed research; Y.M. performed research; Y.M. led experiments, H.S. conducted experiment of deep Koopman, and Y. L. collected vehicle data and implemented self-correcting Phy-Taylor in AutoRally; Y.M., L.S., Q.W., and T.A. participated in research discussion and data analysis; Y.M. and S.H. wrote this manuscript; T.A. edited this manuscript; L.S. and T.A. led the research team. 
\end{document}